\PassOptionsToPackage{pdftex,
	pdfencoding=auto,
	pdfnewwindow=true,
	pdfusetitle=true,
	bookmarks=true,
	bookmarksnumbered=true,
	bookmarksopen=true,
	pdfpagemode=UseThumbs,
	bookmarksopenlevel=1,
	pdfpagelabels=false,
	breaklinks=true,
	colorlinks=true
}{hyperref}
\PassOptionsToPackage{dvipsnames,pdftex}{xcolor}
\PassOptionsToPackage{labelfont=bf}{caption}
\documentclass[10pt,letterpaper,twocolumn,nofootinbib,longbibliography,floatfix]{article}
\usepackage{tikz}
\usetikzlibrary{shapes,positioning,arrows}
\tikzset{>=stealth', c/.style={draw, rectangle, minimum height=2em, inner sep=0.5em}, q/.style={draw, circle, minimum height=2.4em, inner sep=0}, e/.style={<-}, ed/.style={dashed,<-}}

\usepackage[T1]{fontenc} 
\usepackage[utf8]{inputenx}
\usepackage{newtxtext}
\usetikzlibrary{fit}
\usepackage{setspace}
\usepackage{enumitem}
\usepackage{cite}

\usepackage{microtype}

\usepackage{ragged2e}
\usepackage[]{subfig}
\DeclareCaptionJustification{justified}{\justifying}
\usepackage{amsmath}
\usepackage{amsthm}
\usepackage{xcolor}

\usepackage{amssymb}
\usepackage{mathtools}
\usepackage{fdsymbol}
\usepackage{xcolor}
\definecolor{myurlcolor}{rgb}{0,0,0.7}
\definecolor{myrefcolor}{rgb}{0.8,0,0}
\definecolor{purple}{RGB}{128,0,128}
\definecolor{ultramarine}{RGB}{63, 0, 255}
\definecolor{medblue}{RGB}{0, 0, 100}
\definecolor{googleblue}{RGB}{34, 0, 204}
\definecolor{panblue}{RGB}{0,24,150}
\definecolor{carmine}{RGB}{150, 0, 24}
\definecolor{gray}{RGB}{150, 150, 150}

\usepackage[unicode=true,pdfusetitle, bookmarks=false,bookmarksnumbered=false,
bookmarksopen=false, breaklinks=false,pdfborder={0 0 0},backref=false,
colorlinks=true,
linkcolor=carmine,
citecolor=googleblue,
urlcolor=panblue,
anchorcolor=OliveGreen,
]{hyperref}

\let\oldperp\perp
\renewcommand{\perp}{\mathbin{\oldperp_{\text{d}}}}

\usepackage[most]{tcolorbox}

\usepackage{amsfonts}
\usepackage{graphicx} 
\usepackage{hyperref}
\usepackage{placeins}

\usepackage{pifont}

\usepackage{thmtools}
\usepackage{xspace}
\usepackage{appendix}
\usepackage{mathrsfs}

\newtheorem{definition}{Definition}
\newtheorem{lemma}{Lemma}
\newtheorem{theorem}{Theorem}
\newtheorem{proposition}{Proposition}

\usepackage{chngcntr}
\counterwithin{figure}{section}
\usepackage{tikz}
\usetikzlibrary{shapes,positioning,arrows}
\tikzset{>=stealth', c/.style={draw, rectangle, minimum height=2em, inner sep=0.5em}, q/.style={draw, circle, minimum height=2.4em, inner sep=0}, e/.style={<-}, ed/.style={dashed,<-}}
\newcommand{\vis}{\mathtt{Vnodes}}
\newcommand{\repr}{\mathtt{Inodes}}

\newcommand{\lat}{\mathtt{Lnodes}}
\newcommand{\edges}{\mathtt{edges}}
\newcommand{\nodes}{\mathtt{nodes}}
\newcommand{\pa}{\mathtt{pa}}

\newcommand{\ch}{\mathtt{ch}}

\newcommand{\Vpa}{\mathtt{Vpa}}
\newcommand{\Rpa}{\mathtt{Ipa}}
\newcommand{\Lpa}{\mathtt{Lpa}}
\newcommand{\Vch}{\mathtt{Vch}}
\newcommand{\Lch}{\mathtt{Lch}}

\newcommand{\exog}{\mathtt{exog}}
\newcommand{\remove}{\mathtt{remove}}

\newcommand{\spl}{\mathtt{split}}

\newcommand{\mdag}{\mathtt{LnodesToFaces}}
\newcommand{\can}{\mathtt{can}}

\newcommand{\param}{\mathtt{par}}

\newcommand{\GpDAG}{\mathcal{G}}
\newcommand{\GmDAG}{\mathfrak{G}}
\newcommand{\GthreepDAG}{\mathscr{G}}
\newcommand{\GthreemDAG}{\mathbb{G}}

\newcommand{\blk}{\color{black}}

\usepackage{microtype}
\microtypecontext{spacing=nonfrench}
\microtypesetup{
	expansion={true,nocompatibility},
	protrusion={true,nocompatibility},
	activate={true,nocompatibility},
	tracking=true,
	kerning=true,
	spacing={true}
}

\usepackage[all=normal,floats=tight,mathspacing=tight,wordspacing=tight,paragraphs=normal,tracking=tight,charwidths=tight,mathdisplays=normal,sections=normal,margins=normal]{savetrees}
\everypar=\expandafter{\the\everypar\loosness=-1 }
\usepackage{authblk}

\title{Everything that can be learned about a causal structure with latent variables by observational and interventional probing schemes}

\author{Marina Maciel Ansanelli}
\author{Elie Wolfe}
\author{Robert W. Spekkens}
\affil{Perimeter Institute for Theoretical Physics, 31 Caroline Street North, Waterloo, Ontario Canada N2L 2Y5}
\date{}

\begin{document}
\maketitle

\begin{abstract}
	What types of differences among causal structures with latent variables are impossible to distinguish by statistical data obtained by probing each visible variable?
 If the probing scheme is simply passive observation, then it is well-known that many different causal structures can realize the same joint probability distributions.
	 Even for the simplest case of two visible variables, for instance, one cannot distinguish between one variable being a causal parent of the other and the two variables sharing a latent common cause. However, it {\em is} possible to distinguish between these two causal structures if we have recourse to more powerful probing schemes, such as the possibility of intervening on one of the variables and observing the other. 
	   Herein, we address the question of which causal structures
	    remain indistinguishable even given the most informative types of probing schemes on the visible variables.
	   We find that two causal structures remain indistinguishable  if and only if they are both associated with the same mDAG structure (as defined in Ref.~\cite{evans_graphs_2016}). 
	    We also consider the question of when one causal structure {\em dominates} another in the sense that it can realize all of the joint probability distributions that can be realized by the other using a given probing scheme. (Equivalence of causal structures is the special case of mutual dominance.)  Finally, we investigate to what extent one can weaken the probing schemes implemented on the visible variables and still have the same discrimination power as a maximally informative probing scheme. 
\end{abstract}

\section{Introduction}

In the most general type of causal discovery, one seeks to uncover the causal relations among a set of variables by probing them.  It is common for there to be variables that are not probed (for whatever reason), but that act as common causes or causal mediaries between the variables that are probed.  To distinguish these two types of variables, we refer to those that are probed as {\em visible} and those that are not as {\em latent}.
The most common type of probing scheme considered in causal discovery is passive observation.  It is well known, however, that allowing interventions in one's probing scheme (of which there are different varieties) generally allows one to draw stronger conclusions about the causal structure.

Consider a drug trial for example.  If, in a population with some medical condition, one has passively observed a positive correlation between taking the drug and recovering, it is ambiguous as to how much of this correlation is due to the causal influence of the drug on the recovery and how much might be explained by a latent common cause (an unprobed factor that increases both the likelihood of taking the drug and the likelihood of recovery).  If, on the other hand, one conducts a randomized controlled 
 trial, such that every individual in the population is assigned the drug or a placebo at random, thereby breaking the causal connection to any putative common cause, then the positive correlation can only be explained by the causal influence of the drug.
In short, by availing oneself of interventional probing schemes, one can resolve ambiguities that hold for observational probing schemes.

In this work, we determine what ambiguity about the causal structure remains regardless of how the visible variables are probed, including when one makes use of an {\em informationally complete probing scheme}, which is one that provides us with all of the information about the causal structure and the causal parameters that is obtainable by interacting with the visible variables.\footnote{Edge interventions~\cite{edge_interv} are not considered among the probing schemes investigated in this work. Nonetheless, we believe that our conclusions can be easily extended to the case where one has access to edge interventions.} One example of an informationally complete probing scheme is a scheme where, for each visible variable, one observes its natural value and then one implements a do-intervention, which forces the variable to take a fixed desired value. We refer to this as the {\em Observe$\&$Do} probing scheme.  

We find that variations of the causal structure that cannot be discriminated are all and only those that are consistent with the same {\em marginalized DAG} (mDAG) structure, a notion that was introduced in Ref.~\cite{evans_graphs_2016}.  In other words, an informationally complete probing scheme determines the causal structure up to its mDAG equivalence class: two causal structures that are associated with the same mDAG are indistinguishable and two causal structures that are associated with different mDAGs can always be distinguished. Our results therefore establish that the mDAG structure is a fundamental structure for causal analysis. 

We also investigate to what extent one can weaken the probing scheme and still maintain the discriminatory power of an informationally complete probing scheme. 
 We focus on two probing schemes of this type.

Firstly, we consider the probing scheme where, for each visible variable, it is possible to perform a passive observation \emph{or} a do-intervention, but not both. This is the most commonly studied type of interventional probing scheme: in general, in a blind drug trial the experimentalist does not bother to ask the subjects what is their preference regarding taking the drug or not before assigning them the drug or the placebo. If one partitions the ensemble of samples into subensembles and applies do-interventions on different subsets of the visible variables in each subensemble (with passive observations of the complementary set of visible variables) for \emph{all} possible subsets of the set of visible variables, we say that they are applying the \emph{all-patterns Observe-or-Do} probing scheme. As we show here, distinguishability  relative to the all-patterns Observe-or-Do probing scheme is \emph{also} characterized by mDAGs: even this weakened probing scheme can distinguish between any two causal structures that correspond to different mDAGs. Therefore, even though the all-patterns Observe-or-Do probing scheme does not provide us with full information about the causal structure \emph{and} the causal parameters, it does provide us with full information about the causal structure alone. 
 (To see that it cannot distinguish between causal structures that correspond to the same mDAG, it suffices to note that
 it is no more informative than an informationally complete probing scheme.)

In Ref.~\cite{evans_graphs_2016}, it was claimed that there are pairs of causal structures associated with different mDAGs which are indistinguishable when, for each visible variable, the experimenter performs a passive observation \emph{or} a do-intervention (an example is given in Fig. 16 of Ref.~\cite{evans_graphs_2016}). Here, we will argue against this conclusion. The discrepancy between our result and the claim of Ref.~\cite{evans_graphs_2016} comes from the fact that Ref.~\cite{evans_graphs_2016} did not consider the relationship between the data obtained from applying do-interventions on different subsets of the set of visible nodes. 

Secondly, we consider a probing scheme that is even weaker than the one just described,  
 now allowing only for do-interventions that set a visible variable to \emph{one} of its allowed values, called \emph{one-value do-interventions}. Here, one can imagine an example of a test of the effect of smoking on developing lung cancer. While the experimentalist can request people to stop smoking, thus setting the value of the ``smoker'' variable to $0$, they cannot ethically request people to start smoking, so they are not allowed to make an intervention that sets the variable ``smoker'' to the value $1$. When one partitions the ensemble of samples and applies one-value do-interventions on all possible subsets of visible nodes, we say that one is implementing the \emph{all-patterns Observe-or-1Do} probing scheme. As it turns out, we prove that two causal structures are \emph{also} indistinguishable under the \emph{all-patterns Observe-or-1Do} probing scheme if and only if they correspond to the same mDAG.

Apart from solving this indistinguishability problem, in this paper we also fully characterize when one causal structure can realize all the sets of data that are realizable by another causal structure for each of  the three different types of probing schemes studied here.  In this case, we say that the first causal structure \emph{dominates} the second relative to the probing scheme in question. As it turns out, the  characterization of the dominance relations among causal structures, like the characterization of the equivalence relations,  is the same for the three types of probing schemes studied here, and it is determined by the structure of the corresponding mDAGs.

For the case of causal structures with {\em no} latent variables, the problem of what causal structures can be discriminated using passive observation alone or using do-interventions on some nodes and passive observations on others has been studied in many previous works~\cite{Verma_Markov_classes,Andersson_Markov,interventional_Markov}.  We consider what our results imply for the special case of causal structures with no latent variables, and we comment on how these implications relate to previous results.   

In summary, in this work we show that two causal structures that are associated with the same mDAG are indistinguishable even when there is access to an informationally complete probing scheme, which corresponds to a very strong experimental power. Furthermore, two causal structures that are associated with different mDAGs can still be distinguished by the all-patterms Observe-or-1Do probing scheme, which corresponds to a much weaker experimental power.

The structure of the paper is as follows: In Section~\ref{sec_preliminaries}, we will give the background about causal modelling that will be necessary for our examples and results. In Section~\ref{sec_full_SWIGs}, we will discuss the full-SWIG, a structure that will be central to frame most of our discussion. In Section~\ref{sec_probingschemes}, we will discuss in more detail the notion of different probing schemes and the information revealed by these. In Section~\ref{sec_ETT}, we will present our first main result: the mDAG structure completely captures equivalence and dominance with respect to any informationally complete probing scheme, of which the Observe\&Do probing scheme is an example. In Section~\ref{sec_do}, we present our second and third main results. The second main result says that by restricting ourselves to the all-patterns Observe-or-Do probing scheme, we do not lose any power in discriminating causal structures. The third main result strengthens this by showing that even restricting ourselves to the all-patterns Observe-or-1Do probing scheme does not cost any discriminating power. In other words, the mDAG structure also completely captures equivalence and dominance with respect to the all-patterns Observe-or-Do and the all-patterns Observe-or-1Do probing schemes. In Section~\ref{sec_examples}, we provide some concrete examples of the partial order of equivalence classes of causal structures that is implied by the dominance relation induced by any of the three probing schemes considered here (as noted above, the dominance relation is the same for all three).  Specifically, we describe the partial orders for all causal structures over three visible variables and all causal structures over four visible variables. In Section~\ref{sec_confounder_free}, we discuss the special cases of causal structures that do not have latent variables and of causal structures that do not have connections between the visible variables, relating them to previous literature.
 Finally, in Section~\ref{sec_conclusion}, we describe some open problems and future directions for research.

\section{Preliminaries}
\label{sec_preliminaries}

A causal structure is represented by a directed acyclic graph (DAG), which is given by $\GpDAG=(\nodes(\GpDAG),\edges(\GpDAG))$, where $\nodes(\GpDAG)$ is a collection of nodes and $\edges(\GpDAG)$ is a collection of directed edges between these nodes. In this work, we will use lower case letters to indicate nodes, and upper case letters to indicate sets of nodes.

For $u,v\in \nodes(\GpDAG)$, we say that $u$ is a \emph{parent} of $v$, and $v$ is a \emph{child} of $u$ in the DAG $\GpDAG$ if the edge $u\rightarrow v$ appears in $\GpDAG$. The set of all parents of a node $a\in \nodes(\GpDAG)$ in the DAG $\GpDAG$ is denoted by $\pa_\GpDAG(a)$, and the set of all its children is denoted by $\ch_\GpDAG(a)$. Nodes with no parents in the DAG are termed {\em exogenous}, while those that do have parents in the DAG are termed {\em endogenous}.

Some nodes of a causal structure are probed in experiments, while others are not. As noted earlier, we refer to this distinction with the terms \emph{visible} and \emph{latent}.  In order to include this distinction as part of the specification of a causal structure, we introduce the notion of a \emph{partitioned DAG} (pDAG):\footnote{The notion of a pDAG is the same as the notion of a ``latent structure'' defined by Pearl (Definition 2.3.2 of Ref.~\cite{causality_pearl}). Here we opted for the term ``pDAG'' since it is even-handed, putting latent and visible nodes on an equal footing. Also, this term allows for expressions such as ``latent-free pDAG''.}

\begin{definition}[Partitioned DAG]
	A \emph{partitioned DAG (pDAG)} $\GpDAG$ is a DAG together with a partition of its nodes into two subsets: the set of \emph{visible nodes}, denoted $\vis(\GpDAG)$ and the set of \emph{latent nodes}, denoted $\lat(\GpDAG)$, where $\vis(\GpDAG) \cup \lat(\GpDAG) = \nodes(\GpDAG)$ and $\vis(\GpDAG) \cap \lat(\GpDAG)= \emptyset$.
\end{definition}

In a pDAG $\GpDAG$, the set of all \emph{visible} parents of a node $a\in\nodes(\GpDAG)$ will be denoted by $\Vpa_\GpDAG(a)$, while the set of all \emph{latent} parents of $a$ will be denoted by $\Lpa_\GpDAG(a)$. Similarly, the sets of visible and latent children of $a$ are denoted by $\Vch_\GpDAG(a)$ and $\Lch_\GpDAG(a)$  respectively.  When depicting pDAGs in this article we follow the same convention as Refs.~\cite{richardson_SWIGs,biomolecular,nonlinear}, where
 visible nodes are circular with white backgrounds and latent nodes are circular with grey backgrounds. When a pDAG does not have any latent node, that is, when it is a basic DAG, it will be called a \emph{latent-free pDAG}.

The term {\em visible} will be used to refer both to nodes and to the variables associated to these nodes. Similarly for the term {\em latent}. The variable associated with a node $a$ will be denoted by $X_a$ (with a state space denoted by $\mathcal{X}_a$), while the set of variables associated with a set $S$ of nodes will be denoted by $X_S$ (with a state space denoted by $\mathcal{X}_S$). Here we will only deal with variables whose state spaces are finite and discrete; therefore, it suffices to specify the cardinality of $\mathcal{X}_a$ for each node $a$. We denote the cardinalities of the set of state spaces $\{\mathcal{X}_a\}_{a\in S}$ by $\vec c_S$. A value taken by the variable $X_S$ will be indicated by the lower case $x_S$.

 Consider a probability distribution $P(X_{\vis(\GpDAG)})$ over the set of visible variables of a pDAG $\GpDAG$.  The pDAG $\GpDAG$ is said to \emph{realize this distribution under passive observations} if there exists:
\begin{itemize}
	\item A choice of state spaces $\mathcal{X}_{\lat(\GpDAG)}$ (i.e. a choice of cardinalities $\vec{c}_{\lat(\GpDAG)}$)  for the latent variables;
	\item for each node $a\in \nodes(\GpDAG)$, a random variable $E_a$ (called an \emph{error variable})  whose state space is denoted by $\mathcal{E}_a$ and a distribution $P(E_a)$ from which its value is sampled; and
	\item  for each node $a\in \nodes(\GpDAG)$, a function $f_a:\mathcal{X}_{\pa_{\GpDAG}(a)}\times \mathcal{E}_a\rightarrow \mathcal{X}_a$
\end{itemize}
such that setting
\begin{equation}
	X_a=f_a(X_{\pa_{\GpDAG}(a)},E_a) 
	\label{eq_recursive}
\end{equation}
yields a probability distribution $Q(X_{\nodes(\GpDAG)})$,  i.e., 
\begin{align}
&Q(X_{\nodes(\GpDAG)}) \nonumber\\
&= \prod_{a\in\nodes(\GpDAG)}\; \sum_{e \in\mathcal{E}_a} \delta_{X_a,f_a(X_{\pa_{\GpDAG}(a)},E_a=e) } P(E_a=e)
\end{align}
 whose marginal over $X_{\vis(\GpDAG)}$ is $P(X_{\vis(\GpDAG)})$.

Note that the inclusion of the error variables allows for the nodes to respond probabilistically to their parents in $\GpDAG$, even though the functions $f_a$ are deterministic. These functions together with the error-variable distributions are jointly termed the {\em parameters}, and denoted:
\begin{equation} 
	\label{eq_param}
	\param =  \{ (f_a,P(E_a)) : a\in\nodes (\GpDAG)\} 
\end{equation}

A \emph{causal hypothesis} $(\GpDAG,\param)$ is a pDAG $\GpDAG$ together with a choice of parameters for $\GpDAG$. A  particular  causal hypothesis realizes one  particular  probability distribution over the visible variables, which we will denote by $P^{(\GpDAG,\param)}(X_{\vis(\GpDAG)})$.

For the purpose of brevity,  we will  say that a probability distribution is \emph{observationally realizable} by a pDAG if it is realizable under passive observations. The set of probability distributions over visible variables of cardinalities $\vec c_{\vis(\GpDAG)}$ that are observationally realizable by a pDAG $\GpDAG$ is called the \emph{marginal model} of $\GpDAG$ for cardinalities $\vec c_{\vis(\GpDAG)}$, and is denoted by $$\mathcal{M}_{\text{obs}}(\GpDAG,\vec c_{\vis(\GpDAG)}).$$

If it happens that \emph{all} of the probability distributions that are observationally realizable by a pDAG $\GpDAG$ are also observationally realizable by a different pDAG $\GpDAG'$ and vice-versa, then  it is impossible to know which one of these two pDAGs generated the statistical data by passive observations.  This idea is formalized by the definition of \emph{observational equivalence} of pDAGs, which is a special case of \emph{observational dominance}:\footnote{The notion of observational dominance of pDAGs is similar to Pearl’s notion of \emph{preference} among latent structures (see Definition 2.3.3 of Ref.~\cite{causality_pearl}), except that the preference order is opposite to the dominance order. That is, $\GpDAG$ is preferred to $\GpDAG'$ when $\GpDAG'$ observationally dominates $\GpDAG$.  Also, the preference order is defined for a fixed cardinality of the visible variables, whereas the dominance order defined here involves a universal quantifier over these cardinalities.}

\begin{definition}[Observational dominance and equivalence of pDAGs]
	\label{def_obs_equivalence}
	Let $\GpDAG$ and $\GpDAG'$ be two pDAGs such that $\vis(\GpDAG)=\vis(\GpDAG')$. We say that $\GpDAG$ \emph{observationally dominates} $\GpDAG'$ when the set of observationally realizable distributions of  $\GpDAG$ includes the set of observationally realizable distributions of  $\GpDAG'$, regardless of the assignment of cardinalities of the visible variables, i.e., when
	\begin{gather}
		\forall \vec c_{\vis(\GpDAG)} \in \mathbb{N}^{|\vis(\GpDAG)|}: \\ \mathcal{M}_{\text{obs}}(\GpDAG',\vec c_{\vis(\GpDAG')}) \subseteq   \mathcal{M}_{\text{obs}}(\GpDAG,\vec c_{\vis(\GpDAG)}) 
	\end{gather}
	The observational dominance relation is denoted by $\GpDAG\succeq\GpDAG'$.
	
	As a special case, we say that $\GpDAG$ is \emph{observationally equivalent} to $\GpDAG'$ when their sets of observationally realizable distributions are the same:
	\begin{gather}
		\forall \vec c_{\vis(\GpDAG)} \in \mathbb{N}^{|\vis(\GpDAG)|}: \\ \mathcal{M}_{\text{obs}}(\GpDAG',\vec c_{\vis(\GpDAG')}) =   \mathcal{M}_{\text{obs}}(\GpDAG,\vec c_{\vis(\GpDAG)}) 
	\end{gather}	
	The observational equivalence relation is denoted by $\GpDAG\cong\GpDAG'$. 
	
	If $\GpDAG\succeq\GpDAG'$ but $\GpDAG\not\cong\GpDAG'$,
	we say that $\GpDAG$ \emph{strictly observationally dominates} $\GpDAG'$ and denote this relation as $\GpDAG\succ\GpDAG'$.
	If $\GpDAG\not\succeq\GpDAG'$ and $\GpDAG'\not\succeq\GpDAG$, we say that $\GpDAG$ and $\GpDAG'$ are \emph{observationally incomparable}.
\end{definition}

Ref~\cite{evans_graphs_2016} presented two fundamental results under which pDAGs are known to be observationally equivalent. These results will be reproduced as lemmas, and are illustrated in Fig.~\ref{fig_example_lemmas}.

 The first of these results asserts that for any pDAG wherein one or more latent nodes are endogenous, there is another pDAG where all latent nodes are exogenous and that is observationally equivalent to the first.   The result is Lemma 3.7 in Ref.~\cite{evans_graphs_2016},  which we rephrase here: 
\begin{lemma}[Exogenize Latent Nodes]
	\label{lemma_exogenize_latents}
	Let $\GpDAG$ be a pDAG, and let ${\tt EndoLNodes}(\GpDAG)\subseteq\lat(\GpDAG)$ be the set of latent nodes of $\GpDAG$ that are endogenous (i.e., they have one or more parents in $\GpDAG$).  Construct the pDAG ${\tt Exog}(\GpDAG)$ as follows. For every $u\in {\tt EndoLNodes}(\GpDAG)$, start from $\GpDAG$ and: (i) add a directed edge from every parent of $u$ to every child of $u$, (i) delete  all directed edges leading into $u$. The pDAG ${\tt Exog}(\GpDAG)$ so constructed is observationally equivalent to $\GpDAG$, i.e., ${\tt Exog}(\GpDAG)\cong  \GpDAG$.
\end{lemma}

The second result  asserts that certain redundant latent nodes can be removed from the pDAG without affecting the distributions that can be realized. It is based on Lemma 3.8 of \cite{evans_graphs_2016}: 
\begin{lemma}[Eliminate Redundant Latent Nodes]
	\label{lemma_remove_redundant_latents}
	Let $\GpDAG$ be a pDAG where all latent nodes are exogenous.
	Let ${\tt RedundLnodes}(\GpDAG)\subset\lat(\GpDAG)$ be a maximal subset of latent nodes (maximality here implies that the subset cannot be made any larger) such that  their set of children is already common-cause connected by another latent node.  
	Note that there is ambiguity in the choice, since when two latent nodes have the same set of children, one can take either to be redundant to the other. Formally, a set of latent nodes of $\GpDAG$ satisfies the definition of ${\tt RedundLnodes}(\GpDAG)$ if
	(i) For every $u \in {\tt RedundLnodes}(\GpDAG)$ there exists some distinct latent node $v \in \lat(\GpDAG)\setminus  {\tt RedundLnodes}(\GpDAG)$ such that $\ch(u) \subseteq \ch(v)$. (ii) ${\tt RedundLnodes}(\GpDAG)$ is not a strict subset of any other set of latent nodes that obeys (i).  Let ${\tt RemoveRedund}(\GpDAG)$ be the pDAG constructed by removing from $\GpDAG$ every node in ${\tt RedundLnodes}(\GpDAG)$ (and the corresponding outgoing arrows from these). The pDAG ~${\tt RemoveRedund}(\GpDAG)$ so constructed is observationally equivalent to $\GpDAG$, i.e.,  
	${\tt RemoveRedund}(\GpDAG) \cong \GpDAG$. 
\end{lemma}

\begin{figure*}[t!]
	\begin{center}
		\includegraphics[width=0.8\textwidth]{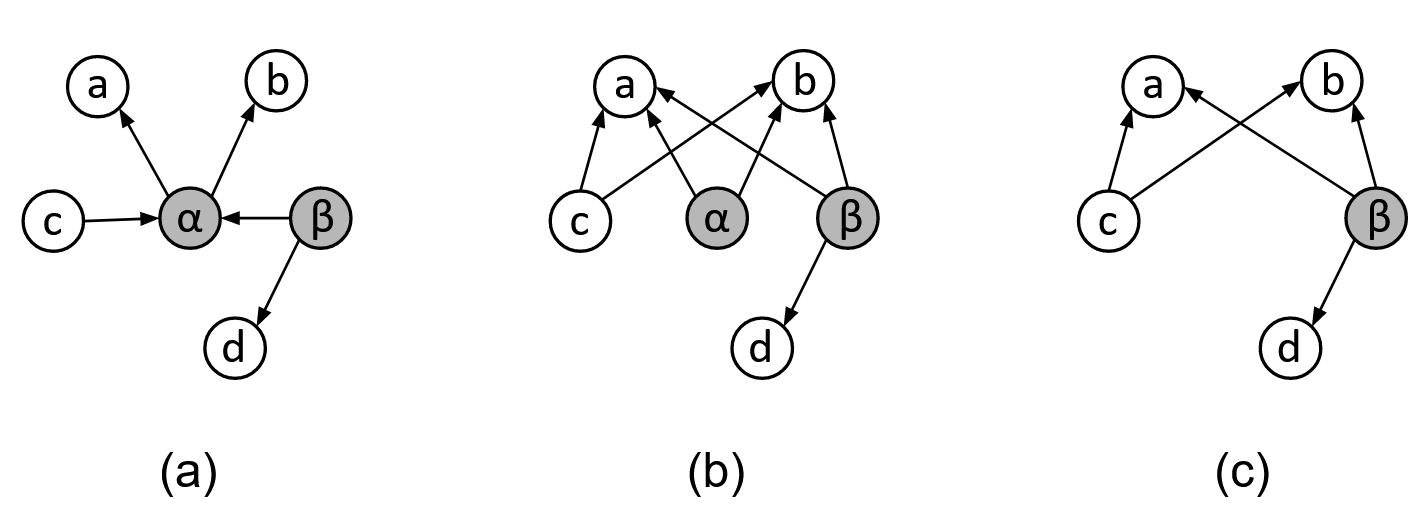}
	\end{center}
	\caption[.]{(a) A pDAG, with visible nodes in white and latent nodes in gray. (b) The pDAG obtained from (a) by exogenizing the latent node $\alpha$. (c) The pDAG obtained from (b) by removing the latent node $\alpha$, which is redundant to $\beta$ since the children of $\alpha$ are a subset of the children of $\beta$. By Lemmas \ref{lemma_exogenize_latents} and \ref{lemma_remove_redundant_latents}, these three pDAGs are observationally equivalent.   Because no further exogenization or removal of redundant latents is possible,  the pDAG in (c) is  the {\tt RE}-reduced pDAG for  the pDAG in (a). 
	}
	\label{fig_example_lemmas} 
\end{figure*}

As we will see in Section~\ref{sec_ETT}, two pDAGs that are related by application of the operations described in these lemmas (exogenizing latent nodes and removing redundant latent nodes) are not only indistinguishable relative to passive observations of the visible variables, but also  relative to more informative probing schemes, including those that are  informationally complete in a sense that we will define further on.

For now, let us simply note that if one is only interested in passive observations, then it is sufficient to consider candidate causal explanations where all of the latent nodes are exogenized and non-redundant. 
We define the map {\tt RE-reduce} via:
\begin{align}\label{REreductionmap}
{\tt RE{-}reduce} := {\tt RemoveRedund}\circ {\tt Exog}
\end{align}
We will refer to the image of a  pDAG $\GpDAG$ under the map {\tt RE{-}reduce}, i.e., the pDAG ${\tt RE{-}reduce}(\GpDAG)$, as the \emph{{\tt RE}-reduction} of $\GpDAG$. 
 Lemmas~\ref{lemma_exogenize_latents} and~\ref{lemma_remove_redundant_latents} imply that, if one only cares about passive observations, then it is enough to consider only {\tt RE}-reduced pDAGs to determine the scope of possibilities for the sets of realizable distributions. 
 
 \begin{lemma}\label{lemmaREreduction}
 Let $\GpDAG$ and $\GpDAG'$ be two pDAGs. If they have the same RE-reduction, i.e., ${\tt RE{-}reduce}(\GpDAG)={\tt RE{-}reduce}(\GpDAG')$, then they are observationally equivalent. 
 \end{lemma}
 The opposite implication does {\em not} hold. For instance, the pDAG wherein $a$ influences $b$ and the pDAG wherein $a$ and $b$ are confounded by a latent common cause are observationally equivalent but they do not have the same RE-reduction.

Inspired by lemma~\ref{lemmaREreduction}, Ref.~\cite{evans_graphs_2016}  introduced the notion of a \emph{marginalized DAG} (mDAG). (As we will see, the possible mDAGs are associated one-to-one with the possible {\tt RE}-reduced pDAGs.)
The notion of an mDAG makes use of the concept of a \emph{simplicial complex}: 
\begin{definition}[Simplicial complex]
	A simplicial complex over a finite set $V$ is a set $\mathcal{B}$ of subsets of $V$ such that
	\begin{itemize}
		\item For every  element of $V$, the singleton set containing that element is in $\cal B$, i.e., $\{v\}\in \mathcal{B}$ for all $v\in V$;
		\item If $A\subseteq B\subseteq V$ and $B\in \mathcal{B}$, then $A\in \mathcal{B}$.
	\end{itemize}
	The elements of $\mathcal{B}$ are called faces. The inclusion-maximal elements of a simplicial complex (the faces that are maximal in the order over faces induced by subset inclusion) are called facets.
\end{definition}

An mDAG is a pair $\GmDAG=(\cal D,B)$, where $\cal D$ is a DAG and $\cal B$ is a simplicial complex over the set of nodes of $\cal D$. Here, we will use the font $\GmDAG$ to denote mDAGs, while the font $\GpDAG$ continues being used to denote DAGs and pDAGs.

With this notion in hand, we now explain how to map pDAGs to mDAGs: 
\begin{definition}[The map $\mdag$ taking pDAGs to mDAGs]
	Let $\GpDAG$ be a  pDAG, and let $\tilde{\GpDAG}$ be the {\tt RE}-reduction of $\GpDAG$, i.e., 
	$\tilde{\GpDAG} = {\tt RE-reduce}(\GpDAG)$.
	If $\GmDAG= \mdag({\GpDAG})$, then $\GmDAG$ is the mDAG  $\mathcal{(D,B)}$ where 
	\begin{itemize}
		\item $\mathcal{D}$ is a DAG such that $\nodes(\mathcal{D})=\vis(\GpDAG)$ and whose edges correspond to the edges between the visible nodes in $\tilde{\GpDAG}$. $\mathcal{D}$ is called the \emph{directed structure} of the mDAG.
		\item $\mathcal{B}$ is a simplicial complex over $\vis(\GpDAG)$.        The facets of $\mathcal{B}$ are the maximal subsets of $\vis(\GpDAG)$ that 
		are children of the same latent node in $\tilde{\GpDAG} $, i.e., for each $u\in \lat(\tilde{\GpDAG} )$, the set $\ch_{\tilde{\GpDAG} }(u)$ 
		is a facet of $\mathcal{B}$. 
	\end{itemize}

The nodes of $\cal D$ will also be referred to as the nodes of $\GmDAG$. The edges of $\cal D$ will also be referred to as the directed edges of $\GmDAG$, and denoted ${\tt DirectedEdges}(\GmDAG)$. The faces of the simplicial complex $\cal B$ will also be denoted by ${\tt Faces}(\GmDAG)$  
\label{defn:LNodesToFaces}
\end{definition}

One can conceptualize the mDAG as a hypergraph with two types of edges, namely, the directed edges and a set of undirected hyperedges, where the latter represent the facets of the simplicial complex. In all of the depictions of mDAGs that we present here, the facets of the simplicial complex will be depicted by red loops. See Fig.~\ref{fig_triangle} for an example.

It is convenient to define a map that goes back from an mDAG to the associated {\tt RE}-reduced pDAG:

\begin{definition}[Canonical pDAG associated with an mDAG]
	Let $\GmDAG=(\mathcal{D},\mathcal{B})$ be an mDAG. The map $\can$ is given by the following procedure: starting from $\cal D$, add one latent node $l$ for each \emph{facet} $A\in \cal B$ and add edges from $l$ such that $\ch(l)=A$. The final pDAG, $\can(\GmDAG)$, is a {\tt RE}-reduced pDAG. It is called the \emph{canonical pDAG} associated with $\GmDAG$.
	\label{def_canonical}
\end{definition}

When the simplicial complex of the mDAG has at least one facet that includes nodes $a$ and $b$, then we say that there is a \emph{confounder} between $a$ and $b$ in the mDAG. 
When the simplicial complex of $\GmDAG$ is trivial in the sense of all its facets being singleton sets,
then $\GmDAG$ is said to be \emph{confounder-free}. 
 Clearly, $\GmDAG$ is confounder-free if and only if $\can(\GmDAG)$ is latent-free. 
When the directed structure of $\GmDAG$ is trivial in the sense of containing no edges, then $\GmDAG$ is said to be \emph{directed-edge-free}.

Another relation that will be useful later on is that of \emph{structural dominance of mDAGs}:
\begin{definition}[Structural Dominance relation between mDAGs]
	\label{def_structural_dominance}
	Let $\GmDAG$ and $\GmDAG'$ be two mDAGs with the same sets of nodes. $\GmDAG$ is said to \emph{structurally dominate}  $\GmDAG'$ if the following pair of conditions hold: (i) the directed structure of $\GmDAG'$ can be obtained from the directed structure of $\GmDAG$ by dropping edges, ${\tt DirectedEdges}(\GmDAG') \subseteq {\tt DirectedEdges}(\GmDAG)$, and  (ii) the simplicial complex of $\GmDAG'$ can be obtained from the simplicial complex of $\GmDAG$ by dropping faces, ${\tt Faces}(\GmDAG') \subseteq {\tt Faces}(\GmDAG)$. 
\end{definition}

Note that the Definition~\ref{def_structural_dominance} of structural dominance requires a subset inclusion relation for the {\em faces}, not the {\em facets} of the simplicial complexes. 
For example, in Fig.~\ref{fig_triangle}, the set of facets 
of the simplicial complex of the mDAG \ref{fig_triangle}(b) 
is not a subset of the set of facets of the simplicial complex  of \ref{fig_triangle}(a), but the set of {\em faces} do stand in a relation of subset inclusion to one another, so that mDAG \ref{fig_triangle}(a) structurally dominates mDAG \ref{fig_triangle}(b). 

\begin{figure}[htbp]
	\centering
	\includegraphics[width=0.47\textwidth]{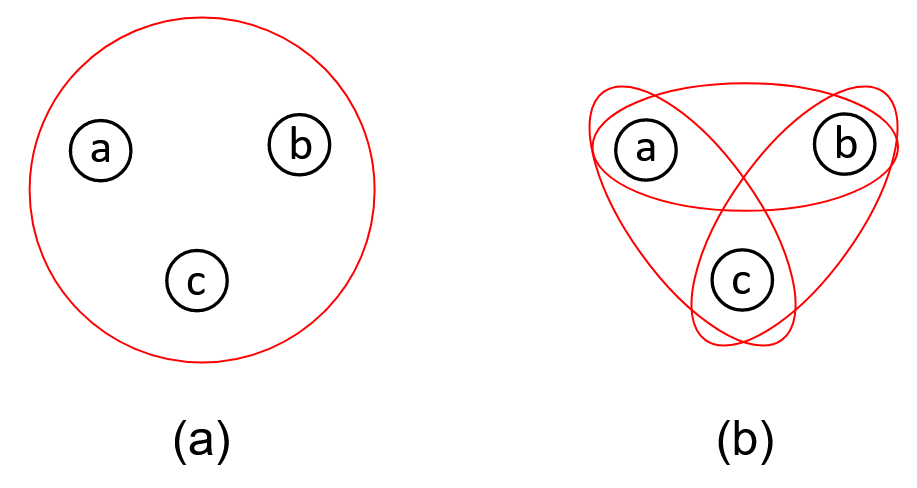}
	\caption{(a) The mDAG $(\mathcal{D},\mathcal{B})$ where the directed structure $\mathcal{D}$ is trivial and the simplicial complex is $\mathcal{B}=\{\{0\},\{1\},\{2\},\{0,1\},\{0,2\},\{1,2\},\{0,1,2\}\}$. (b)  The mDAG $(\mathcal{D}',\mathcal{B}')$  where the directed structure $\mathcal{D}'$ is again trivial and the simplicial complex is $\mathcal{B'}=\{\{0\},\{1\},\{2\},\{0,1\},\{0,2\},\{1,2\}\}$. Because $\mathcal{D}'\subset \mathcal{D}$ and $\mathcal{B}'\subset \mathcal{B}$, it follows from Definition~\ref{def_structural_dominance} that (a) structurally dominates (b). The facets (inclusion-maximal elements of $\mathcal{B}$ and $\mathcal{B'}$) are indicated by the red loops.} 
	\label{fig_triangle}
\end{figure}

 The following is a well-known result, following, for instance, from Proposition 3.3(b) in Ref.~\cite{evans_graphs_2016} or Theorem 26.1 in Ref.~\cite{henson_theory-independent_2014}): 
 \begin{lemma}\label{lemma_structural_obs_dominance}
 Let $\GmDAG$ and $\GmDAG'$ be two mDAGs. If $\GmDAG$ structurally dominates $\GmDAG'$ then $\GmDAG$ observationally dominates $\GmDAG'$.
 \end{lemma}
In Sec.~\ref{sec_ETT}, it will be shown that the opposite implication does {\em not} hold. See, e.g., Figs.~\ref{2node_structural} and \ref{2node_observational}. Furthermore, it will be shown that structural dominance \emph{also} implies dominance relative to informationally complete probing schemes.

To end this section, we will introduce an object that will be useful later on. Imagine a setup where  the values of  some of the variables are not sampled from a distribution that arises from natural causal mechanisms, but  are instead set
 by an experimentalist.  These variables can be conceptualized as {\em inputs}, whereas conventional variables (visible or latent) are outputs.  One can think, for example, of experiments involving interventions, or simply of experiments where some variables are settings that we choose, such as in an experimental test of Bell inequalities~\cite{Clauser_freechoice,Bell_Freevariables}. To describe the causal mechanism of such an experimental scenario, we need another structure, that allows for a third type of node. This structure will be called a \emph{3-pDAG}:
\begin{definition}[3-Partitioned DAG]
	A \emph{3-partitioned DAG (3-pDAG)} $\GthreepDAG$ is a DAG together with a partition of its nodes into \emph{three} subsets: the set of \emph{visible} nodes, denoted $\vis(\GthreepDAG)\subseteq \nodes(\GthreepDAG)$, the set of \emph{latent} nodes, denoted $\lat(\GthreepDAG)\subseteq \nodes(\GthreepDAG)$, and the set of \emph{input} nodes, denoted   $\repr(\GthreepDAG)\subseteq \nodes(\GthreepDAG)$.  Input nodes are exogenous nodes.   
	\label{defn:threepDAG}
\end{definition}
The visible and latent nodes have the same meaning as in pDAGs. The input nodes will be associated with \emph{input variables}, the variables that are 
set by an experimenter rather than sampled from a distribution. As such, the 3-pDAG defines a conditional probability distribution where the variables associated to input nodes are on the right of the conditional. 
   In a 3-pDAG $\GthreepDAG$, the subset of parents of a node $a$ that are input nodes will be denoted $\Rpa_\GthreepDAG(a)$.   
In the depictions of 3-pDAGs in this article, visible and latent nodes will continue to be represented by circles with white and grey backgrounds respectively, while input nodes will be represented as squares.

Similar to the notion of  observational  realizability of a probability distribution for pDAGs, we can define the \emph{  observational  realizability of a conditional distribution by a 3-pDAG}.  Here, the input variables  always appear on the right of the conditional.   A conditional probability distribution $P(X_{\vis(\GthreepDAG)}|X_{\repr(\GthreepDAG)})$ is said to be \emph{observationally\footnote{The use of the term ``observational'' here indicates that the data of interest, i.e., the conditional distributions $P(\vis(\GthreepDAG)|\repr(\GthreepDAG))$, are obtained by a probing scheme where the \emph{visible variables} of the 3-pDAG are passively observed.} realizable} by a 3-pDAG $\GthreepDAG$ if there exist parameters 
\begin{equation} 
	\label{eq_param3pDAG}
	\param =  \{ (f_a,P(E_a)) : a\in\nodes (\GthreepDAG)\setminus \repr(\GthreepDAG)\} 
\end{equation}
such that, after setting Eq.~\eqref{eq_recursive} for all nodes $ a\in\vis (\GthreepDAG) \cup  \lat(\GthreepDAG)$,
 one obtains a conditional probability distribution $Q(X_{\vis (\GthreepDAG) \cup  \lat(\GthreepDAG) }| X_{\repr(\GthreepDAG)})$, i.e., 
\begin{align}
&Q(X_{\vis (\GthreepDAG) \cup  \lat(\GthreepDAG)} | X_{\repr(\GthreepDAG)}) \nonumber\\
&= \prod_{a\in \vis (\GthreepDAG) \cup  \lat(\GthreepDAG)}\; \sum_{e \in\mathcal{E}_a} \delta_{X_a,f_a(X_{\pa_{\GthreepDAG}(a)},E_a=e) } P(E_a=e)
\end{align}
 whose marginal over $X_{\vis(\GthreepDAG)}$ is 
$P(X_{\vis(\GthreepDAG)}|X_{\repr(\GthreepDAG)})$.

Suppose the visible variables of a 3-pDAG $\GthreepDAG$ have cardinalities $\vec c_{\vis(\GthreepDAG)}$ and the input variables have cardinalities $\vec c_{\repr(\GthreepDAG)}$.  The set of conditional probability distributions of the visible variables given the input variables that are observationally realizable by the 3-pDAG $\GthreepDAG$ is denoted by $$\mathcal{M}_{\text{obs}}(\GthreepDAG,\vec c_{\vis(\GthreepDAG)},\vec c_{\repr(\GthreepDAG)}).$$ 

In analogy to the definition for pDAGs, we can define a notion of observational equivalence for 3-pDAGs.
We will say that two 3-pDAGs are \emph{observationally equivalent} if, for any choice of cardinalities for the visible variables, their set of  observationally  realizable conditional distributions are the same  Similarly, a 3-pDAG $\GthreepDAG$ \emph{observationally dominates} another 3-pDAG $\GthreepDAG'$ if, for any choice of cardinalities for the visible variables, its set of observationally realizable distributions is a superset of that of $\GthreepDAG'$.  In short, the definition of observational dominance and equivalence for 3-pDAGs is obtained from Definition~\ref{def_obs_equivalence} by replacing every instance of the pDAGs $\GpDAG$, $\GpDAG'$  by the 3-pDAGs $\GthreepDAG$, $\GthreepDAG'$, the realizable distributions $\mathcal{M}_{\text{obs}}(\GpDAG,\vec c_{\vis(\GpDAG)})$ by the realizable {\em conditional} distributions $\mathcal{M}_{\text{obs}}(\GthreepDAG,\vec c_{\vis(\GthreepDAG)},\vec c_{\repr(\GthreepDAG)})$, and the universal quantification over the vector of cardinalities $\vec c_{\vis(\GpDAG)}\in \mathbb{N}^{|\vis(\GpDAG)|}$ by the universal quantification over both the vector of cardinalities $\vec c_{\vis(\GthreepDAG)}\in \mathbb{N}^{|\vis(\GthreepDAG)|}$ and the vector of cardinalities $\vec c_{\repr(\GthreepDAG)}\in \mathbb{N}^{|\repr(\GthreepDAG)|}$.

It turns out that one can easily leverage the characterization of observational dominance in pDAGs to obtain the characterization of observational dominance in 3-pDAGs.  

Suppose $\GthreepDAG$ is a 3-pDAG.  Now consider the operation of converting all of the input nodes of $\GthreepDAG$ to visible nodes, so that the result has only visible and latent nodes and consequently is a pDAG.  Denote the map that achieves this conversion by ${\tt ConvertItoV}$. Denote the image of the 3-pDAG $\GthreepDAG$ under the ${\tt ConvertItoV}$ map, which is a pDAG, by $\bar{\GpDAG}$: 
$$\bar{\GpDAG} = {\tt ConvertItoV}(\GthreepDAG).$$

\begin{lemma} \label{lemma_convertItoV_distr}
	Let $\GthreepDAG$ be a 3-pDAG and let $\bar{\GpDAG}$ be the pDAG one obtains from it by converting input nodes into visible nodes, i.e., $\bar{\GpDAG}={\tt ConvertItoV}(\GthreepDAG)$. Let ${\tt newVnodes}(\bar{\GpDAG})$
	 be the set of nodes that were input nodes in $\GthreepDAG$ but have become visible in $\bar{\GpDAG}$, and let ${\tt oldVnodes}(\bar{\GpDAG})$
	  be the set of nodes that are visible in both $\GthreepDAG$ and $\bar{\GpDAG}$. That is, ${\tt newVnodes}(\bar{\GpDAG})=  \repr(\GthreepDAG)$, ${\tt oldVnodes}(\bar{\GpDAG})=  \vis(\GthreepDAG)$, and $\vis(\bar{\GpDAG})={\tt newVnodes}(\bar{\GpDAG})\cup {\tt oldVnodes}(\bar{\GpDAG})$. 
	
	 A conditional distribution $P(X_{\vis(\GthreepDAG)}|X_{\repr(\GthreepDAG)})$ is realizable by $\GthreepDAG$ if and only if the corresponding conditional distribution $P(X_{{\tt oldVnodes}(\bar{\GpDAG})}|X_{{\tt newVnodes}(\bar{\GpDAG})})$ is realizable by $\bar{\GpDAG}$.
\end{lemma}
\begin{proof}
	\begin{sloppypar} If the conditional distribution $P(X_{\vis(\GthreepDAG)}|X_{\repr(\GthreepDAG)})$ is observationally realizable by $\GthreepDAG$, then for every product probability distribution $ \prod_{a\in {\tt newVnodes(\bar{\GpDAG})}} P(X_a)$, 
	the joint distribution $P(X_{\vis(\bar{\GpDAG})})= P(X_{{\tt oldVnodes}(\bar{\GpDAG})}|X_{{\tt newVnodes}(\bar{\GpDAG})}) 
	\prod_{a\in {\tt newVnodes(\bar{\GpDAG})}} P(X_a)$
	is observationally realizable by $\bar{\GpDAG}$. This follows from the fact that every node $a \in {\tt newVnodes(\bar{\GpDAG})}$
	is {\em exogenous}  in $\bar{\GpDAG}$, so that the only restriction on the joint distribution over the collection of $X_a$ is that it be a product distribution. \end{sloppypar} 
	
	Similarly, if the joint distribution  $P(X_{\vis(\bar{\GpDAG})})$ is observationally realizable on  $\bar{\GpDAG}$, then because the nodes in ${\tt newVnodes(\bar{\GpDAG})}$ are all exogenous, it follows that the marginal distribution on ${\tt newVnodes(\bar{\GpDAG})}$ factorizes, i.e., $P(X_{{\tt newVnodes(\bar{\GpDAG})}}) = \prod_{a\in {\tt newVnodes(\bar{\GpDAG})}}P(X_a)$.  Consequently, $P(X_{\vis(\bar{\GpDAG})})$ can be expressed as 
	\begin{align}
		P(X_{\vis(\bar{\GpDAG})}) &= P(X_{{\tt oldVnodes}(\bar{\GpDAG})}|X_{{\tt newVnodes}(\bar{\GpDAG})}) \nonumber\\
		& \times \prod_{a\in {\tt newVnodes(\bar{\GpDAG})}} P(X_a).
	\end{align}
	Consider any conditional $P(X_{{\tt oldVnodes}(\bar{\GpDAG})}|X_{{\tt newVnodes}(\bar{\GpDAG})}) $ that can be obtained from a joint distribution $P(X_{\vis(\bar{\GpDAG})})$ that is observationally realizable in $\bar{\GpDAG}$ by dividing by some product distribution $\prod_{a\in {\tt newVnodes(\bar{\GpDAG})}} P(X_a)$. Clearly, it can be used to define a conditional $P(X_{{\tt Vnodes}(\GthreepDAG)}|X_{{\tt Inodes}(\GthreepDAG)})$ which is observationally realizable in $\GthreepDAG$. 
\end{proof}

Finally, we define the notion of a \emph{3-mDAG}, which stands to the notion of an mDAG in the same manner in which the notion of a 3-pDAG stands to the notion of a pDAG.  That is, a 3-mDAG is defined similarly to an mDAG but where the set of nodes can include input nodes (defined in the text below Definition~\ref{defn:threepDAG}) in addition to visible nodes.  These input nodes have no parents in the directed structure, nor can they be part of any nontrivial face of the simplicial complex.  Structural dominance of 3-mDAGs can be defined precisely as it was for mDAGs (see Definition \ref{def_structural_dominance}).

 \begin{lemma}\label{lemma_3pDAGs_structural}
	Let $\GthreemDAG$ and $\GthreemDAG'$ be two 3-mDAGs. If $\GthreemDAG$ structurally dominates $\GthreemDAG'$ then $\GthreemDAG$ observationally dominates $\GthreemDAG'$. 
\end{lemma}
\begin{proof}
	Follows straightforwardly from Lemmas~\ref{lemma_structural_obs_dominance} and~\ref{lemma_convertItoV_distr}.
\end{proof}

Just as we defined the map $\mdag(\cdot)$ from a pDAG to its associated mDAG in Definition \ref{defn:LNodesToFaces}, one can define the map that takes a 3-pDAG to its associated 3-mDAG in an analogous way.  We will again denote this map by $\mdag(\cdot)$.

\section{Full-SWIGs}
\label{sec_full_SWIGs}

Suppose now that we have the experimental power to intervene on a variable $X_a$, in such a way that the value of $X_a$ that arises from natural causal mechanisms is not necessarily the same as the value of $X_a$ that will influence the descendants of the node $a$ in the underlying causal structure. These two values will be kept in two new variables that we define, called $X_{a^\flat}$ and $X_{a^\sharp}$. The variable $X_{a^\flat}$ corresponds to the natural value of $X_a$, while the variable $X_{a^\sharp}$ corresponds to the value that will influence the descendants of $a$. 

We can represent those variables in a 3-pDAG that accounts for the possibility of an intervention on a variable $X_a$. There, the variable $X_{a^\flat}$ will be associated with a visible node that has no children, and the variable $X_{a^\sharp}$ will be associated with an input node, whose information is sent to the nodes that were descendants of $a$ in the original pDAG. An example is provided in Fig.~\ref{fig_def_flatsharp}.\footnote{Note that, if the nodes $a$ and $b$ of Fig.~\ref{fig_def_flatsharp}(a) are associated respectively with a treatment variable and a recovery variable, then the conditional distribution $P(X_{a^\flat} X_b|X_{a^\sharp})$ that is realized by the SWIG of Fig.~\ref{fig_def_flatsharp}(b) with the parameters of the causal hypothesis corresponds to the \emph{Effect of the Treatment on the Treated (ETT)}~\cite{ETT_Shpitser}.} 

\begin{figure}[h!]
	\centering
	\includegraphics[width=0.45\textwidth]{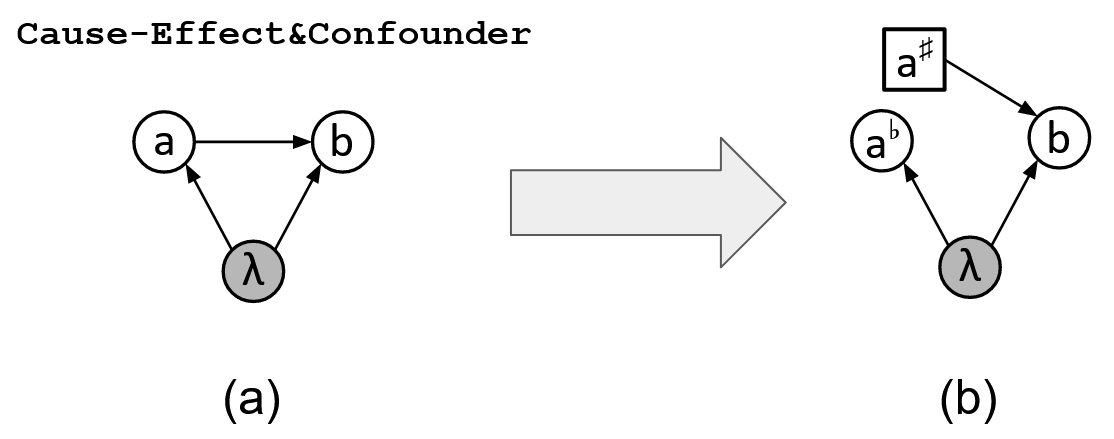}
	\caption{(a) The pDAG {\tt Cause-Effect\&Confounder}. (b) The SWIG that represents the possibility of intervening on the node $a$ of {\tt Cause-Effect\&Confounder}.}
    \label{fig_def_flatsharp}
\end{figure}

A 3-pDAG that is obtained by splitting   one or more visible nodes  of a pDAG into $\flat$ and $\sharp$ versions, like the one in Fig.~\ref{fig_def_flatsharp}, is called a \emph{single world intervention graph} (SWIG)~\cite{richardson_SWIGs}
\footnote{The concept of a SWIG  can also be useful when formulated using frameworks for causal modelling wherein causal structures are represented by circuits, i.e., each variable is represented by a wire and each causal mechanism is represented by a gate~\cite{SchmidOmelette,Lorenz2023}.
A latent variable is represented by a wire that is internal to the circuit.  The act of passively observing a variable is represented by applying a copy operation on the variable, which in the circuit implies creating a new wire to represent the copy and letting this be an open wire that is an output of the circuit. 
Conversely, an input variable is simply represented by an open wire that is an input to the circuit. 
It follows that the circuit framework has a particularly simple manner of representing the node-splitting operation of a visible variable: cut the wire corresponding to the copy of the variable that feeds forward to other variables in the causal model, and marginalize over the open output wire one thereby creates. 
In this way, one preserves the open output wire that was already present and one creates a new open input wire.}
 We will denote the SWIG obtained from $\GpDAG$ by splitting  all the nodes in a particular subset $A\subseteq \vis(\GpDAG)$ by $\text{SWIG}_A(\GpDAG)$.   The SWIG generated by a pDAG $\GpDAG$ by splitting \emph{all} visible nodes, $\text{SWIG}_{\vis(\GpDAG)}(\GpDAG)$, will be called a \emph{full-SWIG}, and will be denoted by $\spl(\GpDAG)$. Fig.~\ref{fig_splitnode_example} presents an example of a full-SWIG. Just as the $\flat$ and $\sharp$ versions of an individual node $a$ will be represented by $a^\flat$ and $a^\sharp$ respectively, the set of $\flat$ and $\sharp$ versions of the nodes in a set $S$ will be represented by $S^\flat$ and $S^\sharp$ respectively .

\begin{figure}[h!]
	\centering
	\includegraphics[width=0.45\textwidth]{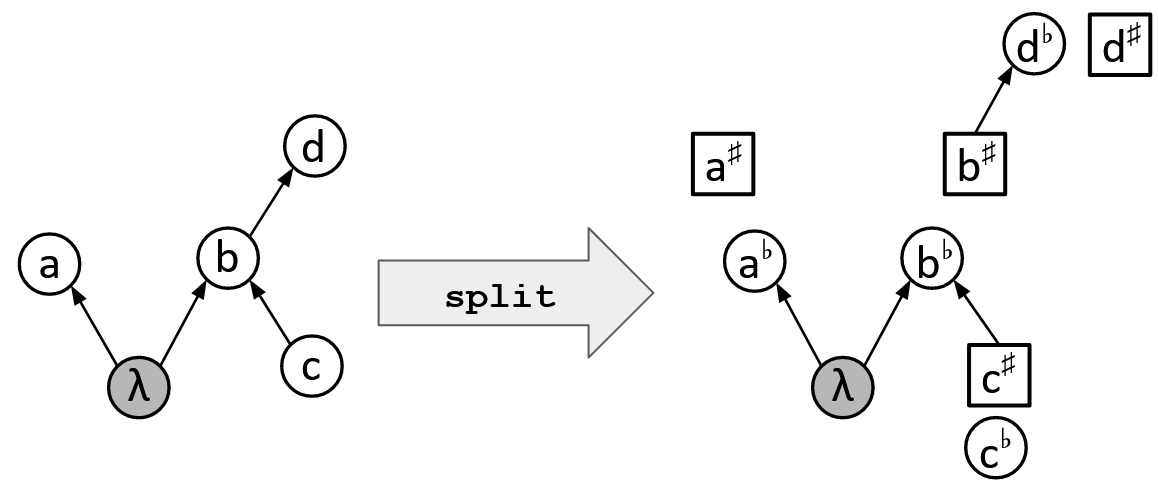}
	\caption{Example of the operation of splitting all visible nodes of a pDAG $\GpDAG$ to obtain the corresponding full-SWIG $\GthreepDAG=\spl(\GpDAG)$.
	}
	\label{fig_splitnode_example}
\end{figure}

We will also define a map that takes an mDAG and returns the 3-mDAG where all of the visible nodes of the original mDAG are split into $\flat$ and $\sharp$ versions. As well as for the case of pDAGs, this map will be denoted by $\spl$. In other words, we are extending the domain of the map $\spl$: now it can take pDAGs to 3-pDAGs, or mDAGs to 3-mDAGs.  An example of the application of the map $\spl$ on an mDAG is given in Fig.~\ref{fig_example_split_mDAG}, which starts from the mDAG that is associated with the pDAG of Fig.~\ref{fig_splitnode_example}.  

\begin{figure}[h!]
	\centering
	\includegraphics[width=0.45\textwidth]{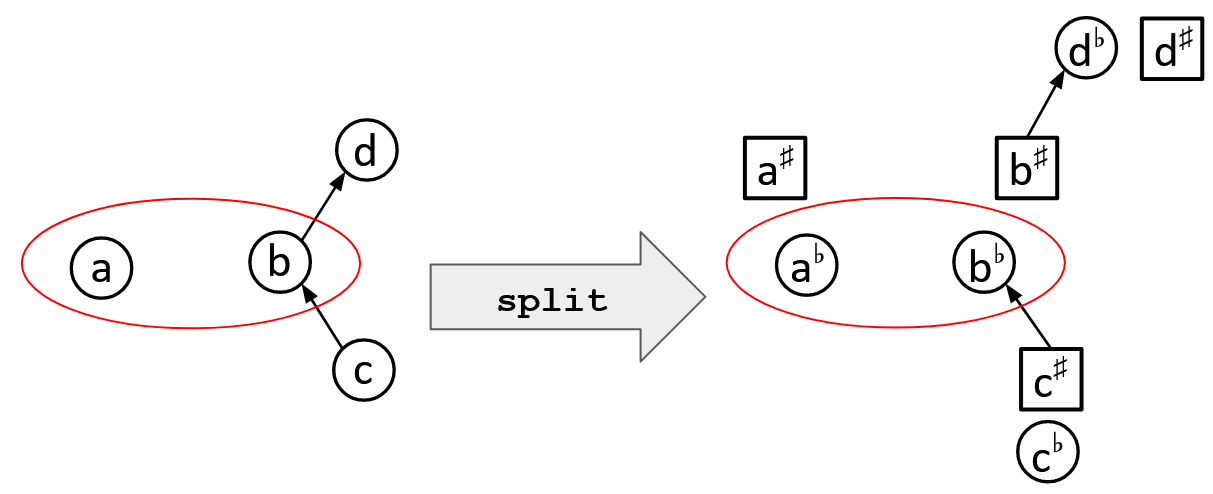}
	\caption{Example of the operation of splitting all nodes of an mDAG $\GmDAG$ to obtain the corresponding 3-mDAG $\GthreemDAG=\spl(\GmDAG)$.
	}
	\label{fig_example_split_mDAG}
\end{figure}

As it turns out, it does not matter whether we first split all the visible nodes of a pDAG $\GpDAG$ and then take the 3-mDAG of the resulting 3-pDAG $\spl(\GpDAG)$, or if we first take the mDAG associated to $\GpDAG$ and then split all of the visible nodes. This is formalized in the following lemma.

\begin{lemma}
	\label{lemma_commutation_maps}
	The maps $\spl$ and $\mdag$ commute. That is, for a pDAG $\GpDAG$, we have
	\begin{align}
		&\spl(\mdag(\GpDAG))\nonumber \\&=\mdag(\spl(\GpDAG)).
	\end{align}
\end{lemma}
 
\begin{proof}
	Given in Appendix~\ref{appendix_proof_Lemmacommutation}.
\end{proof}

For full-SWIGs, one asks about the realizability of a given conditional probability distribution 
\begin{equation}
	P(X_{\vis(\GpDAG)^\flat}|X_{\vis(\GpDAG)^\sharp}). 
	\label{eq_flat_given_sharp}
\end{equation}
To answer this realizability question, we need to specify the possible choices of parameters on the full-SWIG, that is, the scope of causal hypotheses.  Recall that any 3-pDAG has parameters as specified in Eq.~\eqref{eq_param3pDAG}, that is, functions and distributions over error variables for all the nodes that are visible. 
The scope of possible parameter values for the full-SWIG is presumed to be inherited from the pDAG from which the full-SWIG was defined. Specifically, a variable $X_{a^\flat}$ depends functionally on its parents and the local error variable $E_{a^\flat}$ in the same way that $X_{a}$ depended on its parents and $E_a$, and the distribution over the error variable is the same.
That is, for each causal hypothesis $(\GpDAG,\param)$ where $\param = \{ f_a, P(E_a) : a \in \vis(\GpDAG)\}$, we have a corresponding causal hypothesis for the full-SWIG $\spl(\GpDAG)$, namely,
\begin{align}
\param^{\tt split} = \{ f_{a^\flat}, P(E_{a^\flat}) : {a^\flat} \in \vis(\GpDAG)^{\flat}\}
\end{align}
where $f_{a^\flat} =f_a$ and $P(E_{a^\flat})=P(E_a)$.  Note that there are no error variables associated with the $\sharp$ nodes because these are input nodes.

As noted earlier, 
there are instances of observational dominance of pDAGs without structural dominance, that is, that the converse of Lemma~\ref{lemma_structural_obs_dominance} does not hold. 
  By contrast, the observational dominance of full-SWIGs that have the same set of visible and input nodes \emph{can} be fully characterized by structural dominance:

\begin{lemma}
	Let $\GthreepDAG$ and $\GthreepDAG'$ be a pair of 
	full-SWIGs having the same set of visible and of input nodes,
	$\vis(\GthreepDAG)=\vis(\GthreepDAG')$	and $\repr(\GthreepDAG)=\repr(\GthreepDAG')$, and let $\GthreemDAG$ and $\GthreemDAG'$ denote the corresponding 3-mDAGs, i.e., $\GthreemDAG= \mdag(\GthreepDAG)$ and $\GthreemDAG'= \mdag(\GthreepDAG')$.

	Then, the full-SWIG $\GthreepDAG$ observationally dominates the full-SWIG $\GthreepDAG'$ if and only if the 3-mDAG $\GthreemDAG$
	 structurally dominates the 3-mDAG $\GthreemDAG'$. 
	\label{lemma_full_SWIGdominance}
\end{lemma} 
\begin{proof}	
	The ``if'' side follows from Lemma~\ref{lemma_3pDAGs_structural}, so we only need to prove the ``only if'' side. We will proceed by proving the contrapositive.  Suppose that ${\GthreepDAG}$ observationally dominates ${\GthreepDAG}'$ even though ${\GthreemDAG}$ does not structurally dominate ${\GthreemDAG}'$. The lack of structural dominance implies one of two possibilities:
	\begin{enumerate}
		\item There is at least one directed edge between an input node $a^\sharp$ and a visible node $b^\flat$,  $a^\sharp\rightarrow b^\flat$, that is present in ${\GthreemDAG}'$ but not in ${\GthreemDAG}$, or
		\item  ${\GthreemDAG}$ has all the directed edges that are present in ${\GthreemDAG}'$, but there is at least one set $S^\flat$ of visible nodes which is a face of ${\GthreemDAG}'$ but not of ${\GthreemDAG}$.
	\end{enumerate}
	
	These two conditions for 3-mDAGs respectively imply the following conditions for the corresponding 3-pDAGs:
	\begin{enumerate}
		\item There is at least input node $a^\sharp$ and one visible node $b^\flat$ such that ${\GthreepDAG}'$ presents a chain $a^\sharp\rightarrow m_1 \rightarrow ... \rightarrow m_n\rightarrow b^\flat$ where all of the mediary nodes $m_1,...,m_n$ are latent nodes, and there is \emph{no} such chain in ${\GthreepDAG}$, or
		\item Item 1 does not hold, but there is at least one set $S^\flat$ of visible nodes that have a common latent ancestor $\lambda$ in ${\GthreepDAG}'$, but no such  latent ancestor exists in ${\GthreepDAG}$.
	\end{enumerate}

In Case 1, consider a conditional distribution where the value of $X_{b^\flat}$ depends on the value of $X_{a^\sharp}$, that is, there exist values $x_{a^\sharp},x'_{a^\sharp}\in \mathcal{X}_{a^\sharp}$ such that
\begin{equation}
	P(X_{b^\flat}|X_{a^\sharp}=x_{a^\sharp})\neq 	P(X_{b^\flat}|X_{a^\sharp}=x'_{a^\sharp}).
\end{equation}
Such a conditional distribution can clearly be realized by ${\GthreepDAG}'$ given the presence of a chain from $a^\sharp$ to $b^\flat$. However, it cannot be realized by $\GthreepDAG$. This is because there is no chain between $a^\sharp$ and $b^\flat$, and also no latent common cause of $a^\sharp$ and $b^\flat$. This can be seen as follows. Since $a^\sharp$ is a $\sharp$ node of the full-SWIG $\GthreepDAG$, it must be \emph{parentless}, so there cannot be a common cause between $a^\sharp$ and $b^\flat$ nor a chain from $b^\flat$ to $a^\sharp$. By assumption, there is no chain from $a^\sharp$ to $b^\flat$ where all of the mediary nodes are latent; however, in a full-SWIG all of the nodes that are not latent are either parentless ($\sharp$ nodes) or childless ($\flat$ nodes), so chains with visible or input mediary nodes are also not allowed.

In case 2, consider a conditional distribution where all of the variables of $X_{S^\flat}$ are perfectly correlated; that is
\begin{equation}
	P(X_{S^\flat}|X_{\repr(\GthreepDAG)})=p[0,...,0]_{S^\flat}+(1-p)[1,...,1]_{S^\flat}
	\label{eq_perfect_corrSWIG}
\end{equation}
where $p\in[0,1]$. In the right hand side, $p[0,...,0]_{S^\flat}$ indicates that with probability $p$ all the variables in the set $X_{S^\flat}$ take the value $0$, and  $(1-p)[1,...,1]_{S^\flat}$ indicates that with probability $1-p$ all the variables in the set $X_{S^\flat}$ take the value $1$. The probabilities $p$ could depend on the value of $X_{\repr(\GthreepDAG)}$ that we are conditioning upon; all that matters is that all of the variables in $X_{S^\flat}$ are perfectly correlated.

Such a conditional distribution can clearly be realized by $\GthreepDAG'$, given the presence of a common cause $\lambda$ between all of the nodes of $S^\flat$. However, it cannot be realized by $\GthreepDAG$. 

Ref.~\cite[Example 2]{SteudelAy} says that, in a pDAG, a set of variables can only be perfectly correlated if all of the corresponding nodes share a common cause\footnote{Correlation could appear without a common cause if we were conditioning on a common effect. However, this is not the case: the nodes that are being conditioned upon are all input nodes, hence parentless, hence not candidates for a common effect of other nodes.}. Via Lemma~\ref{lemma_convertItoV_distr}, this result can be straightforwardly extended to 3-pDAGs: a set of visible nodes of a 3-pDAG can only be perfectly correlated given the input nodes if they share a common cause that is latent or visible (an input node cannot act as the common cause that establishes perfect correlation, since all input variables are conditioned upon in the data). By assumption, the nodes of the set $S^\flat$ do not all share a latent common cause in $\GthreepDAG$. They also cannot share a visible common cause since, in full-SWIGs, visible nodes do not have children. Thus, $\GthreepDAG$ cannot realize the conditional distribution of Eq.~\eqref{eq_perfect_corrSWIG}.  

Therefore, in both cases where $\GthreemDAG$ does not structurally dominate $\GthreemDAG'$, we have explicitly shown conditional probability distributions that are realizable by $\GthreepDAG'$ but not by $\GthreepDAG$.

\end{proof}

\section{Probing Schemes and Shadows}
\label{sec_probingschemes}

In this work, we study different probing schemes on visible variables with the goal of adjudicating between causal structures. We take the temporal order of the visible variables to be fixed and consider only causal structures that are consistent with this temporal ordering.\footnote{Such consistency holds  if and only if the temporal order corresponds to one of the possible topological orderings of the visible nodes of the pDAG.} 
 If two visible variables, $X_a$ and $X_b$, are temporally ordered such that $X_a$ comes before $X_b$, then this means that \blk in each sample of the ensemble, the variable $X_a$ is associated to a property of a system at one time and the variable $X_b$ is associated to a property of a system (possibly the same system) at a {\em later} time. In this case, although we might be unsure of whether the underlying causal structure has an arrow $a\rightarrow b$ or not, we can nonetheless be sure that there is \emph{no} arrow $b \rightarrow a$, under the assumption that there is no backwards-in-time causation. 
 
There are several reasons why we presume that the visible variables have a fixed temporal ordering.
Clearly, if the variables of interest are temporally localized (and we do not need to take into account relativity theory, such as the possibility of space-like separation), then they will necessarily be temporally ordered.  The reason typically given for why fixed temporal ordering should {\em not} be presumed~\cite[Section 7.5.1]{causality_pearl} is that the variables of interest may fail to be temporally localized, referring instead to properties that persist over time. An individual's health and their exercise regime provides a good example. However, whenever the variables of interest are of this type, they can have a mutual influence on one another, and consequently one cannot restrict attention to directed graphs that are {\em acyclic} when considering the possibilities for the causal structure holding between them. Thus, if one is contemplating a causal discovery algorithm that returns only acyclic graphs (as we are doing here), then one should not apply it to variables that fail to be temporally localized.  A second reason for restricting attention to variables that are temporally localized (and that consequently have a fixed temporal order) is that most probing schemes only make sense for such variables. In particular, it is difficult to make sense of the Observe\&Do probing scheme described in Section~\ref{sec_ETT} for a variable that refers to a property that persists over time.

Therefore, in this work we will only consider causal structures where the visible variables have a fixed temporal ordering. Fig.~\ref{fig_forkchain} shows an example of two causal structures, a fork and a chain, that are consistent with the temporal ordering $(a,b,c)$. Note that they are observationally {\em in}equivalent. (There exist pairs of causal structures, one of the ``fork'' type and the other of the ``chain'' type, that {\em are} observationally equivalent, but this requires 
  the parent node in the fork to correspond to the middle node in the chain, and thus the visible variables of the two structures cannot have the same temporal ordering.) 

\begin{figure}[h!]
	\centering
	\includegraphics[width=0.45\textwidth]{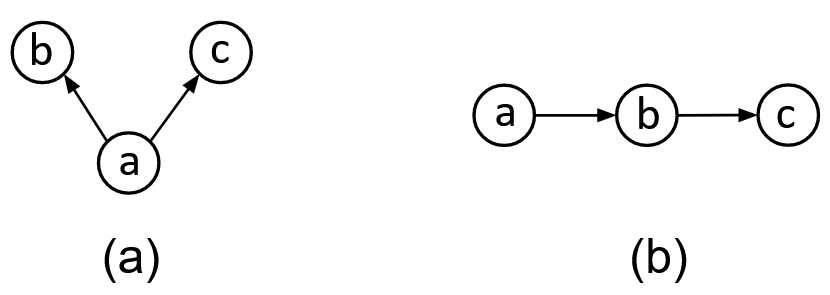}
	\caption{(a) Causal structure of the ``fork'' type that is consistent with the temporal ordering $(a,b,c)$. (b) Causal structure of the ``chain'' type that is consistent with the temporal ordering $(a,b,c)$. These two causal structures are \emph{not} observationally equivalent.}
	\label{fig_forkchain}
\end{figure}

The most we can hope to learn about a causal hypothesis $(\GpDAG,\param)$ by interacting with the visible variables via different probing schemes\footnote{In this paper we do \emph{not} consider edge interventions, which are interventions where one sends a different value of the intervened variable to each of the children of its corresponding node. In other words, when a variable $X_a$ is intervened upon, we will consider that all of the children of $a$ receive information about the same variable $X_{a^\sharp}$.} is expression~\eqref{eq_flat_given_sharp}, that is, the conditional distribution of  the $\flat$ variables given the $\sharp$ variables generated by the causal hypothesis. When a probing scheme is able to provide us the whole conditional distribution  $P(X_{\vis(\GpDAG)^\flat}|X_{\vis(\GpDAG)^\sharp})$, we say it is \emph{informationally complete.} In Section~\ref{sec_ETT}, we will discuss the distinguishability of causal structures under access to informationally complete probing schemes.
 
After characterizing the informationally complete probing schemes, one can also ask what can be done with probing schemes that are \emph{not} informationally complete. To describe them, we start by conceptualizing the conditional probability distribution ${P(X_{\vis(\GpDAG)^\flat}|X_{\vis(\GpDAG)^\sharp})}$
as a vector, where each component is given by a probability corresponding to a choice of values of $X_{\vis(\GpDAG)^\flat}$ and $X_{\vis(\GpDAG)^\sharp}$. A probing scheme that is not informationally complete will not give us the full conditional distribution ${P(X_{\vis(\GpDAG)^\flat}|X_{\vis(\GpDAG)^\sharp})}$, but only some of its components, or some functions of these components.   The set of functions of the components of the full conditional distribution that can be obtained from a probing scheme is said to be the \emph{shadow} of the full conditional that is revealed by this probing scheme.

To illustrate the notion of the shadow revealed by a probing scheme, we consider the example of the causal structure of  Fig.~\ref{fig_fullSWIG_parameters} and two different probing schemes applied to it.  To shorten our notation, here we will abbreviate the name of the pDAG {\tt Cause-Effect\&Confounder} to {\tt CE\&C}.

Suppose that a certain phenomenon is explained by the pDAG {\tt CE\&C} with a specific set of functional parameters
$\param=[f_a,f_b,P(X_\lambda)]$\footnote{Strictly speaking, the $\param$ that we defined in Eq.~\eqref{eq_param} involves the distributions over error variables. However, here we have absorbed all of the error variables into the latent variable $X_\lambda$.}, where $f_a$ and $f_b$ are functions
\begin{align}
	&f_a:\mathcal{X}_\lambda \rightarrow \mathcal{X}_a \\
	& f_b:\mathcal{X}_\lambda \times \mathcal{X}_a \rightarrow\mathcal{X}_b.
\end{align}

\begin{figure}[h!]
	\centering
	\includegraphics[width=0.45\textwidth]{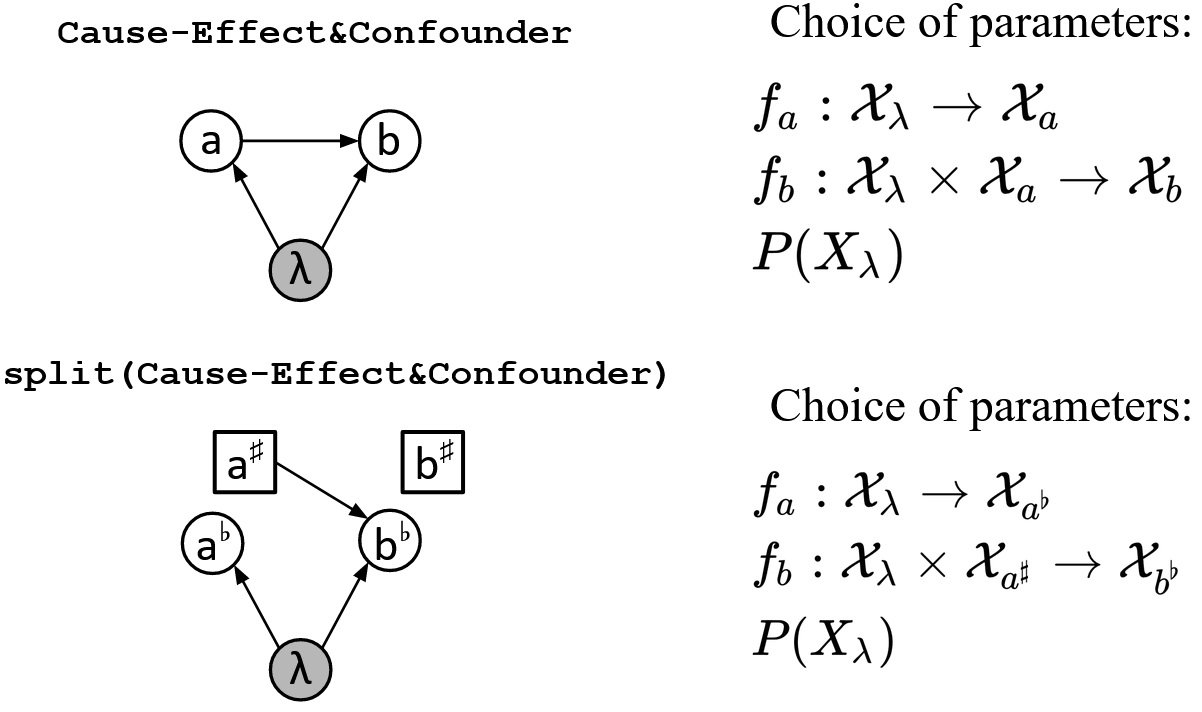}
	\caption{Example used to understand the notion of shadow of the full conditional distribution $P(X_{\vis(\GpDAG)^\flat}|X_{\vis(\GpDAG)^\sharp})$ that is revealed by a specific probing scheme.}
	\label{fig_fullSWIG_parameters}
\end{figure}

The full-SWIG $\spl(\text{{\tt CE\&C}})$ is presented in the bottom part of Fig.~\ref{fig_fullSWIG_parameters}. A choice of parameters for a full-SWIG corresponds to a choice of parameters for the original pDAG, because a variable $X_{a^\flat}$ in the full-SWIG reacts to its parents in the same way as $X_a$ reacts to its parents in the original pDAG. The same choice of parameters $\param$ for this full-SWIG\footnote{The error variables of $a^\flat$ and $b^\flat$ are similarly absorbed into $\lambda$, whose distribution is part of $\param$. Note that $a^\sharp$  and $b^\sharp$, being input variables of a 3-pDAG, do not have associated error variables.}, as indicated in the figure, gives rise to the conditional distribution $P^{(\spl({\tt CE\&C}),\param)}(X_{a^\flat} X_{b^\flat} | X_{a^\sharp} X_{b^\sharp})$. 

Consider first a probing scheme consisting of passive observations of $X_a$ and $X_b$. It is easy to see that the observational probability distribution realized by the original pDAG {\tt CE\&C} under the choice of parameters $\param$ takes the same value as the conditional probability distribution realized by the full-SWIG  $\spl(\text{{\tt CE\&C}})$ under the same choice of parameters  $\param$ \emph{when the values of the $\sharp$ variables coincide with the values of the corresponding $\flat$ variables}. That is,
\begin{gather}
	P^{({\tt CE\&C},\param)}(X_a=x_a, X_b=x_b)= \nonumber \\ P^{(\spl({\tt CE\&C}),\param)}(X_{a^\flat}=x_a, X_{b^\flat}=x_b | X_{a^\sharp}=x_a, X_{b^\sharp}=x_b).
	\label{eq_simple_splitnode_observational}
\end{gather}

Therefore, the shadow of the full conditional $P^{(\spl({\tt CE\&C}),\param)}(X_{a^\flat} X_{b^\flat} | X_{a^\sharp} X_{b^\sharp})$ that is revealed by passive observations of $X_a$ and $X_b$ is the set of components given by Eq.~\eqref{eq_simple_splitnode_observational}.

Consider now a different probing scheme, consisting of passive observation of $X_b$ and a do-intervention on $X_a$, which is a procedure that forces $X_a$ to take a specific value $x'_a$. The probability distribution over $X_b$ obtained after performing this do-intervention on $X_a$ is called a \emph{do-conditional}, denoted by $P(X_b|\text{do}(X_a=x'_a))$. 

The do-conditional obtained by applying this probing scheme on the causal hypothesis ({\tt Cause-Effect\&Confounder}, $\param$), denoted by $P^{({\tt CE\&C},\param)}(X_b|\text{do}(X_a))$, equals the corresponding conditional distribution $P^{({\tt CE\&C},\param)}(X_{b^\flat}|X_{a^\sharp})$ realized by $\spl(\text{{\tt CE\&C}})$ under the same choice of parameters $\param$. That is,

\begin{align}
	&P^{({\tt CE\&C},\param)}(X_b=x_b|\text{do}(X_a=x'_a))=  \nonumber \\
&\begin{aligned}
	\sum_{x_{a} } P^{(\spl({\tt CE\&C}),\param)}(& X_{a^\flat}= x_{a}, X_{b^\flat}=x_b | \\&|X_{a^\sharp}=x'_a, X_{b^\sharp}=x_b),
\end{aligned}
\label{eq_simple_splitnode_interventional}
\end{align}
where the variables $X_{b^\flat}$ and $X_{b^\sharp}$ take the same value and we marginalize over the variable $X_{a^\flat}$. This equation gives the shadow of the full conditional $P^{(\spl({\tt CE\&C}),\param)}(X_{a^\flat} X_{b^\flat} | X_{a^\sharp} X_{b^\sharp})$ that is revealed by a passive observation of $X_b$ together with a do-intervention that forces $X_a$ to take the value $x'_a$.

In summary, each probing scheme is associated with a \emph{shadowing function}. This is the function that, given a conditional distribution ${P(X_{\vis(\GpDAG)^\flat}|X_{\vis(\GpDAG)^\sharp})}$, returns the shadow revealed by the probing scheme in question. In the example of Fig.~\ref{fig_fullSWIG_parameters}, the shadowing function associated with passive observations of $X_a$ and $X_b$ corresponds to taking the components where $X_{a^\flat}=X_{a^\sharp}$ and $X_{b^\flat}=X_{b^\sharp}$. The shadowing function associated with passive observation of $X_b$ and a do-intervention setting $X_a$ to $x'_a$ corresponds to taking the components where $X_{a^\sharp}=x'_a$ and $X_{b^\flat}=X_{b^\sharp}$, and marginalizing out $X_{a^\flat}$.

With this, we can now define a notion of realizability that extends to sets of data obtained from any probing scheme:

\begin{definition}[Realizability of a Shadow] \label{def_shadow_realizability}
	Let $\GpDAG$ be a pDAG, and let $\param$ be a choice of parameters for $\GpDAG$. Consider a probing scheme whose shadowing function is  $\cal F_S$, that is, a probing scheme that reveals the function $\cal F_S$ of the full conditional distribution ${P(X_{\vis(\GpDAG)^\flat}|X_{\vis(\GpDAG)^\sharp})}$. The shadow $\mathcal{S}^{(\GpDAG,\param)}$ obtained by applying this probing scheme on the causal hypothesis $(\GpDAG,\param)$ is given by:
	\begin{equation}
	\mathcal{S}^{(\GpDAG,\param)}=\mathcal{F_S}\left({P^{(\spl(\GpDAG),\param)}(X_{\vis(\GpDAG)^\flat}|X_{\vis(\GpDAG)^\sharp})}\right). 
	\end{equation}

	Accordingly, a set $\cal S$ of data obtained from the probing scheme in question is said to be \emph{realizable by} $\GpDAG$ if there exists some choice of parameters $\param$ such that $\cal S=S^{(G,\param)}$.
\end{definition}

We have previously defined $P^{(\GpDAG,\param)}(X_{\vis(\GpDAG)})$ to be the observational distribution realized by a pDAG $\GpDAG$ under the choice of parameters $\param$. In the definition above, we extend this superscript notation for shadows obtained from a particular probing scheme\footnote{Note that we had already used this notation for the do-conditional of Eq.~\eqref{eq_simple_splitnode_interventional}.}.

\section{Equivalence and Dominance  of causal structures  under the Observe\&Do Probing scheme} 
\label{sec_ETT}

In this section, we will define and completely characterize the equivalence and dominance of causal structures when there is experimental access to an informationally complete probing scheme.

An example of an informationally complete probing scheme consists of what we call the \emph{Observe\&Do probing scheme}.
An Observe\&Do probe of a visible variable
 $X_a$ is the process of first recording the natural value 
 of the variable, say $x_a$, and then forcing the variable to take a potentially different value $x'_a$.
 One can imagine for example a drug trial where one asks the subjects if they would take the drug or not, before giving them a pill that either contains the drug or a placebo,  set independently of a subject's preference. Such an experiment would allow us to compare the effect of the drug on the people that would have taken it versus  its effect on  the people that would not  have. The Observe\&Do \emph{probing scheme}, which we will sometimes abbreviate to O\&D probing scheme, is the name given to the probing scheme where \emph{all} of the visible variables are Observe\&Do probed.
 
In the O\&D probing scheme, implementation of the probings of the different variables follows the same temporal ordering as the variables themselves. That is, if $X_a$ is temporally prior to $X_b$, then the Observe\&Do probe of $X_a$ is temporally prior to the  Observe\&Do probe of $X_b$.

Suppose we have a pDAG $\GpDAG$ on which we perform the O\&D probing scheme. The natural values of all the visible variables are stored in the set of variables $X_{\vis(\GpDAG)^\flat}$, and the values that these variables are subsequently forced to take are stored in the set of variables $X_{\vis(\GpDAG)^\sharp}$. The data obtained from this procedure takes the form $	P(X_{\vis(\GpDAG)^\flat}|X_{\vis(\GpDAG)^\sharp})$,
which is exactly the conditional distribution of expression~\eqref{eq_flat_given_sharp}. Because the O\&D probing scheme provides complete information about the conditional $P(X_{\vis(\GpDAG)^\flat}|X_{\vis(\GpDAG)^\sharp})$, it provides the maximum information about the causal hypothesis that can be obtained by probing the visible nodes (assuming no edge interventions)). Therefore, it is an informationally complete probing scheme. 
While this section focuses on  the Observe\&Do probing scheme, its results are valid for any informationally complete probing scheme.

The set of conditional probability distributions $P(X_{\vis(\GpDAG)^\flat}|X_{\vis(\GpDAG)^\sharp})$ for visible variables of cardinalities $\vec c_{\vis(\GpDAG)}$ that are realizable by the Observe\&Do probing scheme on the pDAG $\GpDAG$, will be denoted $\mathcal{M}_{\text{O$\&$D}}(\GpDAG,\vec c_{\vis(\GpDAG)})$.  Note that it corresponds exactly to the set of conditional distributions that are realizable by $\spl(\GpDAG)$ for the same cardinalities, i.e., $\mathcal{M}_{\text{O$\&$D}}(\GpDAG,\vec c_{\vis(\GpDAG)})= \mathcal{M}_{\text{obs}}(\spl(\GpDAG),\vec c_{\vis(\GpDAG)})$. 
 
We can now define notions of dominance and equivalence relative to the Observe\&Do probing scheme.
\begin{definition}[O$\&$D dominance and O$\&$D equivalence of pDAGs]
	\label{def_obs_equivalence_ETT}
	Let $\GpDAG$ and $\GpDAG'$ be two pDAGs such that $\vis(\GpDAG)=\vis(\GpDAG')$. We say that $\GpDAG$ \emph{O$\&$D-dominates} $\GpDAG'$ when the set of conditional probability distributions realizable by an O$\&$D probing scheme on  $\GpDAG$ includes the set of conditional probability distributions realizable by an O$\&$D probing scheme on  $\GpDAG'$, regardless of the assignment of cardinalities of the visible variables, i.e., when
	\begin{gather}
		\forall \vec c_{\vis(\GpDAG)} \in \mathbb{N}^{|\vis(\GpDAG)|}: \nonumber \\ 
		\mathcal{M}_{\text{O$\&$D}}(\GpDAG',\vec c_{\vis(\GpDAG')}) \subseteq   \mathcal{M}_{\text{O$\&$D}}(\GpDAG,\vec c_{\vis(\GpDAG)}).
	\end{gather}
	
	We say that $\GpDAG$ is \emph{O$\&$D-equivalent} to $\GpDAG'$ when their sets of O$\&$D-realizable distributions are the same:
	\begin{gather}
		\forall \vec c_{\vis(\GpDAG)} \in \mathbb{N}^{|\vis(\GpDAG)|}: \nonumber \\ 
		\mathcal{M}_{\text{O$\&$D}}(\GpDAG',\vec c_{\vis(\GpDAG')}) =   \mathcal{M}_{\text{O$\&$D}}(\GpDAG,\vec c_{\vis(\GpDAG)}).
	\end{gather}	
\end{definition}

\begin{sloppypar} A conditional probability distribution  $P(X_{\vis(\GpDAG)^\flat}|X_{\vis(\GpDAG)^\sharp})$ is said to be {\em O$\&$D-realizable} by a pDAG $\GpDAG$ if it is {\em observationally realizable} by the full-SWIG $\spl(\GpDAG)$.  This fact implies the following lemma: \end{sloppypar}
\begin{lemma}
	Let $\GpDAG$ and $\GpDAG'$ be two pDAGs with the same set of visible nodes. $\GpDAG$ O$\&$D-dominates $\GpDAG'$ if and only if the full-SWIG $\spl(\GpDAG)$ observationally dominates the full-SWIG $\spl(\GpDAG')$.
	\label{lemma_ETT_fullSWIG}
\end{lemma}

This lemma shows that it suffices to characterize the dominance structure of full-SWIGs under passive observations in order to understand the dominance structure of pDAGs under an O$\&$D probing scheme.  In this way, we reduce questions about the unconventional notion of O$\&$D dominance to questions about the better-studied notion of observational dominance.

With this lemma in hand, we can now present one of the main results of this article.  It is that O$\&$D dominance of pDAGs is completely characterized by structural dominance of the corresponding mDAGs:
\begin{theorem}
Let $\GpDAG$ and $\GpDAG'$ be two pDAGs where $\vis(\GpDAG)$=$\vis(\GpDAG')$, and let  $\GmDAG$ and  $\GmDAG'$ be the corresponding mDAGs, i.e., $\GmDAG\equiv\mdag(\GpDAG)$ and $\GmDAG'\equiv\mdag(\GpDAG')$.
The pDAG $\GpDAG$ O$\&$D-dominates  the pDAG $\GpDAG'$ if and only if the mDAG $\GmDAG$ structurally dominates the mDAG $\GmDAG'$.  \label{th_ETT}
\end{theorem}
\begin{proof}
	Together, Lemmas~\ref{lemma_full_SWIGdominance} and~\ref{lemma_ETT_fullSWIG} imply that $\GpDAG$ O$\&$D-dominates $\GpDAG'$ if and only if the 3-mDAG $\mdag(\spl(\GpDAG))$ structurally dominates the 3-mDAG $\mdag(\spl(\GpDAG'))$. With Lemma~\ref{lemma_commutation_maps}, we can commute $\mdag$ with $\spl$. So, $\GpDAG$ O$\&$D-dominates $\GpDAG'$ if and only if $\spl(\mdag(\GpDAG))=\spl(\GmDAG)$ structurally dominates $\spl(\mdag(\GpDAG'))=\spl(\GmDAG')$. 
	
	It remains therefore to prove that $\GmDAG$ structurally dominates $\GmDAG'$ if and only if  $\spl(\GmDAG)$ structurally dominates $\spl(\GmDAG')$. 	This follows from the fact that, for any two mDAGs $\GmDAG$ and $\GmDAG'$, there is an edge $a\rightarrow b$ present in $\GmDAG$ but not present in $\GmDAG'$ if and only if the corresponding edge $a^\sharp\rightarrow b^\flat$ is present in $\spl(\GmDAG)$ but not present in $\spl(\GmDAG')$, and there is a facet $S$ present in  $\GmDAG$ but not present in $\GmDAG'$ if and only if the corresponding facet $S^\flat$ is present in $\spl(\GmDAG)$ but not  present in $\spl(\GmDAG')$.
\end{proof}

By using structural dominance of mDAGs, we can construct a pre-order of pDAGs where the equivalence classes are given by the sets of all pDAGs that are associated with the same mDAG. We can convey information about this pre-order by instead presenting the corresponding partial order between equivalence classes, represented by the mDAGs. This partial order over mDAGs will be called the \emph{structural partial order}.  As an example, Fig.~\ref{2node_structural} shows the structural partial order of mDAGs with two nodes. 
 By Theorem~\ref{th_ETT}, the structural partial order  captures the partial order over causal structures induced by the O\&D dominance relation.

\begin{figure}[h!]
	\centering
	\includegraphics[width=0.45\textwidth]{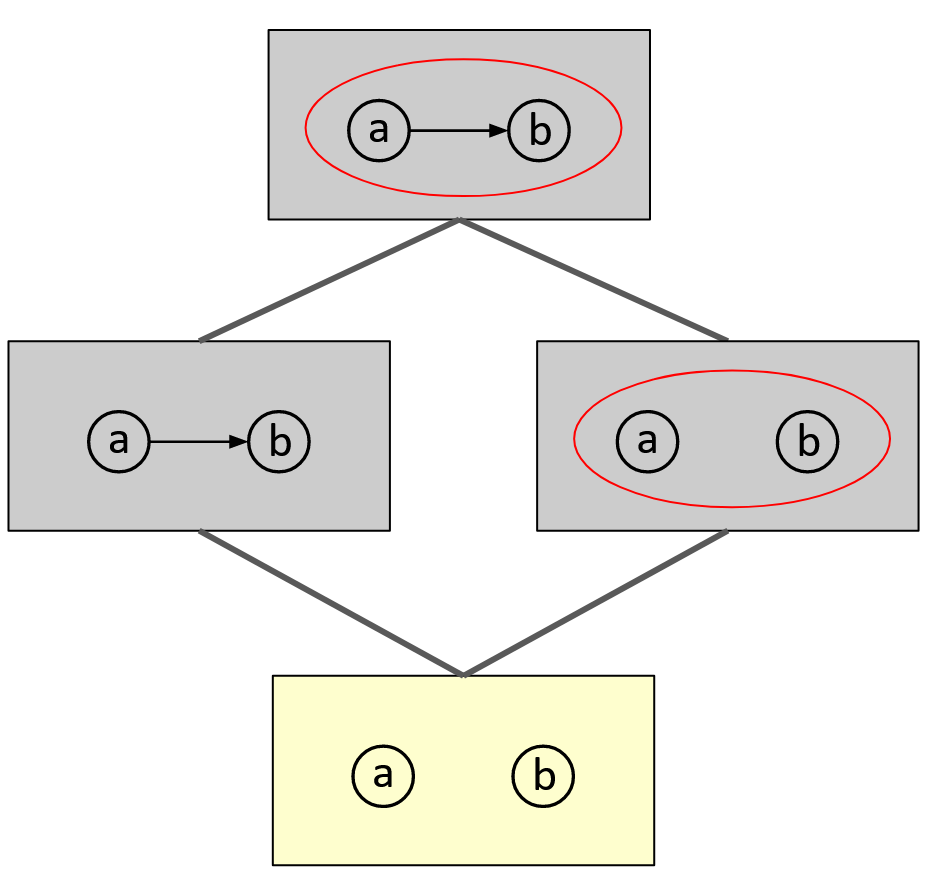}
	\caption{All mDAGs with two nodes $a$ and $b$, with $a$ temporally prior to $b$, organized into the structural partial order, which also expresses the partial order of equivalence classes under the O\&D relation. 
	The background colour of each mDAG indicates its observational equivalence class.
	}
	\label{2node_structural}
\end{figure}

Fig.~\ref{2node_observational} shows the partial order that holds among equivalence classes of mDAGs of two nodes under the observational dominance relation. As can be seen in this figure, there are only two observational equivalence classes of two-node mDAGs: the one represented in gray can realize all probability distributions over two variables, and the one represented in white can only realize distributions where the two variables are independent. By comparing Fig.~\ref{2node_structural} with Fig.~\ref{2node_observational}, we see that the partial order under O$\&$D dominance is distinct from the partial order under observational dominance. In particular, this implies that there are cases where one pDAG observationally dominates another without structurally dominating it, that is, the converse of Lemma~\ref{lemma_structural_obs_dominance} does not hold. For instance, for the pDAG where $a$ and $b$ are cause-effect related and the pDAG where they are confounded, we have observational dominance relations holding in both directions, but no structural dominance in either direction.

\begin{figure}[h!]
	\centering
	\includegraphics[width=0.45\textwidth]{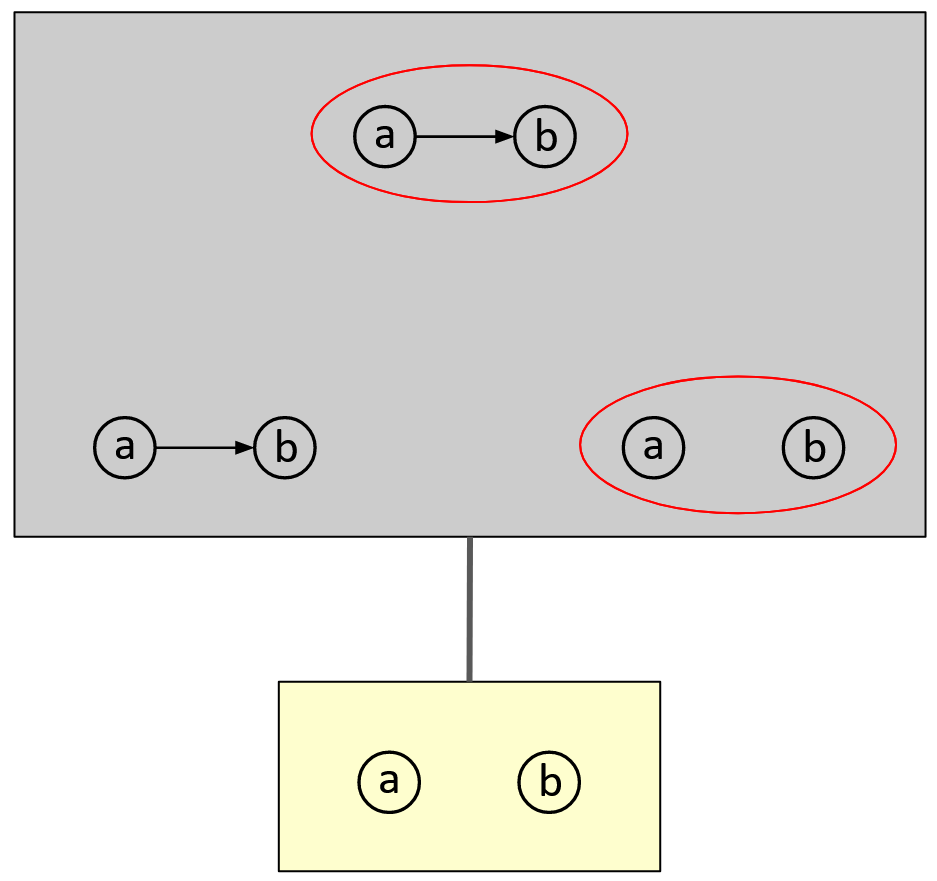}
	\caption{All mDAGs with two nodes $a$ and $b$, with $a$ temporally prior to $b$,  organized into the partial order of equivalence classes under the observational dominance relation.}
	\label{2node_observational}
\end{figure}

\section{Equivalence and Dominance under all patterns of do-interventions}
\label{sec_do}

Theorem~\ref{th_ETT} provides a complete characterization of what can be inferred about the underlying causal structure if there is access to an informationally complete probing scheme, such as the Observe\&Do probing scheme. We have shown that the full conditional distribution ${P(X_{\vis(\GpDAG)^\flat}|X_{\vis(\GpDAG)^\sharp})}$ allows one to identify the pDAG up to its mDAG equivalence class. The question arises, then, whether strictly less information, in the form of certain shadows of the  full conditional distribution,  might still be sufficient for identifying the pDAG up to its mDAG equivalence class, that is, whether knowing the full conditional probability distribution ${P(X_{\vis(\GpDAG)^\flat}|X_{\vis(\GpDAG)^\sharp})}$ is in fact \emph{not necessary} for achieving such an identification.  We will see shortly that this is in fact the case. In this section, we will investigate a probing scheme that is \emph{not} informationally complete, but nevertheless dominance relative to it is \emph{also} completely characterized by structural dominance of mDAGs.

Consider the case where, for each visible variable, the experimentalist implements a passive observation or a do-intervention, but never both, that is, the case where there is a disjunction of the two possibilities rather than the conjunction: Observe {\em or} Do, rather than Observe {\em and} Do. 

Such a probing scheme is of practical signficance because in many real-world experiments that involve an intervention on one or more variables, the intervention is of the purely Do variety rather than of the Observe$\&$Do variety.  For instance, in most blind drug trials, where a given subject receives, at random, the drug or a placebo as the treatment, the experimentalist \emph{does not} trouble themselves to ask the subject what their preference regarding taking the drug would have been were they to have been given the choice.  In other words, the value that the treatment variable would take if it were determined by the natural causal mechanisms (rather than fixed by intervention) is \emph{not} recorded. A do-intervention is in fact the most common type of intervention considered in the causal inference literature. 

We will refer to a specification of the set of visible nodes of a causal structure $\GpDAG$ that are subject to a do-intervention as a {\em do-pattern}, with the understanding that the complementary set of visible nodes is passively observed. The probing scheme wherein the subset of visible nodes that are subject to a do-intervention is $A \subseteq \vis(\GpDAG)$
  will be referred to as the {\em $A$-pattern Observe-or-Do probing scheme}.

The data obtained from implementing the $A$-pattern Observe-or-Do probing scheme for some $A \subseteq \vis(\GpDAG)$ takes the following form:
\begin{equation}
	\{P(X_{(\vis(\GpDAG)\setminus A)^\flat}|X_{A^\sharp}=x'_A) \hspace{0.25em} | \hspace{0.25em} x'_A\in \mathcal{X}_A\},
	\label{eq_do_data}
\end{equation}

where the natural value of a visible variable $X_v\in X_{\vis(\GpDAG)}$ is stored in the variable $X_{v^\flat}$, and the values that the variables $X_A$ are forced to take are stored in the set of variables $X_{A^\sharp}$. 

 Note that the Observe\&Do probing scheme (studied in Section~\ref{sec_ETT}) provides data where, for a node $a\in \vis(\GpDAG)$, the variable $X_{a^\flat}$ appears to the left side of the conditional \emph{at the same time} that the variable $X_{a^\sharp}$ appears to the right side of the conditional. By contrast, this \emph{does not happen} in the A-pattern Observe-or-Do probing scheme for some  $A \subseteq \vis(\GpDAG)$. That is, for a node $a\in \vis(\GpDAG)$, in expression~\eqref{eq_do_data} there is \emph{either} $X_{a^\flat}$ to the left side of the conditional \emph{or}  $X_{a^\sharp}$ to the right side of the conditional.

As discussed in Section~\ref{sec_probingschemes}, the expression~\eqref{eq_do_data} captures the same information as 
the standard notion of a ``do-conditional'', which is traditionally denoted by  $P(X_{\vis(\GpDAG)\setminus A}|\text{do}(X_A))$. 
We will call it the ``$A$-pattern do-conditional''.

The shadowing function associated with the A-pattern Observe-or-Do probing scheme consists of taking the components of ${P(X_{\vis(\GpDAG)^\flat}|X_{\vis(\GpDAG)^\sharp})}$ where $X_{(\vis(\GpDAG)\setminus A)^\flat}=X_{(\vis(\GpDAG)\setminus A)^\sharp}$ and marginalizing out $X_{A^\flat}$. That is, defining  $\bar{A}\equiv\vis(\GpDAG)\setminus A$, the output of the shadowing function is the set
\begin{align}
  \Bigg\{&\sum_{x_{A^\flat} } P( X_{A^\flat}= x_{A^\flat}, X_{\bar{A}^\flat} |X_{A^\sharp}=x'_{A}, X_{\bar{A}^\sharp})\hspace{0.25em}  \Big | \hspace{0.25em} x'_{A}\in \mathcal{X}_A, \nonumber \\
  &P( X_{A^\flat}, X_{\bar{A}^\flat}=x_{\bar{A}^\flat} |X_{A^\sharp}, X_{\bar{A}^\sharp}=x_{\bar{A}^\sharp})=0 \text{ if } x_{\bar{A}^\flat}\neq x_{\bar{A}^\sharp}
    \Big\}
	\label{eq_shadow_Apattern}
\end{align}

This shadowing function gives the definition of realizability of an $A$-pattern do-conditional by a pDAG through Def.~\ref{def_shadow_realizability}.

We will be primarily interested in a probing scheme wherein one collects statistical data from {\em all} possible patterns of do-interventions, that is, wherein one partitions the ensemble of samples and implements the $A$-pattern Observe-or-Do probing scheme for {\em every} subset $A \subseteq \vis(\GpDAG)$.  We refer to this as the {\em all-patterns Observe-or-Do} probing scheme. We will also make some brief remarks about {\em single-pattern} Observe-or-Do probing schemes, where one only has access to the data for a single subset  $A \subseteq \vis(\GpDAG)$.

 In this section, we will show that even though the all-patterns Observe-or-Do probing scheme is generally not informationally complete, dominance relative to it is completely characterized by structural dominance of mDAGs. A claim found in Ref.~\cite[Section 7]{evans_graphs_2016} seems, at first glance, to contradict this statement.
 So we endeavour here to explain how the result described there fits into the discussion in this article.

Ref.~\cite{evans_graphs_2016} discussed the problem of distinguishability of different pDAGs under (what we are here calling) single-pattern Observe-or-Do probing schemes, that is, access to the statistical data obtained from do-interventions on a single subset $A$ of visible nodes. 
As noted in Ref.~\cite{evans_graphs_2016}, a single-pattern Observe-or-Do probing scheme generally {\em cannot} distinguish pDAGs associated to different mDAGs, no matter which pattern one chooses. Although Ref.~\cite{evans_graphs_2016} considered what could be inferred from any pattern of do-interventions, it crucially did not allow that the data for one pattern could be compared with the data from another pattern, unlike the all-patterns probing scheme.  This, ultimately, accounts for the differences between the results described earlier in this section and the claim of Ref.~\cite{evans_graphs_2016}.

To make the distinction between the single-pattern and all-patterns Observe-or-Do probing schemes more explicit, here we discuss the example of Fig.16 of Ref.~\cite{evans_graphs_2016}, which we reproduce in Fig.~\ref{fig_evans_16}. 

\begin{figure}[h!]
	\centering
	\includegraphics[width=0.45\textwidth]{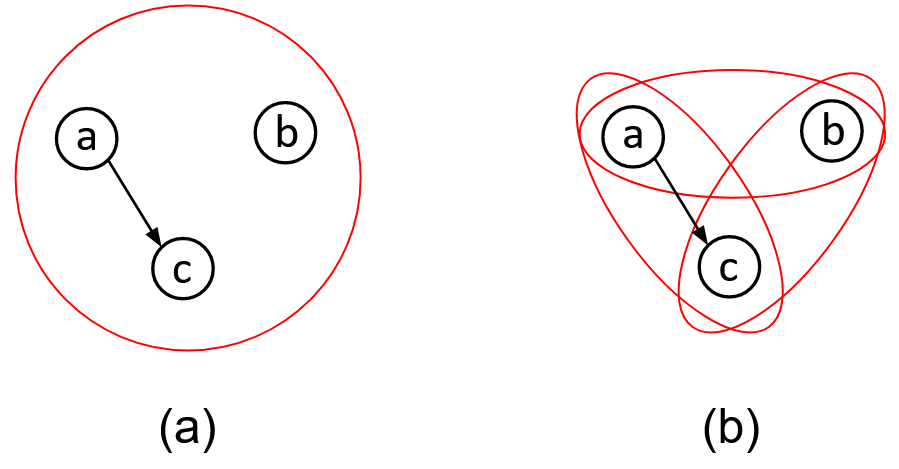}
	\caption{mDAGs that appear in Fig. 16 of Ref~\cite{evans_graphs_2016}.
	}    \label{fig_evans_16}
\end{figure}

Individually, all of the distributions obtained from passive observations or from different do-patterns in the mDAG of Fig.~\ref{fig_evans_16}(a) can also be realized by the mDAG of Fig.~\ref{fig_evans_16}(b), which led Ref.~\cite{evans_graphs_2016} to classify these mDAGs as indistinguishable under each of these probing schemes. However, further inspection reveals that there are sets of data obtained from passive observations and do-interventions that can be \emph{jointly} realized by Fig.~\ref{fig_evans_16}(a) but not Fig.~\ref{fig_evans_16}(b).
The following is an example: under passive observation, the three variables are perfectly correlated; under a do-intervention on variable $X_a$,
  the variables associated $X_b$ and $X_c$
   remain perfectly correlated. 
In Fig.~\ref{fig_evans_16}(a), one can jointly realize this behaviour if all variables copy the latent variable that acts as a three-way common cause (note that the arrow from $a$ to $b$ plays no role). 
 In Fig.~\ref{fig_evans_16}(b), the {\em only} way one can realize the behaviour for the passive observation case is by having the variables $X_a$ and $X_c$ copy the latent variable that acts upon both of them, and having variable $X_b$ copy variable $X_a$.  But this choice of parameters  does not also realize  the behaviour for the do-intervention case because the intervention on $X_a$ eliminates the correlation between $X_a$ and $X_c$ and thus also the correlation between $X_b$ and $X_c$.

If one imagines having the experimental capability to implement {\em any} pattern of do-interventions, as Ref.~\cite{evans_graphs_2016} clearly did, there is no reason not to consider what can be inferred from comparing the statistical data obtained from these different patterns.  That is, if one has access to the single-pattern probing scheme for every pattern, then there is no reason not to implement each of these on a subset of samples, and hence to implement the all-patterns probing scheme. This seems to us to be an oversight of the analysis presented in Ref.~\cite{evans_graphs_2016}.

In conclusion, to implement the all-patterns Observe-or-Do probing scheme it is \emph{not} enough to show that every distribution from each do-pattern that is realizable by $\GpDAG$ is also realizable by $\GpDAG'$. Rather, it is necessary to show that all the \emph{combinations} of distributions from different do-patterns that are realizable by $\GpDAG$ under one given choice of parameters $\param$ are also realizable by $\GpDAG'$ under one single choice $\param'$ of parameters.

To help us consider the relation between data obtained from different do-patterns, we now introduce the terminology of \emph{joint realizability} to indicate that a set of data obtained from different do-patterns is realizable by a pDAG under the same choice of parameters.    

\begin{definition}[Joint Realizability of data from Different Do-Patterns]
	Let $\GpDAG$ be a pDAG, let $A_i\subseteq \vis(\GpDAG)$ be a subset of visible nodes of $\GpDAG$ and let $\mathcal{S}=\{A_i: i\in \{1,\dots,N\}\}$ be a set of such subsets.  Let $\mathcal{Q}_i( x_{A_i}' )$ be the data obtained by performing a do-intervention on  the variables associated to nodes in $A_i$, namely fixing $X_{A_i}=x_{A_i}'$, and passive observation on the rest,
	\begin{equation}
		\mathcal{Q}_i(x_{A_i}') =Q(X_{\vis(\GpDAG)\setminus A_i}| X_{A_i}=x_{A_i}')
	\end{equation}
Let $\mathcal{Q}$ be the union, over all values of $i$ and of $x_{A_i}'p$, of the $\mathcal{Q}_i(x_{A_i}')$,
	\begin{equation}
		\mathcal{Q}= \{ \mathcal{Q}_i(x_{A_i}')  | x_{A_i}' \in \mathcal{X}_{A_i}, i\in \{1,\dots,N\} \}
	\end{equation}
	We say that the set of data $\cal Q$ is \emph{jointly realizable} by the pDAG $\GpDAG$ if there exists a choice of parameters $\param$ in $\GpDAG$ such that 
	\begin{align}
	& \forall i\in \{1,...,n\}, \forall x_{A_i}' \in \mathcal{X}_{A_i}:\nonumber\\
	&P^{(\GpDAG,\param)}(X_{\vis(\GpDAG)\setminus {A_i}}| \text{do}(X_{A_i}=x_{A_i}')) = \nonumber\\
	&=Q(X_{\vis(\GpDAG)\setminus {A_i}}| X_{A_i}= x_{A_i}')  \text{ }.
	\end{align}
\end{definition}

Note that $\cal Q$ can include data obtained from passive observations: this is the case if, for some value of $i$, $A_i=\emptyset$. The set of all such sets of data $\cal Q$ that are jointly realizable by a pDAG $\GpDAG$ and exhibit the cardinalities $\vec c_{\vis(\GpDAG)}$ for the visible variables will be denoted by $\mathcal{M}_{\text{All O-or-D}}(\GpDAG,\vec c_{\vis(\GpDAG)})$, where All O-or-D stands for ``all-patterns Observe-or-Do''.  

Now we can finally give our definitions of dominance and equivalence relative to the all-patterns Observe-or-Do probing scheme:

\begin{definition}[All-patterns O-or-D dominance and equivalence of pDAGs]
	\label{def_obs_and_do_equivalence}
	Let $\GpDAG$ and $\GpDAG'$ be two pDAGs such that $\vis(\GpDAG)=\vis(\GpDAG')$. We say that $\GpDAG$ \emph{all-patterns O-or-D dominates} $\GpDAG'$ when, regardless of the assignment of cardinalities of the visible variables, all of the sets of data obtained from the all-patterns probing scheme that are jointly realizable by $\GpDAG'$ are also jointly realizable by $\GpDAG$. That is, when
	\begin{gather}
		\forall \vec c_{\vis(\GpDAG)} \in \mathbb{N}^{|\vis(\GpDAG)|}: \nonumber \\ \mathcal{M}_{\text{All O-or-D}}(\GpDAG',\vec c_{\vis(\GpDAG')}) \subseteq   \mathcal{M}_{\text{All O-or-D}}(\GpDAG,\vec c_{\vis(\GpDAG)}) 
	\end{gather}
	
	We say that $\GpDAG$ is \emph{all-patterns O-or-D equivalent} to $\GpDAG'$ when dominance holds in both directions.
\end{definition}

In other words, $\GpDAG$ all-patterns O-or-D dominates $\GpDAG'$ 
if for all choices $\param'$ of parameters of $\GpDAG'$ there exists at least one choice $\param $ of parameters of $\GpDAG$ such that, for \emph{any} choice of set of visible variables $X_A\in X_{\vis(\GpDAG)}$ to intervene upon and of values to set these variables to, the realized do-conditional (or joint distribution, if $A=\emptyset$) is the same:
\begin{gather}
	\forall A\in \vis(\GpDAG),  \forall x_{A}' \in \mathcal{X}_{A} : \nonumber \\ P^{(\GpDAG,\param)}(X_{\vis(\GpDAG)}\setminus X_{A}| \text{do}(X_{A}=x'_A))= \nonumber \\ =P^{(\GpDAG',\param')}(X_{\vis(\GpDAG')}\setminus X_{A}| \text{do}(X_{A}=x'_A)).
	\label{eq_def_obs_and_do_dominance}
\end{gather}

Now, we prove our second main result: even though the all-patterns Observe-and-Do probing scheme is not informationally complete, all-patterns O-or-D dominance is \emph{also} characterized by structural dominance of mDAGs.
\begin{theorem}
	Let $\GpDAG$ and $\GpDAG'$ be two pDAGs and let  $\GmDAG$ and  $\GmDAG'$ be the corresponding mDAGs, i.e., $\GmDAG\equiv\mdag(\GpDAG)$ and $\GmDAG'\equiv\mdag(\GpDAG')$.  $\GpDAG$ all-patterns O-or-D dominates  $\GpDAG'$ if and only if $\GmDAG$
	 structurally dominates $\GmDAG'$.
	\label{th_do_and_ob}
\end{theorem}
\begin{proof}
	The ``if'' side follows from Theorem~\ref{th_ETT}. This is so because, as we saw with the example of Fig.~\ref{fig_fullSWIG_parameters}, passive observations and do-interventions reveal different shadows of the conditional probability distribution ${P(X_{\vis(\GpDAG)^\flat}|X_{\vis(\GpDAG)^\sharp})}$. If $\GmDAG$ structurally dominates $\GmDAG'$, then Theorem~\ref{th_ETT} says that even with access to the full \emph{distribution} ${P(X_{\vis(\GpDAG)^\flat}|X_{\vis(\GpDAG)^\sharp})}$ it is impossible to find data that is realizable by $\GpDAG'$ but not by $\GpDAG$; therefore, the same is certainly true when one only has access to \emph{shadows} of this distribution.
	
	To prove the ``only if'' side, we will proceed by proving its contrapositive. Assume that $\GmDAG$ \emph{does not} structurally dominate $\GmDAG'$.  The lack of structural dominance implies one of two possibilities:
	\begin{enumerate}
		\item There is at least one directed edge between visible nodes $a$ and $b$,  $a\rightarrow b$, that is present in ${\GmDAG}'$ but not in ${\GmDAG}$, or
		\item  ${\GmDAG}$ has all the directed edges that are present in ${\GmDAG}'$, but there is at least one set $S$ of visible nodes which is a face of ${\GmDAG}'$ but not of ${\GmDAG}$.
	\end{enumerate}
	
	These two conditions for mDAGs respectively imply the following conditions for the corresponding pDAGs:
\begin{enumerate}
	\item There is at least one pair of visible nodes $a$ and $b$ such that ${\GpDAG}'$ presents a chain $a\rightarrow m_1 \rightarrow ... \rightarrow m_n\rightarrow b$ where all of the mediary nodes $m_1,...,m_n$ are latent nodes, and there is \emph{no} such chain in ${\GpDAG}$, or
	\item Item 1 does not hold, but there is at least one set $S$ of visible nodes that have a common latent ancestor $\lambda$ in ${\GpDAG}'$, but no such  latent ancestor exists in ${\GpDAG}$.
\end{enumerate}

	 For case 1, note that all the chains that connect $a$ to $b$ in $\GpDAG$ must have at least one  mediary node $m$ that is \emph{visible}. Let the set of all such visible mediary nodes be called $M$. Now, if $A$ is the set of nodes subject to a do-intervention in a given do-pattern, we will compare the case where $A=\emptyset$ with the case where $A=\{a\}$. That is, we consider two do-patterns: (i) passive observations on all the visible variables; and (ii) do-intervention on $a$ and passive observations on all other visible variables. Suppose that the data obtained respectively from these two do-patterns have the marginals:
	\begin{subequations}
		\begin{align}
			\label{eq_ex_observational_proof1}
			&\text{(i)}	    &&P(X_b|X_M) \\
			&\text{(ii)}	&&P(X_b|X_M,\text{do}(X_a=0))\neq P(X_b|X_M).
			\label{eq_ex_interventional_proof1}
		\end{align}
	\end{subequations}	
	 This set of data is jointly realizable by $\GpDAG'$, since in this pDAG the variable $X_b$ can have a direct functional dependence on $X_a$, in a way that the marginal over $X_b$ changes when $X_a$ is forced to take the value $0$. However, this set of data is \emph{not} jointly realizable by $\GpDAG$: after we condition on all of the variables $X_M$ associated with mediary nodes $M$, all of the directed chains between $a$ and $b$ become blocked paths. Furthermore, none of the nodes in $M$ can be common children of $a$ and $b$. Therefore, all of the unblocked paths between $a$ and $b$ are paths that have incoming arrows to $a$; these causal connections are suspended when we perform do-interventions on $a$. Therefore, $\GpDAG$ \emph{does not} all-patterns O-or-D dominates  $\GpDAG'$\footnote{Alternatively, case 1 can be solved by the following observation, made in Ref.~\cite[Proposition 7.5]{evans_graphs_2016}: do-conditionals where
	 	\begin{gather*}
	 		\exists x_a, x'_a	\text{ such that } \\ P(X_b|\text{do}(X_a=x_a))\neq P(X_b|\text{do}(X_a=x'_a)),
	 	\end{gather*}
	 	that is, where the observed value of $X_b$ depends on the value to which $X_a$ was set, are jointly realizable by $\GpDAG'$ but not by $\GpDAG$.}.
	 
	 For case 2, consider the do-patterns $A=\emptyset$ and, for each element $T\in 2^S$ of the powerset of $S$,  $A=\vis(\GpDAG)\setminus T$. That is, consider (i) passive observations on all the visible variables; and (ii) for each element $T\in 2^S$ of the powerset of $S$, a do-intervention that fixes a variable to $0$ on all of the visible variables except for the ones in $X_T$, and a passive observation of $X_T$. Suppose that the data obtained from these do-patterns is:
	 \begin{subequations}
	 	\begin{align}
	 		\label{eq_ex_observational_proof2}
	 		&\text{(i)}&&	    P(X_S)=p[0,...,0]+(1-p)[1,...,1] \\
	 		&\text{(ii)} &&	\forall T\in 2^S, \nonumber \\ & &&Q(X_T=0|\text{do}(X_{\vis(\GpDAG)\setminus T}= 0_{\vis(\GpDAG)\setminus T}))=p 
	 		\label{eq_ex_interventional_proof2}
	 	\end{align}
	 \end{subequations}
 	Where $0_{\vis(\GpDAG)\setminus T}$ denotes an assignment of value 0 to each variable $X_{a}$ where $a\in \vis(\GpDAG)\setminus T$, and the expression in the right hand side of Eq.~\eqref{eq_ex_observational_proof2} means that with $p$ probability all of the variables of $X_S$ take value $0$, and with $1-p$ probability all of them take the value $1$. That is, all of the variables in $X_S$ are perfectly correlated in the data obtained from passive observations, and the marginal over each subset of the variables of $X_S$ does not change by performing a do-intervention that forces all of the other visible variables to take the value $0$. 
	 
	 As it turns out, the set of data of Eqs.~\eqref{eq_ex_observational_proof2} and \eqref{eq_ex_interventional_proof2} is jointly realizable by $\GpDAG'$, but not by $\GpDAG$. It is easy to see why it is jointly realizable by $\GpDAG'$: there, all of the variables in $X_S$ can copy the latent common cause $X_\lambda$ that they share. The proof that this set of data is \emph{not} jointly realizable by $\GpDAG$ makes use of the fact that $\spl(\GpDAG)$ cannot realize the conditional distribution of Eq.~\eqref{eq_perfect_corrSWIG}, together with a technique due to Pedro Lauand that is described in Appendix~\ref{appendix_theorem1_case2}.
	 
	 Therefore, in both cases we have explicitly presented a set of data obtained from different do-patterns that is jointly realizable by $\GpDAG'$ but not by $\GpDAG$. This shows that, if  $\GmDAG=\mdag(\GpDAG)$ does not structurally dominate $\GmDAG'=\mdag(\GpDAG')$, then $\GpDAG$ does not all-patterns O-or-D dominate $\GpDAG'$.	 
\end{proof}

\subsection{One-Value do-interventions}

In the proof of Theorem~\ref{th_do_and_ob}, Eqs.~\eqref{eq_ex_interventional_proof1} and~\eqref{eq_ex_interventional_proof2} only involve interventions that set a variable to \emph{one} of its possible values, here called $0$. This shows that, to distinguish two pDAGs that correspond to the same mDAG, it is sufficient to perform do-interventions that set the variables to only one value.  We will call those \emph{one-value do-interventions}, and the probing scheme that consists of implementing one-value do-interventions on all possible do-patterns will be called the \emph{all-patterns Observe-or-1Do probing scheme}. To see why the restriction to the all-patterns Observe-or-1Do probing scheme can be of interest, imagine an experiment where we are allowed to force the subjects to quit smoking, but we cannot ethically force them to start smoking. In this case, we can intervene on the experiment to force the variable ``smoker'' to take the value $0$, but not to take the value $1$. 

We define dominance and equivalence relative to this probing scheme, termed {\em all-patterns O-or-1D dominance and equivalence}, in an analogous way to Def.~\ref{def_obs_and_do_equivalence}. That is, we say that $\GpDAG$ \emph{all-patterns O-or-1D dominates} $\GpDAG'$ when, regardless of the assignment of cardinalities of the visible variables, all of the sets of data
that are realizable by $\GpDAG'$ in the all-patterns Observe-or-1Do probing scheme are also  realizable by $\GpDAG$. The observation of the previous paragraph implies that
all-patterns O-or-1D dominance is \emph{also} characterized by the structural dominance of mDAGs:
\begin{theorem}
	Let $\GpDAG$ and $\GpDAG'$ be two pDAGs and let  $\GmDAG$ and  $\GmDAG'$ be the corresponding mDAGs, i.e., $\GmDAG\equiv\mdag(\GpDAG)$ and $\GmDAG'\equiv\mdag(\GpDAG')$. $\GpDAG$  all-patterns O-or-1D dominates  $\GpDAG'$ if and only if 
	$\GmDAG$	structurally dominates $\GmDAG'$.
	\label{th_1do_and_ob}
\end{theorem}
\begin{proof}
	The ``if'' side follows from Theorem~\ref{th_do_and_ob}. For the ``only if'' side, because our proof of Theorem~\ref{th_do_and_ob}
	 only made use of do-interventions that force the variables to take the value $0$, it serves as a proof for this case as well. 
\end{proof}

\section{Structural Partial Order for mDAGs with three or four visible nodes}
\label{sec_examples}

Theorems~\ref{th_ETT}, \ref{th_do_and_ob} and~\ref{th_1do_and_ob} respectively say that the partial order of equivalence classes of causal structures under the relations of O$\&$D dominance, all-patterns O-or-D dominance and all-patterns O-or-1D dominance all correspond exactly to the partial order induced by structural dominance of mDAGs, i.e., the structural partial order. Finding this partial order for mDAGs with any number of nodes is straightforward. It  is nonetheless pedagogically useful to see some concrete examples. The first example was already presented in Fig.~\ref{2node_structural}: the structural partial order for two-node mDAGs. We will here present the structural partial orders for three-node and four-node mDAGs. 

The structural order for the full set of $n$-node mDAGs  (with a fixed temporal ordering of the nodes) can be understood in terms of the structural order holding among two subsets thereof: those that are confounder-free and those that are directed-edge-free.  

First, we note that one can think of the full set as the Cartesian product of these two subsets.  The set of directed-edge-free $n$-node mDAGs describes the set of all possible simplicial complexes that can hold among $n$ nodes, while the set of confounder-free $n$-node mDAGs describes the set of all possible directed-edge structures that can hold for these.  The set of all mDAGs is simply the set of all possible combinations of a simplicial complex with a directed-edge structure. For the simplest possible example, of two-node mDAGs, the structural partial order for the  subset of two-node mDAGs that are confounder-free is presented in Fig.~\ref{fig_2node_confounderfree}(a), and the structural partial order for the subset of two-node mDAGs that are   directed-edge-free is presented in Fig.~\ref{fig_2node_confounderfree}(b). The full structural partial order in this case was presented in Fig.~\ref{2node_structural}. 

Second, we note that a dominance relation that holds between two mDAGs that  are directly connected in the partial order is due either to dropping a face from the simplicial complex or from dropping an edge from the directed-edge structure. 
It is useful to organize the full set of $n$-node mDAGs (with a fixed temporal ordering of the nodes)  into ``islands'' where the  mDAGs within an island have the same simplicial complex and differ only in their directed-edge structure. For the simple case of two-node mDAGs, these islands are given by the large rectangular boxes of Fig.~\ref{fig_2node_compressed}. The partial order holding among the elements of each island is the one induced by dropping edges, which is simply the partial order of the confounder-free mDAGs. Meanwhile, if one takes the collection of mDAGs that have the same directed-edge structure but different simplicial complexes (containing one mDAG from each island), the partial order holding among these is the one induced by dropping faces, which is simply the partial order of the directed-edge-free mDAGs.  These two sets of partial order relations exhaust the set of relations holding among the mDAGs.  

When depicting the full structural partial order, we explicitly depict the partial order relations that hold between the mDAGs {\em within} each island, but we use a shorthand notation for the partial order relations that hold between mDAGs from different islands (those induced by dropping of faces of the simplicial complex).  
Rather than drawing the face-dropping partial order explicitly for each  collection of mDAGs that have the same directed-edge structure, one puts a box around each island and one draws the face-dropping partial order among the islands. This is explicit in Fig.~\ref{fig_2node_compressed}, which is our shorthand depiction of the full structural order of Fig.~\ref{2node_structural}. Within each island of Fig.~\ref{fig_2node_compressed}, the partial order is that of Fig.~\ref{fig_2node_confounderfree}(a).  The partial order between the islands is that of Fig.~\ref{fig_2node_confounderfree}(b), and represents the partial order between the collection of mDAGs that have the same directed-edge structure (one from each island). 

\begin{figure}[h!]
	\centering
	\includegraphics[width=0.45\textwidth]{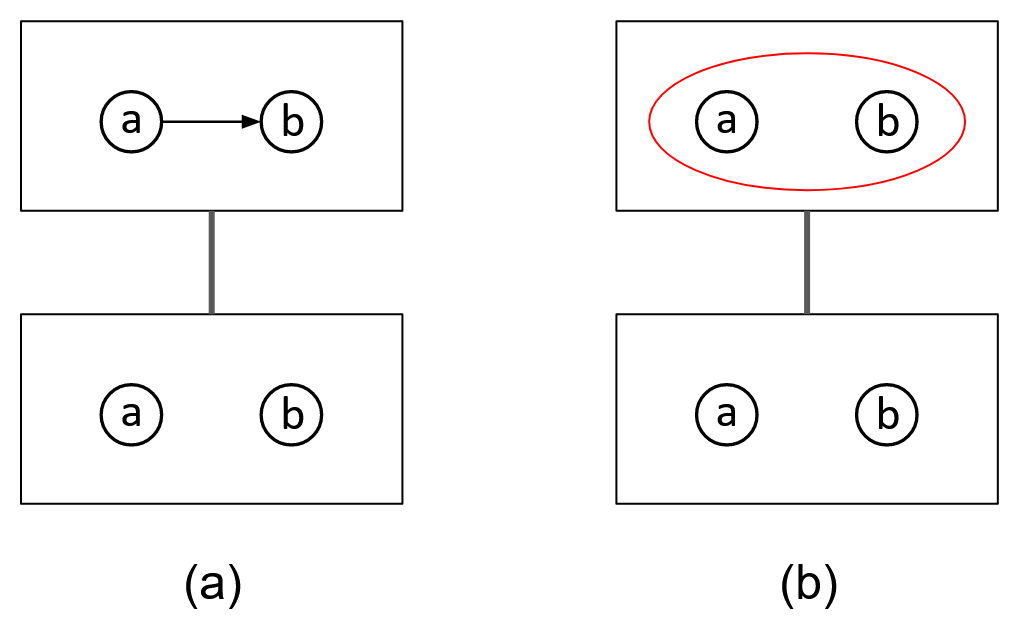}
	\caption{(a) Structural partial order that holds between all of the two-node confounder-free mDAGs that are consistent with the temporal ordering $(a,b)$. (b) Structural partial order that holds between all of the two-node directed-edge-free mDAGs. Note that a directed-edge-free mDAG is always consistent with any temporal ordering of the visible nodes.} 
	\label{fig_2node_confounderfree} 
\end{figure}

\begin{figure}[h!]
	\centering
	\includegraphics[width=0.25\textwidth]{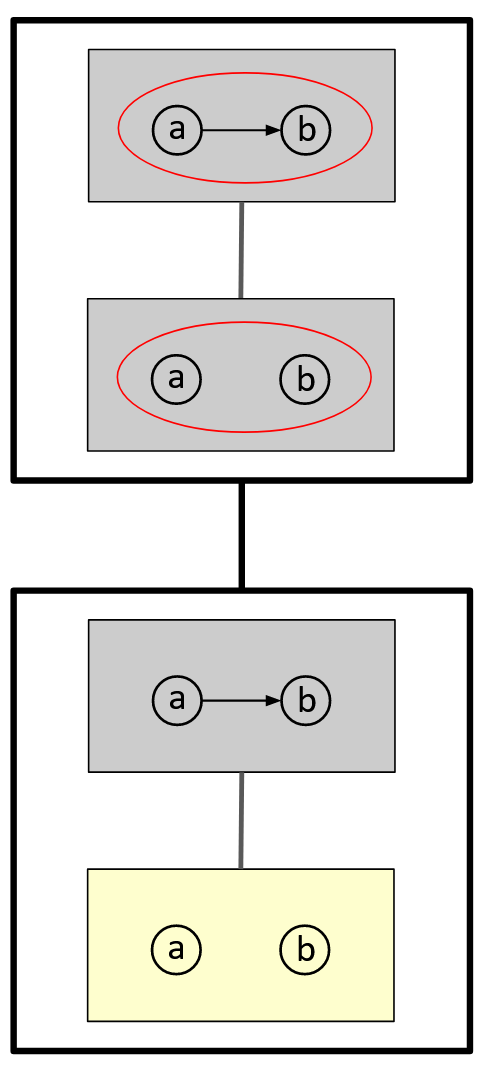}
	\caption{Compressed way of expressing the structural partial order that holds between two-node mDAGs. The decompressed structural partial order was given in Fig.~\ref{2node_structural}.} 
	\label{fig_2node_compressed} 
\end{figure}

To obtain the structural partial order that holds between the three-node mDAGs that are consistent with a choice $(a,b,c)$ of temporal ordering for the visible nodes, we can combine the structural partial order that holds between confounder-free three-node mDAGs consistent with $(a,b,c)$, shown in Fig.~\ref{fig_3node_confounderfree}(a), with the structural partial order that holds between directed-edge-free three-node mDAGs consistent with $(a,b,c)$, shown in Fig.~\ref{fig_3node_confounderfree}(b). The resulting partial order is presented in Fig.~\ref{fig_3node_splitnode}. Note that because the subset of confounder-free mDAGs is of cardinality 8 and the subset of directed-edge-free mDAGs is of cardinality 9, the full set, which can be obtained from their Cartesian product, is of cardinality 72. 
\begin{figure}[h!]
	\centering
	\includegraphics[width=0.45\textwidth]{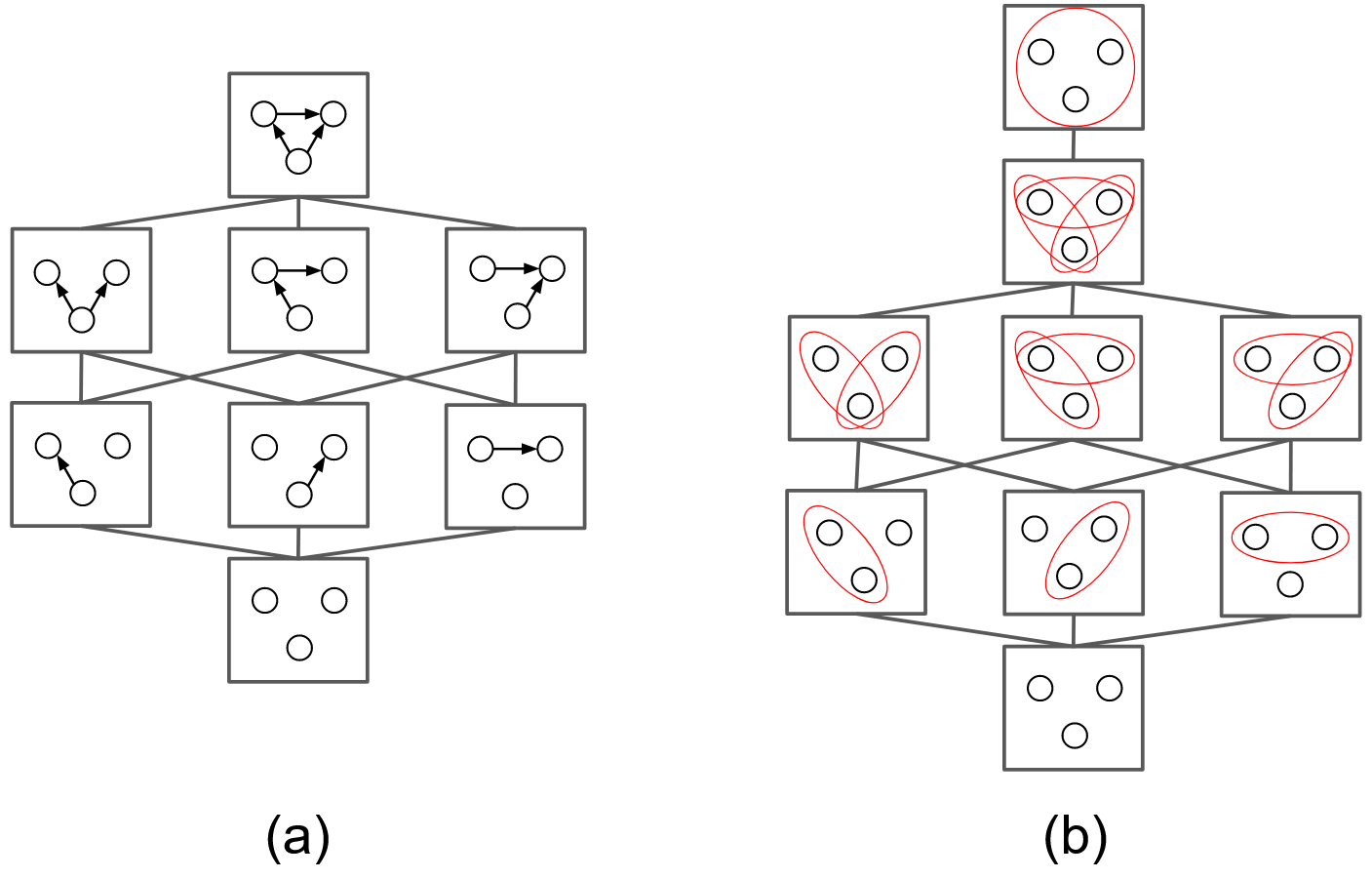}
	\caption{In all of the mDAGs depicted in this picture, let $a$ be the bottom node, $b$ be the top-left node and $c$ be the top-right node. (a) is the structural partial order that holds between all of the three-node confounder-free mDAGs that are consistent with the temporal ordering $(a,b,c)$. (b) is the structural partial order that holds between all of the three-node directed-edge-free mDAGs. Note that a directed-edge-free mDAG is always consistent with any temporal ordering of the visible nodes.} 
	\label{fig_3node_confounderfree} 
\end{figure}

\begin{figure*}[h!]
	\centering
	\includegraphics[width=0.5\textwidth]{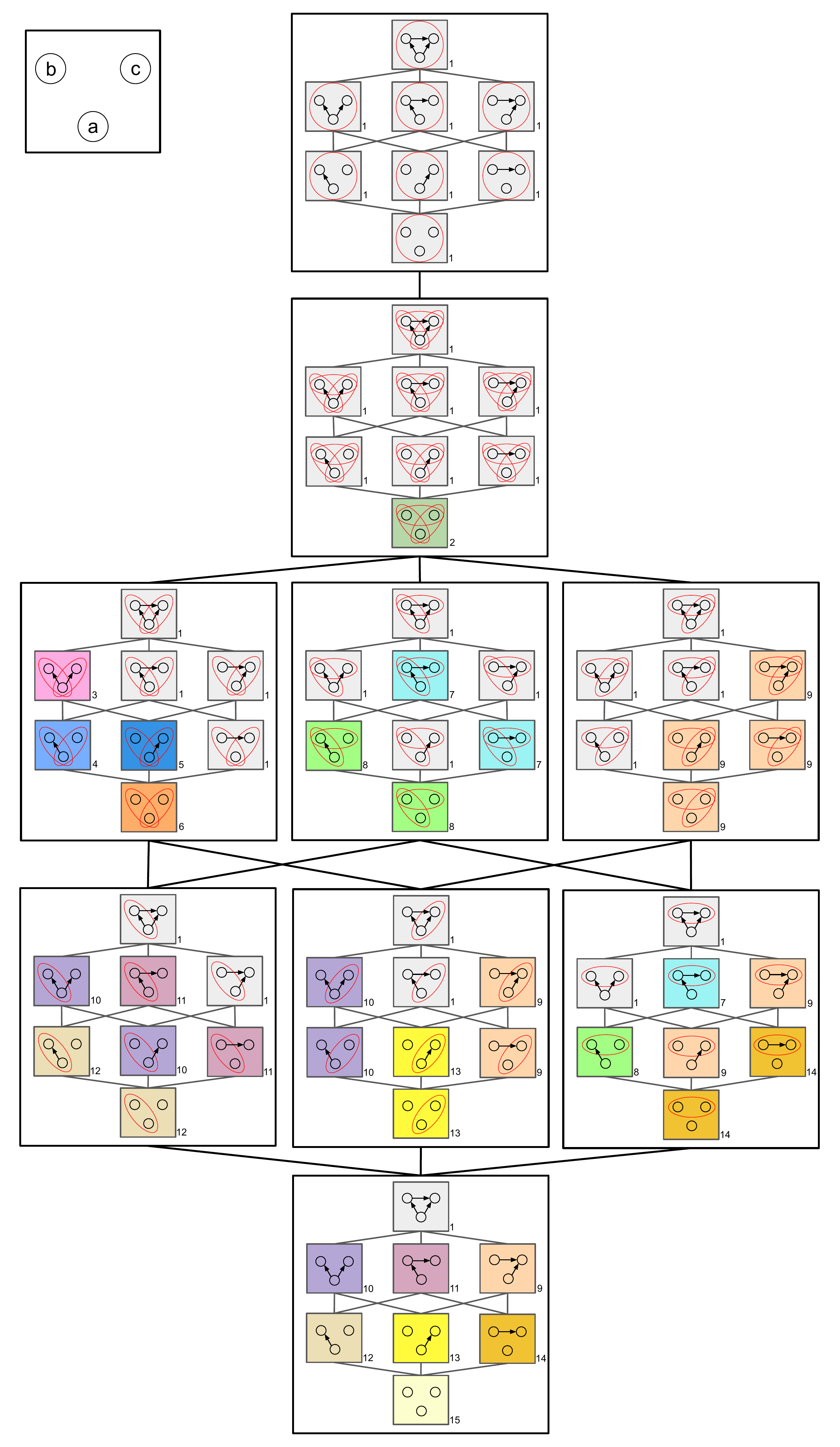}
	\caption{Structural partial order that holds between all three-node mDAGs consistent with the temporal ordering $(a,b,c)$, where the labelling is given in the top-left. The line between a pair of large boxes (islands) is to be interpreted as every mDAG within the higher island dominating the {\em corresponding} mDAG within the lower island, i.e., the mDAG that has the same directed-edge structure. Two mDAGs of this picture have the same background colour and the same number if they are \emph{observationally} equivalent. } 
	\label{fig_3node_splitnode} 
\end{figure*}

Although complicated, Fig.~\ref{fig_3node_splitnode} is useful as it gives us a picture of \emph{all} the structural dominance relations that hold between mDAGs of three nodes that are consistent with the temporal ordering $(a,b,c)$, and thus all of the O$\&$D, all-patterns O-or-D and all-patterns O-or-1D dominance relations as well. 
Sameness of background colour and of number in Fig.~\ref{fig_3node_splitnode} represents membership in the same observational equivalence class~\cite{taxonomy}, i.e.,  two mDAGs have the same background colour and the same number 
if they are observationally equivalent. This is yet another example of the difference between observational equivalence and equivalence when there is access to interventional probing schemes, under the three notions defined here.

We turn now to the case of the four-node mDAGs that are consistent with a single temporal ordering $(a,b,c,d)$ of the visible nodes.  The structural partial order of the subset that is confounder-free is presented in Fig~\ref{fig_confounder_free_splitnode},  and the structural partial order of the subset that is directed-edge-free is presented in Fig.~\ref{fig_directed_edge_free_partial_order4}. The full set of four-node mDAGs can be obtained by taking the Cartesian product of these two subsets.  However, because the confounder-free subset has cardinality 64 and the directed-edge-free subset has cardinality 113, the full set has cardinality 7232 and it is thus too cumbersome to depict it explicitly.  Nonetheless, one can easily infer the shape of the structural partial order of  the full set by imagining combining Figs.~\ref{fig_confounder_free_splitnode} and~\ref{fig_directed_edge_free_partial_order4} with the procedure described above.
\begin{figure*}[h!]
	\centering
	\includegraphics[width=0.9\textwidth]{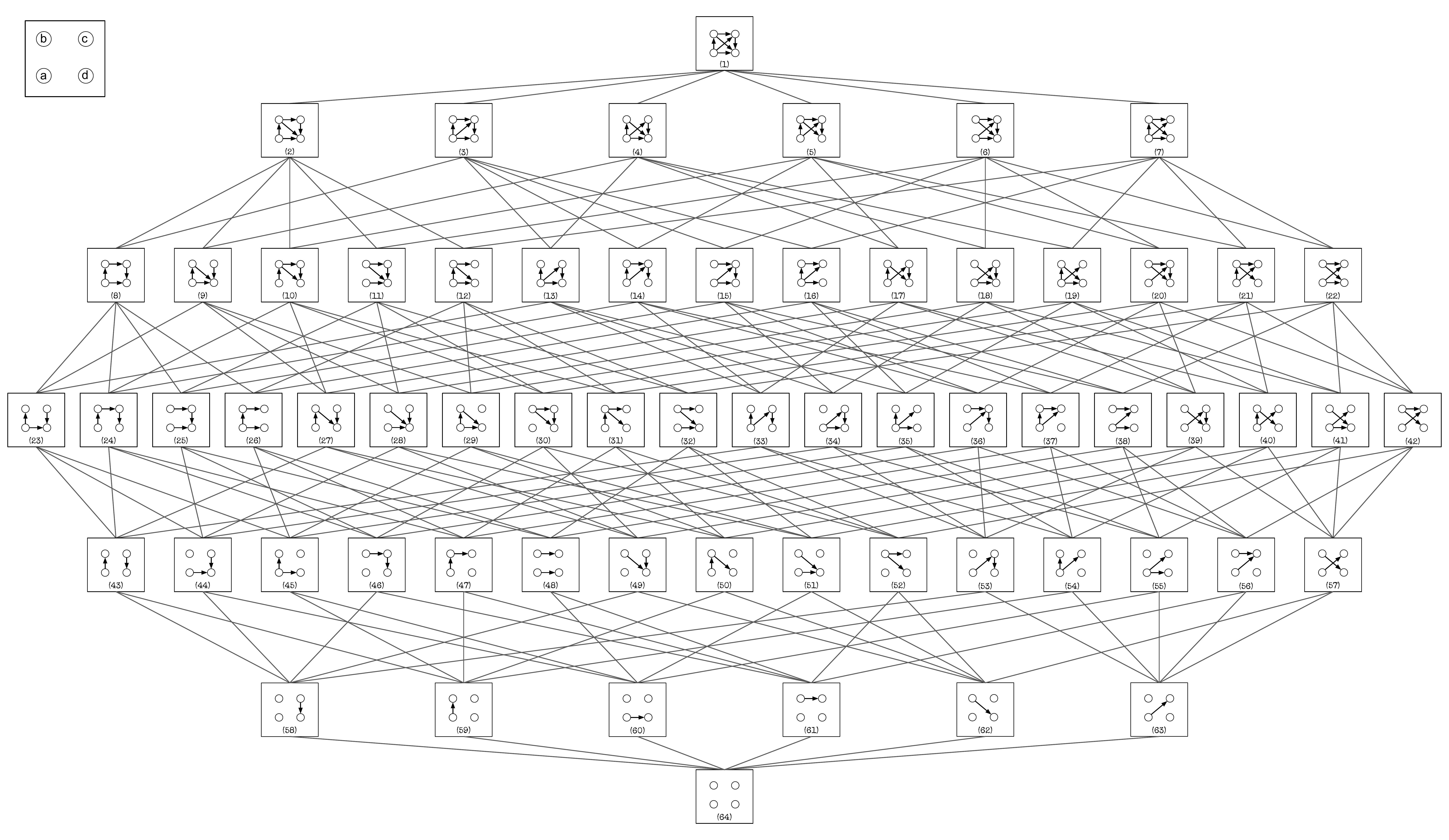}
	\caption{Structural partial order that holds between all four-node confounder-free mDAGs consistent with the temporal ordering $(a,b,c,d)$, where the labelling is given in the top-left.} 
	\label{fig_confounder_free_splitnode}
\end{figure*}

\begin{figure*}[htbp]
	\centering
	\includegraphics[width=0.9\textwidth]{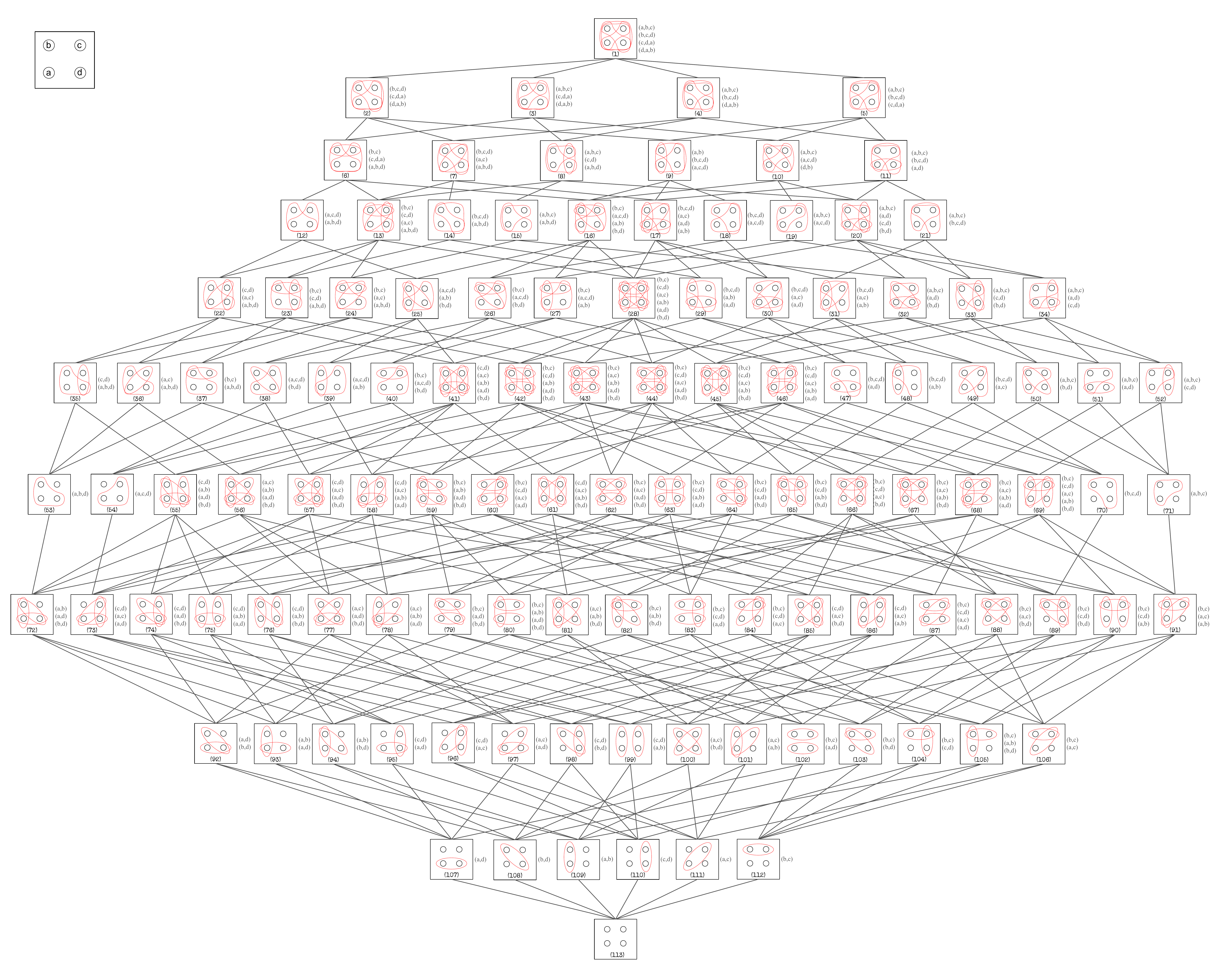}
	\caption{Structural partial order that holds between all four-node directed-edge-free mDAGs. Next to each mDAG, we indicate the set of facets of its simplicial complex.  Note that a directed-edge-free mDAG is always consistent with any temporal ordering of the visible nodes.} 
	\label{fig_directed_edge_free_partial_order4}
\end{figure*}

\section{Confounder-Free and Directed-Edge-Free mDAGs}
\label{sec_confounder_free}

We now consider the special case of confounder-free mDAGs. Those correspond to latent-free pDAGs, that is, basic DAGs. This is the case that has been extensively studied in the previous causal inference literature. The confounder-free mDAGs with 2 nodes are depicted in Fig.~\ref{fig_2node_confounderfree}(a), those with 3 nodes in Fig.~\ref{fig_3node_confounderfree}(a) and those with 4 nodes in Fig.~\ref{fig_confounder_free_splitnode}. 

A paradigm example of the impossibility of discriminating causal structures by passive observation alone is the example of a fork and a chain over three variables.  If one has a fork structure on three nodes and a chain structure on three nodes that {\em are} observationally indistinguishable, then this fork and chain do not admit of any topological ordering in common and so cannot arise as two possible causal structures consistent with single temporal ordering of the variables.  (In other words, the only fork and the only chain that are consistent with a single temporal ordering of the variables are observationally distinguishable, as noted in Fig.~\ref{fig_forkchain}.) 

Of course, much work has been done on the question of observational distinguishability of latent-free pDAGs (i.e., DAGs) where the variables are {\em not} temporally ordered but where one excludes cycles.  Because there are then nontrivial instances of observational {\em in}distinguishability, researchers were led to define the notion of a Markov equivalence class of such causal structures~\cite{Verma_Markov_classes}. It is also well-known that access to interventions can resolve such indistinguishability. Specifically, when one has access to do-interventions on any subset of nodes, i.e., all do-patterns, it is possible to distinguish between any two latent-free pDAGs.\footnote{Furthermore, researchers have considered distinguishability relative to a limited set of do-patterns (called “family of targets” in Ref.~\cite{interventional_Markov}), and defined the notion of {\em interventional Markov equivalence} for each such set~\cite{interventional_Markov}.  Upper bounds have been derived on the minimal cardinality of the set of do-patterns necessary for distinguishing all latent-free pDAGs with a given number of nodes~\cite{Eberhardt_interventions,Eberhardt_almost_optimal}.}

In terms of the notion of Markov equivalence, the point made above regarding the chain and the fork can be generalized as follows: if one is restricted to the set of causal structures that are consistent with a specific temporal ordering of the nodes, then there are no longer any examples of nontrivial Markov equivalence classes. In other words:
\begin{proposition}
	All latent-free pDAGs consistent with a fixed temporal ordering of the nodes are observationally inequivalent. 
	\label{theorem_confounder_free}
\end{proposition}

This fact seems to be well-known (see, e.g., below Theorem 1.2.8 in Ref.~\cite{causality_pearl}).  For pedagogical purposes, we provide a proof in Appendix~\ref{appendix_proof_theorem_confounder_free}.

Thus, while one must supplement passive observation with do-interventions to distinguish elements of a Markov equivalence class, in the context of causal structures with a fixed temporal ordering of the variables, passive observation alone is already sufficient to solve the discrimination problem.

It is also worth considering the discrimination problem for the subset of mDAGs that are directed-edge-free.  The directed-edge-free mDAGs with 2 nodes are depicted in Fig.~\ref{fig_2node_confounderfree}(b), those with 3 nodes in Fig.~\ref{fig_3node_confounderfree}(b) and those with 4 nodes in Fig.~\ref{fig_directed_edge_free_partial_order4}.  (The temporal ordering of nodes is not significant in this case as there are no directed edges.) As it turns out,
 passive observation is also sufficient in this case to distinguish between the causal structures.  This was shown in Proposition 6.8 of Ref.~\cite{evans_graphs_2016}.  

Therefore, observational indistinguishability of causal structures on a temporally ordered set of nodes arises entirely from  the presence of both latent common causes and directed edges in the set of such structures. The case of a structure with just two nodes $a$ and $b$ that are temporally ordered is the paradigm example: passive observation cannot distinguish whether $a$ causes $b$ or whether there is a latent common cause acting on both, or a combination of the two (see Fig.~\ref{2node_observational}).

\section{Conclusion}
\label{sec_conclusion}

	The findings of this work are summarized in Fig.~\ref{fig_experimental_hierarchy}. As discussed in Section~\ref{sec_probingschemes}, each probing scheme (without edge interventions) that can be implemented on the visible variables $X_{\vis(\GpDAG)}$ reveals a \emph{shadow} of the full conditional probability distribution $P(X_{\vis(\GpDAG)}^\flat|X_{\vis(\GpDAG)}^\sharp)$.  By applying the shadowing function associated with a probing scheme to the set of all possible conditional probability distributions realizable by a causal structure, we obtain the set of possible shadows realizable by the probing scheme. 
	
We can define an equivalence relation among probing schemes based on whether they reveal the same information about the causal hypothesis.  More precisely, two probing schemes are judged to be in the same equivalence class if they are associated to the same shadowing function.  
	For example, all of the \emph{informationally complete} probing schemes, that is, the ones that allow us to infer the full conditional probability distribution $P(X_{\vis(\GpDAG)^\flat}|X_{\vis(\GpDAG)^\sharp})$ (so that the shadowing function is the identity function), are inside the same equivalence class. 
	
	Moreover, we can define a partial order over the equivalence classes of probing schemes using the ordering relation induced by set inclusion of the images of their shadowing functions. Thus, one probing scheme is above another in the partial order if the image of the first's shadowing function contains all of the information that is included in the image of the second's shadowing function; 
	  in other words, if it reveals strictly more information about the full conditional probability distribution $P(X_{\vis(\GpDAG)^\flat}|X_{\vis(\GpDAG)^\sharp})$.

	Fig.~\ref{fig_experimental_hierarchy} presents the partial order of the equivalence classes of probing schemes that we studied here. 

	 In the box at the top of Fig.~\ref{fig_experimental_hierarchy}, we have probing schemes that are informationally complete. In this case, the problem of equivalence and dominance of pDAGs under the given probing scheme is completely characterized by the mDAG structure, and by the structural dominance of mDAGs respectively.  In particular, we noted that
an example of an informationally complete probing scheme is the Observe$\&$Do probing scheme, wherein, for each visible variable, one observes its natural value and subsequently performs a do-intervention upon it.  Theorem~\ref{th_ETT} demonstrates that dominance relative to such a probing scheme corresponds to structural dominance of the mDAGs. 
	
	Fig.~\ref{fig_experimental_hierarchy} also summarizes what one can infer about the causal structure when one has access to various probing schemes that are \emph{not} informationally complete. The box at the second highest level of the partial order of Fig.~\ref{fig_experimental_hierarchy} corresponds to the all-patterns Observe-or-Do probing scheme, wherein, for each subset of visible nodes (which we here call a \emph{do-pattern}), one implements a do-intervention on the variables in the subset and passive observation on the rest. 
	 In Theorem~\ref{th_do_and_ob} we show that, even though this probing scheme is not informationally complete, it is informative enough to be able to distinguish between different mDAGs. More specifically, the problem of equivalence and dominance of pDAGs relative to the  all-patterns Observe-or-Do probing scheme is also completely characterized by the mDAG structure and by the structural dominance of mDAGs respectively.
	
	Finally, in Theorem~\ref{th_1do_and_ob}, we strengthen this result by showing that it is not necessary to have the ability to perform do-interventions that set the visible variables to each of their possible values; in fact, implementing \emph{one-value} do-interventions (wherein a variable can be set to just one of its possible values) is sufficient to distinguish different mDAGs. We refer to this probing scheme as the {\em all-patterns Observe-or-1-value-Do} probing scheme. The mDAG structure also completely characterizes equivalence and dominance relative to this probing scheme. 
	
The lowest level in the hierarchy of probing schemes presented in Fig.~\ref{fig_experimental_hierarchy} is passive observation of all visible nodes.  Equivalence and dominance of pDAGs relative to passive observations is \emph{not} characterized by structural dominance of mDAGs. That is, one can find pairs of pDAGs that are observationally equivalent in spite of being associated to different mDAGs~\cite{evans_graphs_2016}.

	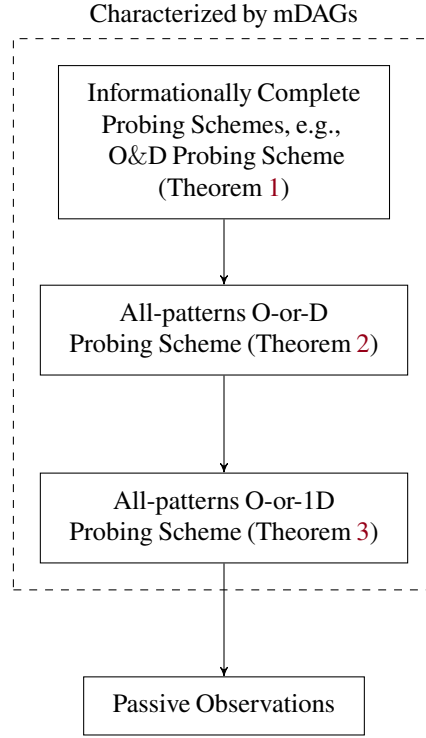
\begin{figure*}[h!]
	\centering
	\begin{tikzpicture}
	\node[c](C) at (2.5,0){\begin{tabular}{c} Informationally Complete \\ {Probing Schemes, e.g., } \\ { O$\&$D Probing Scheme} \\ {(Theorem~\ref{th_ETT})}  \end{tabular} };
	\node[c](B) at (2.5,-2.5){\begin{tabular}{c} All-patterns O-or-D  \\ Probing Scheme (Theorem~\ref{th_do_and_ob}) \end{tabular}};
	\draw[e] (B) to node[left,xshift=-4pt,yshift=0pt]{} (C);
	\node[c](A) at (2.5,-5){\begin{tabular}{c} All-patterns O-or-1D \\ Probing Scheme (Theorem~\ref{th_1do_and_ob}) \end{tabular}};
	\draw[e] (A) to node[left,xshift=-4pt,yshift=0pt]{} (B);
	\node[c](D) at (2.5,-7.5){\begin{tabular}{c} Passive Observations \end{tabular}};
	\draw[e] (D) to node[left,xshift=-4pt,yshift=0pt]{} (A);
	\node[draw,dashed,fit=(C)(B)(A),inner sep=10pt] (rectangle2) {};
	
	\node[above=0.02cm] at (rectangle2.north) {Characterized by mDAGs};
\end{tikzpicture}
	\caption{Partial order of equivalence classes of the different probing schemes studied in this work, where the relation is inclusion of the image of the shadowing function of each probing scheme. The informationally complete probing schemes, that reveal the entirety of this conditional distribution $P(X_{\vis(\GpDAG)}^\flat|X_{\vis(\GpDAG)}^\sharp)$, form an equivalence class. In the O\&D Probing Scheme, for each visible variable one can first observe its natural value and then force it to take any desired value. In the all-patterns O-or-D probing scheme, for each visible variable one can either passively observe it or force it to take any desired value, without previously observing the natural value of the variable. In the all-patterns O-or-1D probing scheme, for each visible variable one can either passively observe it or force it to take \emph{one} fixed value, without previously observing the natural value of the variable.}
	\label{fig_experimental_hierarchy}
\end{figure*}

 We end with some open problems. 

The first open problem is whether pDAGs associated to the same mDAG are still indistinguishable when one has access to edge interventions, that is, interventions that send a different value of the intervened node to each of its children. It is worth noting how one can conceptualize an edge intervention as a probing scheme on a visible variable. It suffices to imagine that a variable is first copied and then the different copies are what causally influence each of the original variable's children. One then intervenes differently on each of these copies. Is there any probing scheme \emph{with} edge interventions that can distinguish between two pDAGs that are associated with the same mDAG? We believe the answer is no.

A second open problem to be explored goes in the other direction: how much can a probing scheme be weakened and \emph{still} induce a dominance order of realizable distributions that is characterized by structural dominance of the associated mDAG structure?  Of the probing schemes we studied, the weakest one that satisfies this is the all-patterns Observe-or-1Do probing scheme. Are there even weaker probing schemes that also satisfy this condition? Also, might the condition be satisfied by other probing schemes that are not informationally complete and strictly incomparable (in the partial order) to the all-patterns Observe-or-Do scheme or the all-patterns Observe-or-1Do scheme?

Finally, all of the results presented in this work pertain to causal structures where all the nodes are associated with classical random variables. When the latent nodes of a pDAG are associated with \emph{quantum systems}, however, it is known that even passive observations of the visible variables can distinguish some pDAGs that are associated to \emph{the same} mDAG. An example is given in Figure 8 of Ref.~\cite{QuantumInflation}. Consequently, an open question of interest to quantum physicists is: what characterizes equivalence and dominance of causal structures  under an informationally complete probing scheme on the visible variables when the latent nodes are quantum?

\section*{Acknowledgements}
We thank Pedro Lauand for sharing with us a technique that was crucial for the proof presented in Appendix~\ref{appendix_theorem1_case2}. Research at Perimeter
Institute is supported in part by the Government of Canada through the Department of
Innovation, Science and Economic Development and by the Province of Ontario through
the Ministry of Colleges and Universities. MMA is supported by the Natural Sciences and Engineering Research Council of Canada (Grant No. RGPIN-2024-04419 ).

\FloatBarrier
\bibliographystyle{unsrt} 
\bibliography{references}

\begin{thebibliography}{10}

\bibitem{evans_graphs_2016}
Robin~J. Evans.
\newblock {Graphs for Margins of Bayesian Networks}.
\newblock {\em Scandinavian Journal of Statistics}, 43(3):625--648, 2016.
\newblock \\{\underline{Note}: Whenever we cited Propositions and Lemmas from
  this reference, we mentioned the numbering of the arXiv version, that is
  different from the numbering of the published version. This choice was made
  because the arXiv version may be more accessible to the reader.}

\bibitem{edge_interv}
Ilya Shpitser and Eric~Tchetgen Tchetgen.
\newblock Causal inference with a graphical hierarchy of interventions, 2014.

\bibitem{Verma_Markov_classes}
Thomas Verma and Judea Pearl.
\newblock Equivalence and synthesis of causal models.
\newblock In {\em Proceedings of the Sixth Annual Conference on Uncertainty in
  Artificial Intelligence}, UAI '90, page 255–270, USA, 1990. Elsevier
  Science Inc.

\bibitem{Andersson_Markov}
Steen~A. Andersson, David Madigan, and Michael~D. Perlman.
\newblock {A characterization of Markov equivalence classes for acyclic
  digraphs}.
\newblock {\em The Annals of Statistics}, 25(2):505 -- 541, 1997.

\bibitem{interventional_Markov}
Alain Hauser and Peter B\"{u}hlmann.
\newblock {Characterization and greedy learning of interventional Markov
  equivalence classes of directed acyclic graphs}.
\newblock {\em J. Mach. Learn. Res.}, 13(1):2409–2464, aug 2012.

\bibitem{causality_pearl}
Judea Pearl.
\newblock {\em {Causality: Models, Reasoning and Inference}}.
\newblock Cambridge University Press, 2nd edition, 2009.

\bibitem{richardson_SWIGs}
Thomas~S Richardson and James~M Robins.
\newblock Single world intervention graphs (swigs): A unification of the
  counterfactual and graphical approaches to causality.
\newblock {\em Center for the Statistics and the Social Sciences, University of
  Washington Series. Working Paper}, 128(30), 2013.

\bibitem{biomolecular}
Sara Mohammad~Taheri, Jeremy Zucker, Charles Hoyt, Karen Sachs, Vartika Tewari,
  Robert Ness, and Olga Vitek.
\newblock Do-calculus enables estimation of causal effects in partially
  observed biomolecular pathways.
\newblock {\em Bioinformatics (Oxford, England)}, 38:i350--i358, 06 2022.

\bibitem{nonlinear}
David Kaltenpoth and Jilles Vreeken.
\newblock Nonlinear causal discovery with latent confounders.
\newblock In {\em Proceedings of the 40th International Conference on Machine
  Learning}, ICML'23. JMLR.org, 2023.

\bibitem{henson_theory-independent_2014}
Joe Henson, Raymond Lal, and Matthew~F. Pusey.
\newblock {Theory-independent limits on correlations from generalized Bayesian
  networks}.
\newblock 16(11):113043, 2014.
\newblock Publisher: {IOP} Publishing.

\bibitem{Clauser_freechoice}
John~F. Clauser and Michael~A. Horne.
\newblock Experimental consequences of objective local theories.
\newblock {\em Phys. Rev. D}, 10:526--535, Jul 1974.

\bibitem{Bell_Freevariables}
John Bell.
\newblock Free variables and local causality.
\newblock {\em Epistemological Letters}, 15, 1977.

\bibitem{ETT_Shpitser}
Ilya Shpitser and Judea Pearl.
\newblock Effects of treatment on the treated: Identification and
  generalization.
\newblock In Jeff~A. Bilmes and Andrew~Y. Ng, editors, {\em {UAI} 2009,
  Proceedings of the Twenty-Fifth Conference on Uncertainty in Artificial
  Intelligence, Montreal, QC, Canada, June 18-21, 2009}, pages 514--521. {AUAI}
  Press, 2009.

\bibitem{SchmidOmelette}
David Schmid, John~H. Selby, and Robert~W. Spekkens.
\newblock Unscrambling the omelette of causation and inference: The framework
  of causal-inferential theories.
\newblock 2020.

\bibitem{Lorenz2023}
Robin Lorenz and Sean Tull.
\newblock Causal models in string diagrams.
\newblock 2023.

\bibitem{SteudelAy}
Bastian Steudel and Nihat Ay.
\newblock Information-theoretic inference of common ancestors.
\newblock {\em Entropy}, 17(4):2304--2327, 2015.

\bibitem{taxonomy}
Marina~Maciel Ansanelli, Elie Wolfe, and Robert Spekkens.
\newblock {In Preparation}.

\bibitem{Eberhardt_interventions}
Frederick Eberhardt, Clark Glymour, and Richard Scheines.
\newblock On the number of experiments sufficient and in the worst case
  necessary to identify all causal relations among n variables.
\newblock UAI'05, page 178–184, Arlington, Virginia, USA, 2005. AUAI Press.

\bibitem{Eberhardt_almost_optimal}
Frederick Eberhardt.
\newblock Almost optimal intervention sets for causal discovery.
\newblock UAI'08, page 161–168, Arlington, Virginia, USA, 2008. AUAI Press.

\bibitem{QuantumInflation}
Elie Wolfe, Alejandro Pozas-Kerstjens, Matan Grinberg, Denis Rosset, Antonio
  Ac\'{\i}n, and Miguel Navascu\'es.
\newblock Quantum inflation: A general approach to quantum causal
  compatibility.
\newblock {\em Phys. Rev. X}, 11:021043, May 2021.

\bibitem{verma_pearl}
Tom Verma and Judea Pearl.
\newblock Causal networks: Semantics and expressiveness.
\newblock {\em Proceedings of the Fourth Workshop on Uncertainty in Artificial
  Intelligence}, 4, 03 2013.

\bibitem{Geiger1988}
D.~Geiger.
\newblock Toward the formalization of informational dependencies.
\newblock 1988.

\end{thebibliography}

\onecolumn

\begin{appendices}
	\section{Proof of Lemma~\ref{lemma_commutation_maps}}
\label{appendix_proof_Lemmacommutation}

In this appendix we will prove Lemma~\ref{lemma_commutation_maps}, which asserts that the maps $\spl$ and ${\tt RE{-}reduce}$
 commute, as depicted in the commutative diagram of Fig.~\ref{fig_commuting_splitnode}. 
\begin{figure}[h!]
	\centering
	\includegraphics[width=0.5\textwidth]{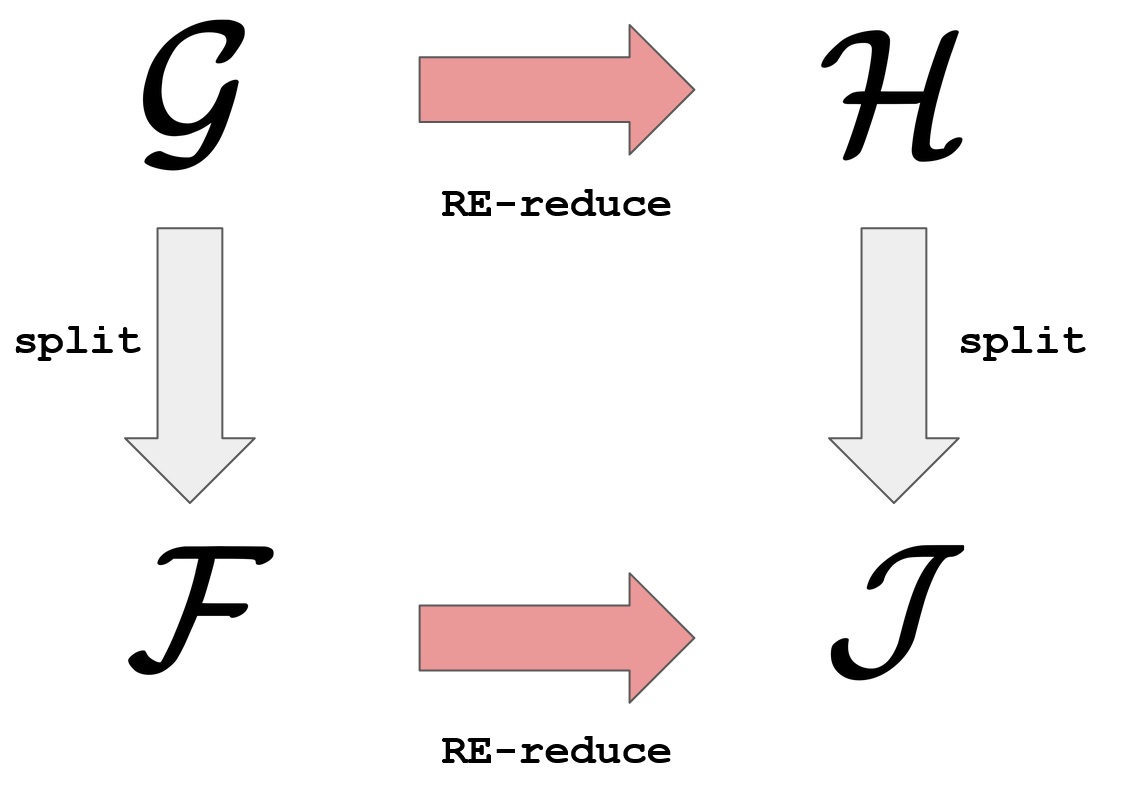}
	\caption{Commutative diagram corresponding to the statement of Lemma~\ref{lemma_commutation_maps}.}
	\label{fig_commuting_splitnode}
\end{figure}

Applying the {\tt RE-reduce} map can be broken into a sequence of steps, where one exogenizes or removes only \emph{one} latent node in each step. Therefore, it is enough to show that exogenizing \emph{one} latent node commutes with the split-node map, as well as removing  \emph{one} redundant latent node commutes with the split-node map. That is, we just need to prove that
\begin{align}
	&\spl(\exog(\mathcal{G},u))= \exog(\spl(\mathcal{G}),u) \text{, and} \label{eq_exog_commutes}
	\\ &\spl(\remove(\mathcal{G},v))= \remove(\spl(\mathcal{G}),v) \label{eq_remove_commutes}
\end{align}
for $u,v \in \lat(\mathcal{G})$ and where $v$ is parentless and redundant, i.e., there is another latent node $w \in \lat(\mathcal{G})$ whose set of children is a superset of the set of children of $v$.

Eqs.~\eqref{eq_exog_commutes} and~\eqref{eq_remove_commutes} are respectively proven by the commuting diagrams presented in Figs.~\ref{fig_commuting_exog_split} and~\ref{fig_commuting_remove_split}. In those diagrams, elongated ellipses represent 
a {\em set} of nodes (either visible or latent), and
rectangles represent a set of input nodes.

\begin{figure*}[h!]
	\centering
	\includegraphics[width=0.8\textwidth]{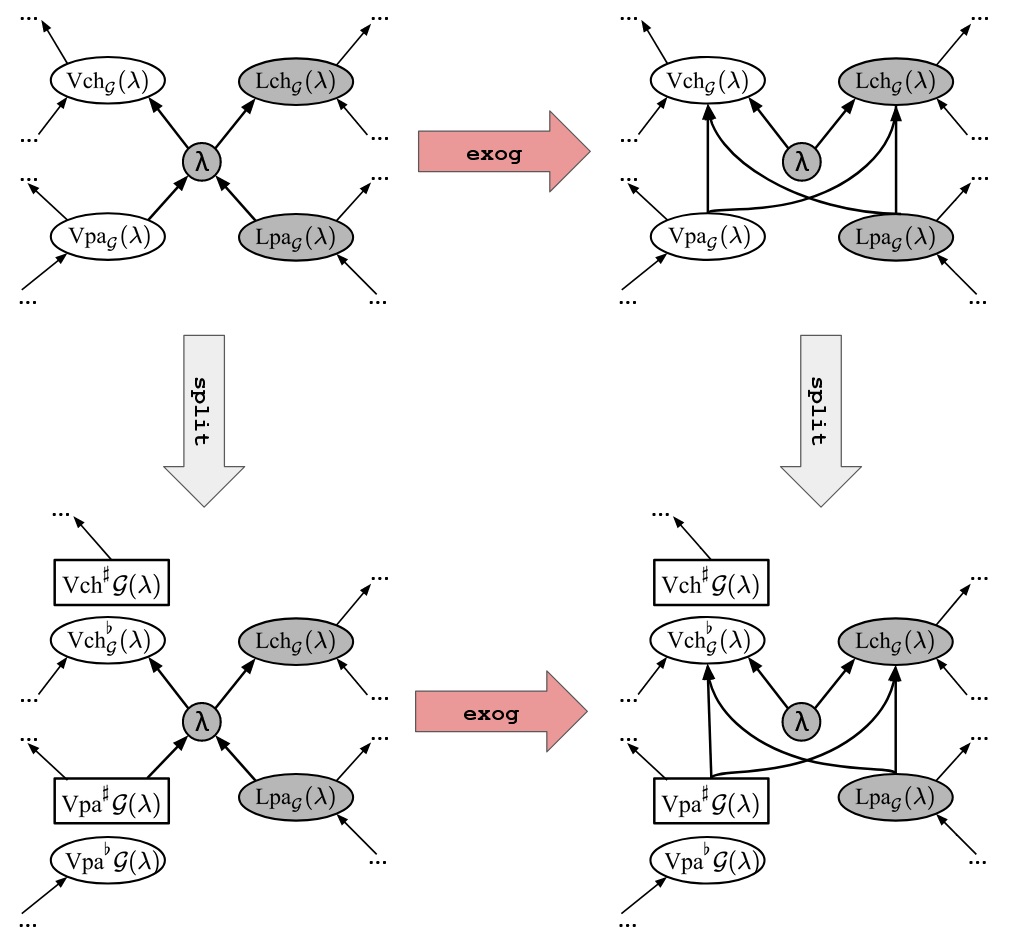}
	\caption{ Exogenizing a latent node commutes with the split-node map. }
	\label{fig_commuting_exog_split}
\end{figure*}

\begin{figure}[h!]
	\centering
	\includegraphics[width=0.8\textwidth]{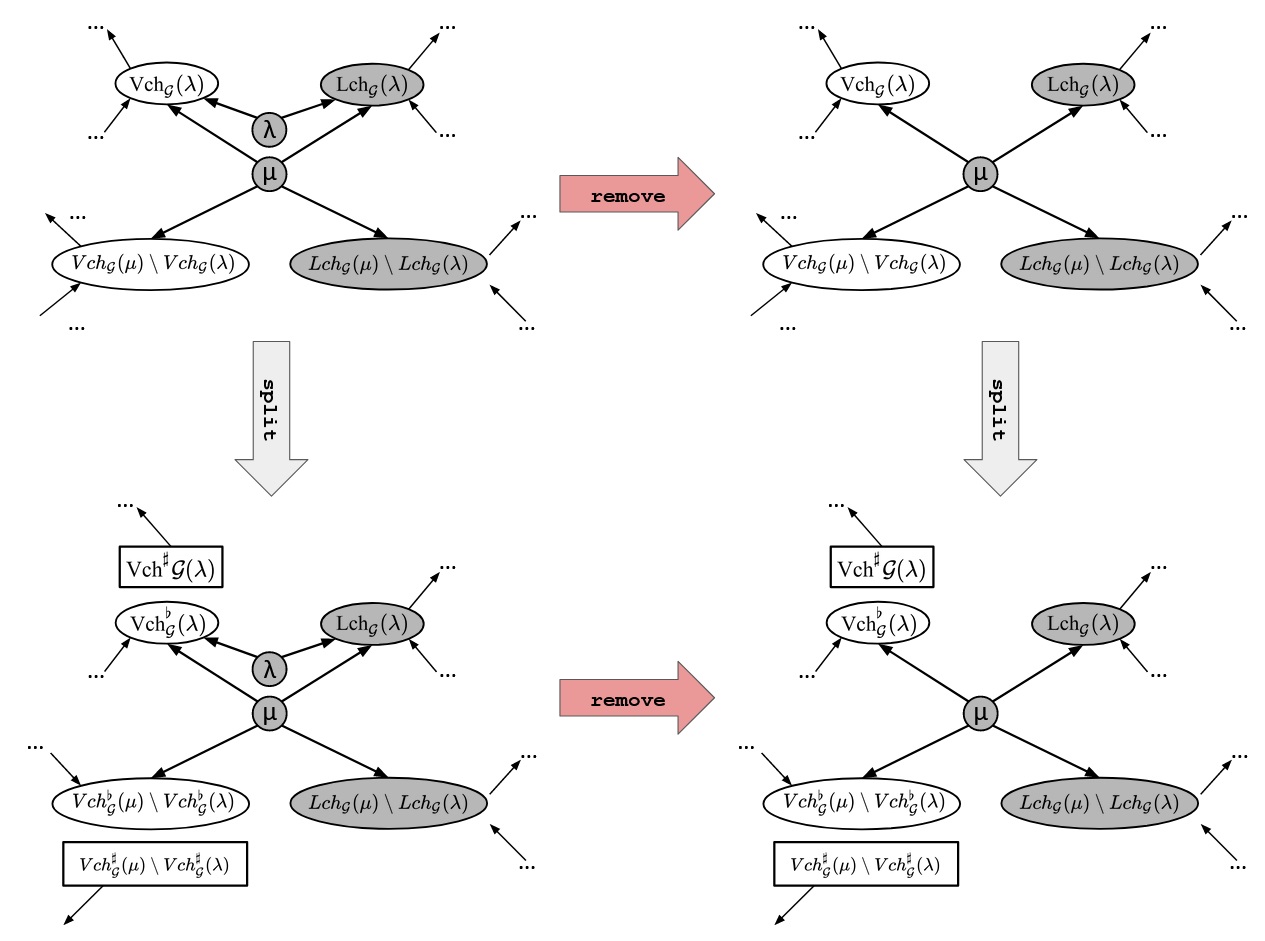}
	\caption{ Removing a redundant latent node commutes with the split-node map.}
	\label{fig_commuting_remove_split}
\end{figure}
	\section{Proof of the Case 2 of Theorem~\ref{th_do_and_ob}}
\label{appendix_theorem1_case2}

In this appendix, we will complete the proof of Theorem~\ref{th_do_and_ob} by showing that in case 2, where there is at least one set $S$ of nodes which is a face of $\GmDAG'=\mdag(\GpDAG')$ but not of $\GmDAG=\mdag(\GpDAG)$, the pDAG $\GpDAG$ does not all-patterns O-or-D dominate the pDAG $\GpDAG'$.

If the set $S$ does not have one common latent ancestor in $\GpDAG$, then the corresponding set $S^\flat$ also does not have one common latent ancestor in the full-SWIG $\GthreepDAG=\spl(\GpDAG)$. Thus, in the proof of Lemma~\ref{lemma_full_SWIGdominance} we have shown that the conditional probability distribution of Eq.~\eqref{eq_perfect_corrSWIG} cannot be realized by the full-SWIG $\GthreepDAG$.
We reproduce this conditional distribution below:
\begin{equation}
	P(X_{S^\flat}|X_{\vis(\GpDAG)^\sharp})=p[0,...,0]_{S^\flat}+(1-p)[1,...,1]_{S^\flat}
	\label{eq_repr}
\end{equation}


In Theorem~\ref{th_do_and_ob}, for each visible variable our experimentalist can only perform a passive observation \emph{or} a do-intervention. However, in the case of \emph{binary} distributions it is possible to employ a technique~\footnote{The technique is due to Pedro Lauand, and was provided to us via private communication.} to find the full conditional distribution $P(X_{\vis(\GpDAG)^\flat}|X_{\vis(\GpDAG)^\sharp})$ from data that can be obtained from the all-patterns Observe-or-Do probing scheme.
  To understand it, imagine a simple case with two variables $X_a$ and $X_b$. The most we can hope to learn from the causal hypothesis is given by the full conditional distribution $P(X_{a^\flat},X_{b^\flat}|X_{a^\sharp},X_{b^\sharp})$, which is given by a set of probabilities
\begin{equation}
	P\left(X_{a^\flat}=x_{a^\flat},X_{b^\flat}=x_{b^\flat}|X_{a^\sharp}=x_{a^\sharp},X_{b^\sharp}=x_{b^\sharp}\right).
	\label{eq_p_proof}
\end{equation}

Since we assume that the distribution is binary, the values $x_{a^\flat}$, $x_{b^\flat}$, $x_{a^\sharp}$ and $x_{b^\sharp}$ are either $0$ or $1$. We break the possibilities of values into four cases, where $\oplus$ indicates addition modulo $2$:

(i) $x_{a^\flat}=x_{a^\sharp}$ and $x_{b^\flat}=x_{b^\sharp}$

In this case, the intervention does not change anything relative to a simple passive observation. Therefore, expression~\eqref{eq_p_proof} reduces to the joint distribution obtained from passive observations on this causal hypothesis:
\begin{equation}
	P\left(X_{a^\flat}=x_{a^\sharp},X_{b^\flat}=x_{b^\sharp}\right).
	\label{eq_Pedro1}
\end{equation}

(ii) $x_{a^\flat}=x_{a^\sharp}$ and $x_{b^\flat}=x_{b^\sharp} \oplus 1$ 

In this case, we use the fact that $X_b$ is binary and the normalization of probabilities to reduce expression~\eqref{eq_p_proof} to:
\begin{equation}
	P(X_{a^\flat}=x_{a^\sharp}|X_{b^\sharp}=x_{b^\sharp})-P\left(X_{a^\flat}=x_{a^\sharp},X_{b^\flat}=x_{b^\sharp} \right)
	\label{eq_Pedro2}
\end{equation}

(iii) $x_{a^\flat}=x_{a^\sharp} \oplus 1$ and $x_{b^\flat}=x_{b^\sharp}$ 

By symmetry to the previous case, in this case expression~\eqref{eq_p_proof} becomes:
\begin{equation}
	P(X_{b^\flat}=x_{b^\sharp}|X_{a^\sharp}=x_{a^\sharp})-P\left(X_{a^\flat}=x_{a^\sharp},X_{b^\flat}=x_{b^\sharp} \right)
	\label{eq_Pedro3}
\end{equation}

(iv) $x_{a^\flat}=x_{a^\sharp}\oplus 1$ and $x_{b^\flat}=x_{b^\sharp} \oplus 1$ 

For this case, we will use the same technique two times. By using that $X_{b^\flat}$ is binary, expression~\eqref{eq_p_proof} becomes:
\begin{equation}
	P(X_{a^\flat}=x_{a^\sharp}\oplus 1|X_{a^\sharp}=x_{a^\sharp} ,X_{b^\sharp}=x_{b^\sharp})-P\left(X_{a^\flat}=x_{a^\sharp}\oplus 1,X_{b^\flat}=x_{b^\sharp} |X_{a^\sharp}=x_{a^\sharp}\right) 
\end{equation}

And by using that $X_{a^\flat}$ is also binary, this then becomes:
\begin{equation}
	 1 - P(X_{a^\flat}=x_{a^\sharp}|X_{b^\sharp}=x_{b^\sharp})- P\left(X_{b^\flat}=x_{b^\sharp}|X_{a^\sharp}=x_{a^\sharp}\right) + P\left(X_{a^\flat}=x_{a^\sharp},X_{b^\flat}=x_{b^\sharp} \right).
	 \label{eq_Pedro4}
\end{equation}

Note that expressions~\eqref{eq_Pedro1}-~\eqref{eq_Pedro4} all correspond to data that can be obtained from the all-patterns Observe-or-Do probing scheme: in those expressions, for each visible variable, \emph{either} its $\flat$ version appears to the left of the conditional \emph{or} its $\sharp$ version appears to the right of the conditional (but not both at the same time). 
Therefore, the conditional distribution of expression~\eqref{eq_p_proof} is in one-to-one correspondence to a set of data obtainable from  the all-patterns Observe-or-Do probing scheme.

Now, suppose that the re-prepared values of $X_a$ and $X_b$ are $0$, and that the conditional distribution of expression~\eqref{eq_p_proof} presents perfect correlation between the natural values of $X_a$ and $X_b$; that is, for some $p\in[0,1]$:
\begin{align}
	&P\left(X_{a^\flat}=0,X_{b^\flat}=0|X_{a^\sharp}=0,X_{b^\sharp}=0 \right) =p \nonumber \\ &P\left(X_{a^\flat}=1,X_{b^\flat}=1|X_{a^\sharp}=0,X_{b^\sharp}=0 \right) = 1-p
	\label{eq_simple_ETT_p}
\end{align}

This is a special case of Eq.~\eqref{eq_repr} when there are only two visible nodes and both of them are inside the set $S$. For this conditional distribution, the reductions~\eqref{eq_Pedro1}, \eqref{eq_Pedro2} and \eqref{eq_Pedro3} of expression~\eqref{eq_p_proof} give:
\begin{align}
	&P(X_{a^\flat}=0,X_{b^\flat}=0)=p  \label{eq_ETT_to_do1} \\
	&P(X_{a^\flat}=0|X_{b^\sharp}=0)=P(X_{a^\flat}=0,X_{b^\flat}=0) \label{eq_ETT_to_do2} \\
	&P(X_{b^\flat}=0|X_{a^\sharp}=0)=P(X_{a^\flat}=0,X_{b^\flat}=0) \label{eq_ETT_to_do3} 
\end{align}

Therefore, if our experimentalist is interested in attesting that the full conditional distribution is the one of Eq.~\eqref{eq_simple_ETT_p}, instead of performing an informationally complete probing scheme they can simply perform the all-patterns Observe-or-Do probing scheme to check that Eqs.~\eqref{eq_ETT_to_do1}, \eqref{eq_ETT_to_do2} and \eqref{eq_ETT_to_do3} are true. We did not write the equation correspondent to Eq.~\eqref{eq_Pedro4} here, because this would be redundant to Eqs.~\eqref{eq_ETT_to_do1}, \eqref{eq_ETT_to_do2} and \eqref{eq_ETT_to_do3} by normalization of probabilities.

The argument we have made before Eq.~\eqref{eq_repr} implies that, in the special case where $S=\vis(\GpDAG)=\{a,b\}$, Eqs.~\eqref{eq_ETT_to_do1}, \eqref{eq_ETT_to_do2} and \eqref{eq_ETT_to_do3} are not jointly realizable by $\GpDAG$. It is easy to see that, in this special case, the set of data described by Eqs.~\eqref{eq_ex_observational_proof2} and \eqref{eq_ex_interventional_proof2} (where passive observations show perfect correlation between the variables $X_a$ and $X_b$) obeys Eqs.~\eqref{eq_ETT_to_do1}, \eqref{eq_ETT_to_do2} and \eqref{eq_ETT_to_do3}. Thus, at least when $S=\vis(\GpDAG)=\{a,b\}$,  Eqs.~\eqref{eq_ex_observational_proof2} and \eqref{eq_ex_interventional_proof2} are \emph{not} jointly realizable by $\GpDAG$.

For the generic case, the proof is very similar. We will suppose that the set of data described by Eqs.~\eqref{eq_ex_observational_proof2} and \eqref{eq_ex_interventional_proof2} is jointly realizable by $\GpDAG$, and we will show that this implies that the full conditional distribution of Eq.~\eqref{eq_repr} must be realizable by $\GpDAG$, which is a contradiction. 

We start by noting that, for the evaluation of the conditional distribution where all of the variables in $X_{S^\flat}$ take value $0$, we have:
\begin{equation}
	P\left(X_{S^\flat}= 0_{S^\flat}\big\rvert X_{\vis(\GpDAG)^\sharp}= 0_{\vis(\GpDAG)^\sharp}\right)=P\left(X_{S^\flat}= 0_{S^\flat}\big\rvert X_{\vis(\GpDAG)^\sharp \setminus S^\sharp}= 0_{\vis(\GpDAG)^\sharp \setminus S^\sharp}\right),
	\label{eq_all_zero}
\end{equation}
which can be obtained by the all-patterns Observe-or-Do probing scheme. The equality above holds due to a logic similar to the one that lead to Eq.~\eqref{eq_Pedro1}. In this case, Eqs.~\eqref{eq_ex_observational_proof2} and \eqref{eq_ex_interventional_proof2} say that this expression takes the value $p$, which is exactly what is given by the corresponding evaluation of the full conditional distribution of  Eq.~\eqref{eq_repr}.

Eq.~\eqref{eq_repr} says that evaluations of the full conditional distribution   where some of the variables in $X_{S^\flat}$ take the value $0$ while others take the value $1$ must be zero. Starting with the assumption that the set of data obtained from the all-patterns Observe-or-Do probing scheme is given by Eqs.~\eqref{eq_ex_observational_proof2} and \eqref{eq_ex_interventional_proof2}, we will now prove that all of these evaluations are indeed zero, again in accordance with  Eq.~\eqref{eq_repr}. As an example, below we show the first of these evaluations, where we enumerate the nodes of the set $S$ as $S=\{s_1,...,s_n\}$:
\begin{align}
	&P\left(X_{s^\flat_1}=1,X_{s^\flat_2}=...=X_{s^\flat_n}=0|X_{\vis(\GpDAG)^\sharp} = 0_{\vis(\GpDAG)^\sharp}\right) \nonumber \\&=P\left(X_{s^\flat_2}=...=X_{s^\flat_n}=0|X_{\vis(\GpDAG)^\sharp}=0_{\vis(\GpDAG)^\sharp}\right)-P\left(X_{S^\flat}=  0_{S^\flat}| X_{\vis(\GpDAG)^\sharp}= 0_{\vis(\GpDAG)^\sharp}\right) 	\label{eq_induction_basecase}  \\ 
	&=P\left(X_{s^\flat_2}=...=X_{s^\flat_n}=0|X_{s^\sharp_1}=0,X_{\vis(\GpDAG)^\sharp\setminus S^\sharp}= 0_{\vis(\GpDAG)^\sharp\setminus S^\sharp}\right)-P\left(X_{S^\flat}= 0_{S^\flat}|X_{\vis(\GpDAG)^\sharp\setminus S^\sharp}= 0_{\vis(\GpDAG)^\sharp\setminus S^\sharp}\right), \nonumber
\end{align}
where we applied the technique described in the simple example above. In the data given by Eqs.~\eqref{eq_ex_observational_proof2} and \eqref{eq_ex_interventional_proof2}, both terms of the last line of Eq.~\eqref{eq_induction_basecase} are equal to $p$, and thus Eq.~\eqref{eq_induction_basecase} is equal to zero.

In the generic case (an evaluation where the first $i$ elements of $X_{S^\flat}$ take the value $1$, while the last $n-i$ take the value $0$), we will recursively apply the same technique. Applying it one time, we obtain:
\begin{align}
	&P\left(X_{s^\flat_1}=...=X_{s^\flat_i}=1,X_{s^\flat_{i+1}}=...=X_{s^\flat_n}=0|X_{\vis(\GpDAG)^\sharp}=0_{\vis(\GpDAG)^\sharp}\right)  \nonumber \\ &= P\left(X_{s^\flat_1}=...=X_{s^\flat_{i-1}}=1,X_{s^\flat_{i+1}}=...=X_{s^\flat_n}=0|X_{\vis(\GpDAG)^\sharp}= 0_{\vis(\GpDAG)^\sharp}\right) \label{term1} \\ & - P\left(X_{s^\flat_1}=...=X_{s^\flat_{i-1}}=1,X_{s^\flat_i}=...=X_{s^\flat_n}=0|X_{\vis(\GpDAG)^\sharp}= 0_{\vis(\GpDAG)^\sharp}\right) \label{term2}
\end{align}

In the first term of the right hand side, given in line~\eqref{term1}, the variable $X_{s^\flat_i}$ does not appear. In the second term, given in line~\eqref{term2}, this variable appears as $X_{s_i^\flat}=0$. Therefore, this first application of the technique got rid of any terms with $X_{s^\flat_i}=1$.  Note that when we recursively apply this technique for all variables, both of the terms \eqref{term1} and \eqref{term2} will be decomposed in many sub-terms that do not have any variable being equal to $1$. However, the sub-terms in which \eqref{term1} will be decomposed are equal to the sub-terms in which \eqref{term2} will be decomposed, except for the fact that in the case of \eqref{term2} there will be $X_{s^\flat_i}=0$ appearing in each of them, while in the case of \eqref{term1} there will not. In other words, a sub-term in which \eqref{term1} will be decomposed is of the form
\begin{equation}
		P\left(X_{T^\flat}=0|X_{\vis(\GpDAG)^\sharp\setminus T^\sharp}= 0_{\vis(\GpDAG)^\sharp\setminus T^\sharp}\right),
		\label{eq_subterm1}
\end{equation}
where $T\subseteq S$ is a subset of $S$ that does not include $s_i$. For each of these, there is a sub-term in which \eqref{term2} will be decomposed that is of the form
\begin{equation}
	P\left(X_{T^\flat}=0, X_{s^\flat_i}=0|X_{\vis(\GpDAG)^\sharp\setminus \{T^\sharp\cup \{s^\sharp_i\}\}}=0_{\vis(\GpDAG)^\sharp\setminus \{T^\sharp\cup \{s^\sharp_i\}\}}\right).
	\label{eq_subterm2}
\end{equation}
In both sub-terms~\eqref{eq_subterm1} and~\eqref{eq_subterm2}, we already used the fact that do-interventions that force the variable to take the value $0$ do not do anything if the variable already naturally takes the value $0$ (the same fact that was used in Eq.~\eqref{eq_all_zero}). 

In both expressions~\eqref{eq_subterm1} and~\eqref{eq_subterm2}, there is no variable whose $\flat$ version appears to the left of the conditional at the same time that the $\sharp$ version appears to the right of the conditional. Therefore, they are both obtainable from the all-patterns Observe-or-Do probing scheme. We can thus look at Eqs.~\eqref{eq_ex_observational_proof2} and \eqref{eq_ex_interventional_proof2} to see that both sub-terms \eqref{eq_subterm1} and \eqref{eq_subterm2} are equal to $p$, which implies that they will cancel each other when the term \eqref{term2} is subtracted from the term \eqref{term1}. This implies that 
\begin{equation}
	P\left(X_{s^\flat_1}=...=X_{s^\flat_i}=1,X_{s^\flat_{i+1}}=...=X_{s^\flat_n}=0|X_{\vis(\GpDAG)^\sharp}= 0_{\vis(\GpDAG)^\sharp}\right)= 0,
	\label{eq_giveszero}
\end{equation}
which is in accordance with Eq.~\eqref{eq_repr}.

Therefore, we proved that when all of the variables in $X_{S^\flat}$ evaluate to $0$ (Eq.~\eqref{eq_all_zero}), Eqs.~\eqref{eq_ex_observational_proof2} and \eqref{eq_ex_interventional_proof2} imply that $P(\vis(\GpDAG)^\flat|\vis(\GpDAG)^\sharp)$ is equal to $p$, and when some of the variables evaluate to $0$ and others to $1$ (Eq.~\eqref{eq_giveszero}), Eqs.~\eqref{eq_ex_observational_proof2} and \eqref{eq_ex_interventional_proof2} imply that $P(\vis(\GpDAG)^\flat|\vis(\GpDAG)^\sharp)$ is equal to $0$. In other words, Eqs.~\eqref{eq_ex_observational_proof2} and \eqref{eq_ex_interventional_proof2} imply in the conditional distribution of Eq.~\eqref{eq_repr}, which we already proved to be non-realizable by $\GthreepDAG=\spl(\GpDAG)$. This contradiction implies that Eqs.~\eqref{eq_ex_observational_proof2} and \eqref{eq_ex_interventional_proof2} are not jointly realizable by $\GpDAG$.
	\section{Proof of Proposition~\ref{theorem_confounder_free}}
\label{appendix_proof_theorem_confounder_free}

In this proof of observational inequivalence, we will explicitly show a d-separation relation that is presented by one of the pDAGs in question but not by the other. Given a pDAG $\GpDAG$ and three disjoint sets of visible nodes $A,B,C\in \vis(\GpDAG)$, the \emph{d-separation criterion}~\cite{verma_pearl,Geiger1988} returns whether $A$ and $B$ are ``d-separated'' by $C$ or not. The affirmative case is denoted by $A\perp B|C$. As shown in Ref.~\cite{verma_pearl}, a d-separation relation of a DAG implies that the realizable distributions have to satisfy an associated \emph{conditonal independence constraint}. In a certain distribution $P$, the variable $X_A$ is said to be independent of the variable $X_B$ after conditioning on the variable $X_C$ if the following is valid: 
\begin{equation}
	P(X_A X_B|X_C)=P(X_A|X_C)P(X_B|X_C).
	\label{eq_CI}
\end{equation}
Importantly, Ref.~\cite{verma_pearl} also showed that, if a pDAG $\GpDAG$ \emph{does not} present such d-separation relation, i.e. if $A\not\perp B|C$ in  $\GpDAG$, then there \emph{is} some distribution that is realizable by $\GpDAG$ and does not obey the constraint of Eq.~\eqref{eq_CI}. That is, the d-separation criterion is necessary and sufficient to give us the set of conditional independence constraints imposed by a pDAG on its realizable distributions.

The following Lemma is a consequence of the observations above:
\begin{lemma}
	Let $\GpDAG$ and $\GpDAG'$ be two pDAGs that have the same set of visible nodes, $\vis(\GpDAG)=\vis(\GpDAG')$. If they present different sets of d-separation relations, then they are observationally inequivalent.
	\label{lemma_dsep_obs}
\end{lemma}
\begin{proof}
	Without loss of generality, let $A\perp B|C$ be a d-separation relation that is presented by $\GpDAG'$ but not in $\GpDAG$. Then, by the first result of Ref.~\cite{verma_pearl} mentioned above, all of the distributions realizable by $\GpDAG'$ have to satisfy the constraint of Eq.~\eqref{eq_CI}. By the second result of Ref.~\cite{verma_pearl} mentioned above, there \emph{exists} some distribution that is realizable by  $\GpDAG$ and does not satisfy this constraint. Therefore, there exists at least one probability distribution obtained from passive observations that is realizable by  $\GpDAG$ but not by  $\GpDAG'$, implying that they are observationally inequivalent.
\end{proof}

Now, we can proceed to the proof of Proposition~\ref{theorem_confounder_free}. Let $\GpDAG$ and $\GpDAG'$ be two latent-free pDAGs (i.e., two DAGs) with $n$ nodes, and let $\left\{a^{(i)}\right\}_{i=1,...,n}$ be the set of their nodes such that $i<j$ if and only if $a^{(i)}$ comes before $a^{(j)}$ in the temporal ordering. That is, the temporal ordering of nodes is $\left (a^{(1)}, a^{(2)},...,a^{(n)}\right )$. Assume that both $\GpDAG$ and $\GpDAG'$ are consistent with this temporal ordering.

Since $\GpDAG$ and $\GpDAG'$ are not equal, there must be at least one arrow that is present in one of them but not the other. Without loss of generality, assume that $\GpDAG$ possesses the arrow  $a^{(i)}\rightarrow a^{(j)}$, while $\GpDAG'$ does not. Note that this necessarily implies that $i<j$, since we assumed that $\GpDAG$ and $\GpDAG'$ are consistent with the temporal ordering  $\left (a^{(1)}, a^{(2)},..., a^{(n)}\right )$. Now, we show that the following d-separation relation is presented by $\GpDAG'$ but not $\GpDAG$:
\begin{equation}
	a^{(i)}\perp a^{(j)}| a^{(1)},...,a^{(i-1)},a^{(i+1)},...,a^{(j-1)}
	\label{eq_dsep}
\end{equation}

It is clear that $\GpDAG$ \emph{does not} present this d-separation relation, since it presents a direct arrow $a^{(i)}\rightarrow a^{(j)}$. To see that $\GpDAG'$ presents this d-separation relation note that, since $\GpDAG'$ is latent-free and consistent with the temporal ordering $\left (a^{(1)}, a^{(2)},...,a^{(n)}\right )$, any node that is a common cause between $a^{(i)}$ and $a^{(j)}$ \emph{or} a mediary in a chain between $a^{(i)}$ and $a^{(j)}$ is necessarily in the set $\left\{a^{(1)},...,a^{(i-1)},a^{(i+1)},...,a^{(j-1)}\right\}$. Therefore, all such paths between  $a^{(i)}$ and $a^{(j)}$ are blocked by conditioning on this set. Furthermore, conditioning on this set does not open any new paths via colliders, because the temporal order says that the set  $\left\{a^{(1)},...,a^{(i-1)},a^{(i+1)},...,a^{(j-1)}\right\}$ does not include any children of $a^{(j)}$. 

Therefore, the d-separation relation of expression~\eqref{eq_dsep} is presented by $\GpDAG'$ but not by $\GpDAG$. Together with Lemma~\ref{lemma_dsep_obs}, this proves  Theorem~\ref{theorem_confounder_free}.
\end{appendices}

\end{document}